\documentclass[10pt]{article} 

\usepackage[preprint]{rlj}           

\usepackage{amssymb}            
\usepackage{mathtools}          
\usepackage{mathrsfs}           
\usepackage{graphicx}           
\usepackage{subcaption}         
\usepackage[space]{grffile}     
\usepackage{url}                
\usepackage{lipsum}             

\usepackage{authblk}
\usepackage{graphicx}
\usepackage{amsfonts,amssymb,amsthm}
\usepackage{float}
\usepackage[mathscr]{euscript}
\usepackage{subcaption}
\usepackage{multicol}
\usepackage{stmaryrd}
\usepackage{thm-restate}
\tcbuselibrary{theorems}
\tcbuselibrary{skins}

\usepackage[nameinlink]{cleveref}
\newtcbtheorem{condition}
{\bfseries Empirical Condition}
    {
    enhanced,
    drop shadow={black!50!white},
    coltitle=black,
    top=0.2in,
    attach boxed title to top left={xshift=1.5em,yshift=-\tcboxedtitleheight/2},
    boxed title style={size=small,colback=white}
    }
{condi}

\newtcbtheorem[use counter=definition]{boxdef}
{\bfseries Definition}
    {
    enhanced,
    drop shadow={black!50!white},
    coltitle=black,
    top=0.2in,
    attach boxed title to top left={xshift=1.5em,yshift=-\tcboxedtitleheight/2},
    boxed title style={size=small,colback=white}
    }
{boxdef}

\makeatletter
\def\condition@counter{condition}
\def\condition@counter{boxdef}
\makeatother
\crefname{tcb@cnt@condition}{Empirical Condition}{Empirical Conditions}
\Crefname{tcb@cnt@condition}{Empirical Condition}{Empirical Conditions}

\crefname{tcb@cnt@boxdef}{Definition}{Definition}
\Crefname{tcb@cnt@boxdef}{Definition}{Definition}

\usepackage[utf8]{inputenc}
\usepackage{enumitem}

\newcommand{\Ibb}{\mathbb{I}}

\newcommand{\Nbb}{\mathbb{N}}

\newcommand{\Pbb}{\mathbb{P}}

\newcommand{\Rbb}{\mathbb{R}}

\newcommand{\Acal}{\mathcal{A}}

\newcommand{\Dcal}{\mathcal{D}}

\newcommand{\Gcal}{\mathcal{G}}
\newcommand{\Hcal}{\mathcal{H}}

\newcommand{\Ocal}{\mathcal{O}}
\newcommand{\Pcal}{\mathcal{P}}
\newcommand{\Qcal}{\mathcal{Q}}
\newcommand{\Rcal}{\mathcal{R}}
\newcommand{\Scal}{\mathcal{S}}

\newcommand{\Xcal}{\mathcal{X}}

\newcommand{\Zcal}{\mathcal{Z}}

\DeclareMathOperator*{\argmax}{arg\,max}

\newtheorem{definition}{Definition}
\newtheorem{corollary}{Corollary}

\newtheorem{example}{Example}

\title{Artifacts as Memory Beyond the Agent Boundary}

\setrunningtitle{Artifacts as Memory Beyond the Agent Boundary}

\author{John D. Martin\textsuperscript{1,2,$\dagger$}, 
Fraser Mince\textsuperscript{3,$\dagger$}, 
Esra'a Saleh\textsuperscript{3,4,5}, 
Amy Pajak\textsuperscript{3,6}}

\emails{john.martin@openmindresearch.org, frasermince@gmail.com, esraa.saleh@mila.quebec, pajak@seas.upenn.edu}

\affiliations{
$^{1}$\textbf{Openmind Research Institute}\\
$^{2}$\textbf{University of Alberta}\\
$^{3}$\textbf{Cohere Labs Community}\\
$^{4}$\textbf{Universit\'e de Montr\'eal}\\
$^{5}$\textbf{Mila -- Qu\'ebec AI Institute}\\
$^{6}$\textbf{University of Pennsylvania}\\
$^{\dagger}$\textbf{Equal contribution}\\
}

\contribution{
    We introduce a formalism for how the environment can functionally serve as an agent's memory  (\autoref{sec:formalizing_externalization}). Central to our formalism is the concept of artifacts (\autoref{def:artifact}), which we define as observations that inform the past, and external memory ~(\hyperref[def:externalized_compute]{Definition 3}). 
    }
    {
    Prior work has hypothesized about externalizing memory \citep{clark1998extended, sutton2003constructive}, but without a precise mathematical characterization of the phenomena.
    }

\contribution{
    The Artifact Reduction Theorem (\autoref{thm:artifact_reduction}): our main theoretical result proves artifacts (\autoref{def:artifact}) reduce the information needed to represent history. 
    }
    {
    The formalism we introduce naturally raises the question of what formal properties can be established.
    }

\contribution{
    Empirical evidence that RL agents can use their environment as a form of memory.
    }
    {
    Prior work, including our proposed theory, suggests that RL agents can externalize memory. Yet empirical evidence supporting this claim remains absent.    
    }
    
\contribution{
    An argument that spatial artifacts from our experiments satisfy qualitative memory properties \citep{michaelian2012external} used to ground accounts of externalized memory (\autoref{sec:criteria}).
    }
    {
    Our theoretical results establish a link between artifacts and memory, and our empirical results provide evidence that artifacts can reduce memory demands. However, the nature of memory we study requires further contextualization.
}


\summary{
The situated view of cognition holds that intelligent behavior depends not only on internal memory, but on an agent's active use of environmental resources. Here, we begin formalizing this intuition within Reinforcement Learning (RL). We introduce a mathematical framing for how the environment can functionally serve as an agent's memory, and prove that certain observations, which we call artifacts, can reduce the information needed to represent history. We corroborate our theory with experiments showing that when agents observe spatial paths, the amount of memory required to learn a performant policy is reduced. Interestingly, this effect arises unintentionally, and implicitly through the agent's sensory stream. We discuss the implications of our findings, and show they satisfy qualitative properties previously used to ground accounts of external memory. Moving forward, we anticipate further work on this subject could reveal principled ways to exploit the environment as a substitute for explicit internal memory.
}

\begin{document}

\makeCover  
\maketitle  

\begin{abstract}
The situated view of cognition holds that intelligent behavior depends not only on internal memory, but on an agent's active use of environmental resources. Here, we begin formalizing this intuition within Reinforcement Learning (RL). We introduce a mathematical framing for how the environment can functionally serve as an agent's memory, and prove that certain observations, which we call artifacts, can reduce the information needed to represent history. We corroborate our theory with experiments showing that when agents observe spatial paths, the amount of memory required to learn a performant policy is reduced. Interestingly, this effect arises unintentionally, and implicitly through the agent's sensory stream. We discuss the implications of our findings, and show they satisfy qualitative properties previously used to ground accounts of external memory. Moving forward, we anticipate further work on this subject could reveal principled ways to exploit the environment as a substitute for explicit internal memory.
\end{abstract}

\section{Introduction}
According to the situated view of cognition, competent action depends not only on internal memory, but on an agent's use of environmental resources~\citep{hutchins1995cognition, clark1998being, menary2010cognitive}. On some accounts, the environment itself can implicitly function as an agent's memory \citep{clark1998extended, sutton2003constructive}. In this paper, we aim to formalize such cases within Reinforcement Learning (RL). As a first step, we focus on one form of \textit{externalized memory} which centers on the use of \textit{artifacts} \citep{hutchins2001cognitive} to store information about an agent's previous interactions---for instance, a trail of breadcrumbs indicating where the agent has been before.  

We make three main contributions. First, we introduce a mathematical framing for how the environment can functionally serve as an agent's memory. Our framework grounds the concept of artifacts as observations that inform the past (\autoref{def:artifact}), and proves the amount of information needed to represent a history is reduced when artifacts are present (\autoref{thm:artifact_reduction}). We equate externalized memory to a condition on the amount of capacity needed to learn a performant policy (\hyperref[def:externalized_compute]{Definition 3}), and show the amount of externalized memory can be systematically quantified. Our proposed method compares the capacity needed to match performance across two settings that differ in whether the agent can observe behavioral artifacts, such as a spatial path. 

Second, we empirically confirm that RL agents can use spatial environments as a form of memory. We find evidence for this in a five different settings and from two core agent designs: Q-learning \citep{watkins1992q} and DQN \citep{mnih2015human}. In each case, we find the use of external memory arises unintentionally; leaving behind a spatial path---like a trail of breadcrumbs---is enough for the agent to experience the effect. 

Third, we place our results in a broader conceptual context and show that they satisfy qualitative properties previously used to ground accounts of external memory \citep{michaelian2012external, sims2022externalized}. We discuss our results, and suggest that further work in this area could yield principled ways to exploit the environment as a substitute for explicit internal memory.
\section{The Reinforcement Learning Formalism}
We adopt a purely experiential framing of RL, inspired by observable operator models \citep{jaeger2000observable}, predictive state representations \citep{littman2001predictive, singh2004predictive}, and other generalizations of RL \citep{hutter2004universal,dong2022simple,abel2023convergence,bowling2023settling}. Experiential models are appealing because they make few assumptions and are grounded entirely in observable data. That said, alternative frameworks such as POMDPs \citep{kaelbling1998planning} remain available.

\paragraph{Interaction.} At every moment $t\in \Nbb$, an agent draws on its sense-data $o_t\in\Ocal$, takes an action $a_t\in\Acal$, then observes the outcome through the updated observation $o_{t+1}\in\Ocal$ and scalar reward $r_{t+1}\in\Rcal$. We assume the sets $\Ocal,\Acal$ and $\Rcal$ are all finite. In episodic settings, this interaction repeats until a termination condition is met, after which the interaction reinitializes. Traces of interaction are described by sequences of observations and actions, called \textit{histories}. From the agent's perspective, a single history is $h=o_1a_1,o_2a_2,\cdots$, coming from the set of all histories $\Hcal\equiv(\Ocal \times \Acal)^*$. 

\paragraph{Bounded Agents.} We study agents with bounded representations and input channels, defining a bounded agent in two parts relative to a history-dependent agent $\lambda \colon \Hcal\times\Ocal \rightarrow \Delta(\Acal)$ over interface $(\Ocal,\Acal)$. First, an agent has a \textit{bounded representation} if it possesses a finite set of internal states $\Scal$, defined formally as $\pi\colon \Scal \times \Ocal \rightarrow \Delta(\Acal)$, where there exists some $s\in\Scal$ such that $\lambda(\cdot|h) = \pi(\cdot|s)$ for all $h\in\Hcal$ \citep{abel2023convergence}. Second, an agent has a \textit{bounded input channel} if its observations are filtered through a fixed mapping $\tau\colon \Xcal \rightarrow \Ocal$, which we call the \textit{transduction function}\footnote{The transduction function is part of the environment controlled by design. For example, consider a resource-constrained robot with a camera. The camera is part of the environment and chosen by the designer. If the camera provides $n\times n$ images, and the robot is only equipped to support $m\times m$ images, where $m < n$, transduction models a morphological constraint of the agent's embodiment.}, taking the full set of observable signals $\Xcal$ to the agent-accessible signals $\Ocal \subseteq \Xcal$ \citep{delchamps1990stabilizing, brockett2000quantized}. We consider two transduction functions: a linear projection $T$, such that $o = Tx$ for all $x \in \Xcal$, and the identity. A \textit{bounded agent} is thus characterized by the triple $(\Scal, \pi, \tau)$, where $\Scal$ constrains the internal representation, $\pi$ governs behavior, and $\tau$ determines the signals received from the environment\footnote{In contrast to partially-observed formulations, which map between latent and observable variables, the transduction function maps exclusively between sets of observable variables. This distinction allows us to remain in the experiential setting, and to compare the performance of agents with different input channels within the same environment and task.}. Correspondingly, the \textit{environment} is a stochastic mapping $\xi \colon \Gcal \rightarrow \Delta(\Xcal)$ over the interface $(\Acal, \Xcal)$, where $\Gcal \equiv (\Xcal\times \Acal)^*$ denotes the set of finite environment histories.

\paragraph{Objective.} RL agents adapt to maximize the occurrence of future reward, which we capture by the discounted sum $R_{t+1} + \gamma R_{t+2} + \gamma^2 R_{t+3} + \cdots$, taking $\gamma\in[0,1)$ as the discount factor. The \textit{action-value function} is the expected discounted sum following action $a$ from history $h$ in $\xi$:
\begin{align*}
    q(h,a) = \mathbf{E}_{\lambda,\xi}\left[\sum_{k=0}^{\infty} \gamma^k R_{t+k+1} \,\middle|\, H_t=h,\, A_t=a\right].
\end{align*}

\paragraph{Learning.}We consider bounded agents that learn an approximate action value $\hat{q}(o,a)\approx q(h,a)$. Here, we suppress the dependence of the internal state $s\in\Scal$, but assume $s$ carries a set of learnable parameters along with additional overhead. The agent behaves according to an $\epsilon$-greedy policy, selecting a uniform-random action with probability $\epsilon$ and otherwise selecting $a \in \argmax_{a\in\mathcal{A}}\, \hat{q}(o,a)$. 

Two learning methods are relevant to our study: Linear Q-learning \citep{bradtke1992reinforcement} and the Deep Q-Network (DQN) \citep{mnih2015human}. Linear Q-learning represents $\hat{q}$ with a weighted sum of the observation-inputs: $\hat{q}(o,a; w) \equiv w_a^\top o$, where $w\equiv\{w_a\in \Rbb^{|\Ocal|}, a\in\Acal\}$. Given a sample transition, $(o,a,r,o')$, weights $w_a$ are updated with the Q-learning rule \citep{watkins1992q}, using a constant, scalar step-size $\alpha > 0$: 
\begin{align*}
    w_a \gets w_a + \alpha [r + \gamma \max_{a'\in\Acal}\hat{q}(o',a'; w) - \hat{q}(o,a; w)]o.
\end{align*}
DQN represents $\hat{q}$ with a neural network of real-valued weights $\theta$. Network weights are updated with backpropagation to minimize the following loss over mini-batches of cached experience $\Dcal$.  
\begin{align*}
    L(\theta) = \frac{1}{2}\sum_{(o,a,r,o')\in\Dcal}[r + \gamma \max_{a'\in\Acal}\hat{q}(o',a';\theta') - \hat{q}(o,a;\theta)]^2.
\end{align*}
Following \cite{mnih2015human}, we employ a target network with separate weights $\theta'$. 

\paragraph{System Capacity.} Let $C\in\Nbb$ denote an agent's \textit{capacity}: the total internal memory available for learning. Similar to \citep{tamborski2025memory}, we adopt an operational measure of capacity proportional to the number of learnable parameters: $C \propto |w_a|$ for Linear Q-learning and $C \propto |\theta|$ for DQN, excluding the replay buffer as a constant factor.

\section{A Formalism of Externalized Memory}\label{sec:formalizing_externalization}
RL treats memory as an internal resource whose capacity is specified at design-time and generally assumed to remain constant throughout operation. In this section, we formalize what it means to externalize memory through the use of artifacts.  

\subsection{Artifacts as Memory}
Here, artifacts are features of the environment that help an agent remember its past. Examples are ubiquitous: a folded page, a string tied around a finger, a trail of footprints in the snow. We specifically formalize artifacts that affect perception \citep{heersmink2021varieties}: those in which some present observation reliably encodes information about the past, enabling an agent to recover that information through observation alone.
\begin{definition}[Artifact]\label{def:artifact}
    An \textbf{artifact} is an observation $o$, provided that for any $t$, if $O_t =o$, there exists some non-zero $t' < t $ and $o'\neq o$ such that $O_{t'}=o'$. We say $o$ is the artifact of $o'$.
\end{definition}
Our definition establishes a fine-grained condition of certainty for a single observation from the past. For a given environment $\xi$, the set of all artifacts $\Omega_\xi\subseteq \Ocal$ is the collection of observations satisfying \autoref{def:artifact}. We give a special name to environments that support this condition.\begin{definition}[Artifactual Environment]\label{def:artifactual_env}
    An environment is \textbf{artifactual} if and only if $\Omega_\xi$ is non-empty.
\end{definition}
\begin{example}[Page Keeping]
\begin{figure}
    \centering
    \includegraphics[width=\linewidth]{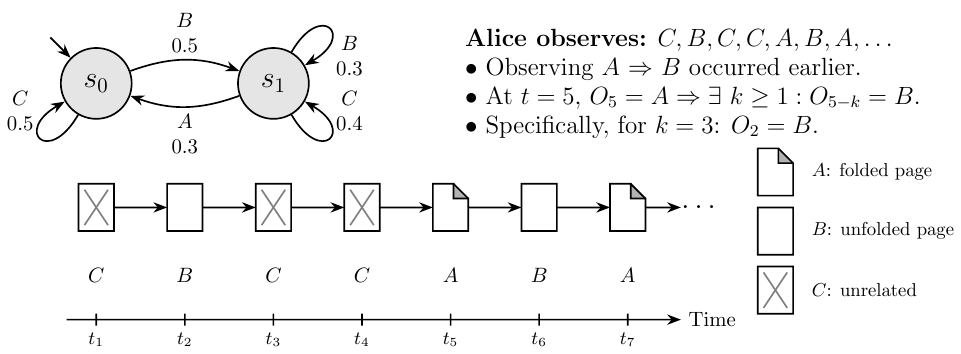}
    \caption{\textbf{Page Keeping example:} Transitions are labeled with probabilities and observations. Actions are omitted for clarity. Interaction starts from $s_0$. $A$ is an artifact of $B$, and is only observed after $s_1$. Observing $A$ at $t=5$ implies that $B$ was observed in the past: specifically at $t=2$.}
    \label{fig:single_artifact_env}
\end{figure}
Alice is an avid reader of books. Like many, she reads only a few pages at a time. Instead of remembering the page number where she stopped, she marks her place by folding the corner of the page. When she picks up the book later, she unfolds the corner and continues to read. 

This interaction can be represented by the artifactual environment pictured in \autoref{fig:single_artifact_env}. Observations indicate three basic situations where Alice sees a folded page ($A$), an unfolded page ($B$), or something unrelated ($C$). Whenever Alice observes $A$, she knows that $B$ must have occurred. Thus, in this context, a folded page serves as an artifact.  
\end{example}
The existence of artifacts can be expressed as a probabilistic property of the environment. Proofs of formal claims are provided in Section \ref{app:proofs} of the Supplement.
\begin{restatable}[]{lemma}{prob}\label{lemma:prob}
    An environment $\xi$ is artifactual if, and only if, for any $t>0$ there exist distinct observations $o,o'$, and non-zero $t' < t $ such that $\Pbb(O_{t'}=o' | O_{t}=o) = 1.$
\end{restatable}
Next, we present our main theoretical result, proving that artifacts reduce the amount of information needed to represent a history. In what follows let $\Ibb(X;Y)$ be the mutual information between two discrete random variables $X$ and $Y$.
\begin{restatable}[Artifact Reduction]{theorem}{orderreduction}\label{thm:artifact_reduction}
Let $\xi$ be an artifactual environment, and let $H$ be a history from $\xi$ containing $m>1$ observations and at least one artifact. There exists a reduced sequence $H'$ with $m-1$ observations, such that 
\begin{align*}
\Ibb(O_{t+1}; H) = \Ibb(O_{t+1}; H').
\end{align*}
\end{restatable}
The Artifact Reduction Theorem guarantees any history containing an artifact can be reduced by at least one observation. Thus, knowing $H'$ is equivalent to knowing $H$, even though $|H'|<|H|$. The reduction increases when $H$ contains multiple artifacts (\autoref{corr:multiple_reduction}); due to space constraints, this result is deferred to the Supplement. Importantly, reduction can only occur from distinct pairs of artifacts $o$ and referents $o'$. Otherwise, when multiple artifacts inform the same observation, their information becomes redundant.    


\subsection{Memory Beyond the Agent Boundary}\label{sec:externalization}
An agent is said to externalize memory if achieving a goal requires greater internal capacity in the absence of environmental artifacts than it does when those artifacts are available. To quantify this condition, our method compares performance across two settings: an artifactual environment $\xi$ and a corresponding control environment $\xi'$, defined as a copy of $\xi$ with all artifactual properties removed. We formally define this control setting and prove its existence.
\begin{restatable}[Existence of an Artifactless Copy]{proposition}{indet}
    For every artifactual environment $\xi$ and any $\epsilon \in (0,1)$, there exists a $\xi'$, called an \textbf{artifactless copy} of $\xi$, such that $\xi'$ has the same observations, actions, rewards, and transition topology, but differs in its randomness, such that for all pairs $(o,o')$, with $o\in\Omega_\xi$, and non-zero time-steps $t'<t$, we have 
    \begin{align*}
       \Pbb(O_{t'}=o' \mid O_{t}=o) \leq 1-\epsilon.
    \end{align*}
\end{restatable}
Artifactless copies model common settings where observations provide no guarantee about what occurred in the past. Mathematically, the proof shows that an artifact can be obscured by adding noise to the observation distribution of $\xi$ such that $\xi'$ contains no artifacts: $\Omega_{\xi'}=\emptyset$. 

We define the externalization of memory as a performance-matching condition on an agent's internal capacity. Let $P$ be a scalar performance measure (such as average, discounted, or total reward), and recall that an agent's capacity is a scalar $C>0$ proportional to the number of action-value parameters. Our definition applies to agents that share the same design and vary only in their capacity.   
\begin{boxdef}{Externalizes Memory}{externalized_memory}\label{def:externalized_compute}
Suppose agent $\pi$ with capacity $C$ achieves performance $P$ in an artifactual environment $\xi$. Let $\xi'$ be an artifactless copy of $\xi$. Then $\pi$ \textbf{externalizes memory} to $\xi$ if any agent $\pi'$ with the same design as $\pi$ and capacity $C' \leq C$ achieves performance $P' < P$ in $\xi'$.
\end{boxdef}
The residual $C'-C$ serves as an upper bound on the amount of externalized memory.

\section{Experiments}
In this section, we present results from three experiments. All the experiments provide evidence that RL agents externalize memory in accordance with \hyperref[def:externalized_compute]{Definition 3}. Data is gathered in a simulated domain, where agents learn to navigate while observing different spatial artifacts. In the first experiment, we consider the effect of learning in the presence of a shortest path. Here we find the strongest evidence of externalization. Our second experiment studies other artifacts of varying optimality. We find externalization is present with some but not others. In the third experiment, agents learn in a non-stationary environment, where a path is dynamically generated throughout interaction. Each experiment evaluates linear Q-learning and DQN agents over a range of capacities.  
\paragraph{Environments.}
We consider simulated domains for spatial navigation, as in \autoref{fig:base_env}, all sharing common dynamics. An agent explores a two-dimensional space with the goal of finding an unknown location. Locations come from a $13\times 13$ grid. Each grid cell emits an $8\times 8$ binary image, which only contains a small amount of salt and pepper noise \citep{boncelet2005saltandpepper}. The noise patterns provide subtle but distinct markers to identify the states. Observations are composite images of $24\times 24$ pixels, providing an allocentric view of the $3\times 3$ region of cells surrounding the agent's current location. The agent cannot observe walls; at boundary states, observations are padded by additional images to prevent the locations from appearing visually distinct from any other state. Transitions are deterministic and occur along the four cardinal directions. When the agent reaches the goal, it is rewarded with a bonus of +1. For all other transitions, it receives a zero-valued reward. Episodes terminate when the agent reaches the goal. Subsequently, the agent resets to the starting location to begin a new episode. Additional details are provided in \autoref{sec:environment_details} of the Supplement.

\subsection{Methodology}
All three experiments follow a similar methodology. Learning performance is compared across two settings where the observability of an artifact differs. In one, an agent observes empty space with no visible artifacts (No Path, \autoref{subfig:no_path}), and in the other a fixed artifact is observable (e.g.~\autoref{subfig:opt_path}). We consider various spatial paths as artifacts, and we compare performance across these settings for a range of system capacities.   

Performance is quantified with both average and total reward. Average reward measures an agent's reward rate at any given time $t>0$; we calculate this as $\frac{1}{t}\sum_{n=1}^t r_n$ and note that is only achieved in the limit, as $t$ approaches infinity. Total reward provides an aggregate measure summarizing the performance across an agent's entire lifetime. The total reward over $N$ time steps is $\sum_{n=1}^Nr_n$. 

Following \hyperref[def:externalized_compute]{Definition 3}, externalization is established by the relative performance of two agents: $\pi$ and $\pi'$. Suppose $\pi$ learns when artifacts are present, with capacity $C$ and total reward of $P$. Separately, let $\pi'$ be a learner restricted to the No Path setting, with capacity $C'$ and total reward $P'$. Our experiments test the following condition.
\begin{condition}{Effective Memory Externalization}{externalization}
Whenever $P \geq P'$ and $C < C'$, we conclude that $\pi$ has effectively externalized memory in accordance with \hyperref[def:externalized_compute]{Definition 3}.
\end{condition}

We make thirty independent observations of performance, corresponding to different random initializations of the agent's parameters. We use \cref{condi:externalization} to define the alternative hypothesis of a one-sided statistical test (\autoref{app:stat_tests}). We reject at the 0.05 level then conclude there is sufficient evidence for memory externalization. We choose the most performant step-size for each capacity from a uniform grid search and report statistics from a separate evaluation. See the Supplement for additional details.

We consider linear Q-learning agents that can have 16, 64, 256, 400, or 576 weights $w_a$. These correspond to images of size $4\times 4, 8\times 8, 16\times 16, 20\times 20, 24\times 24$, respectively. A capacity constraint is imposed on the input channel, through the transduction function, defined as a projection that selects the sub-image centered on the agent's location. 

DQN agents use fully-connected ReLU networks of two or three layers. Specifically, we consider the set of networks produced by $\{2,3\}\times\{4,8,16,32\}$, where the first set specifies the number of layers and the second is the number of hidden units per layer. The final layer for all networks has an output dimension equal to the number of actions. The transduction function is the identity.
\begin{figure}[h]
    \centering
    \begin{subfigure}[t]{0.49\textwidth}
        \centering
        \includegraphics[width=\textwidth]{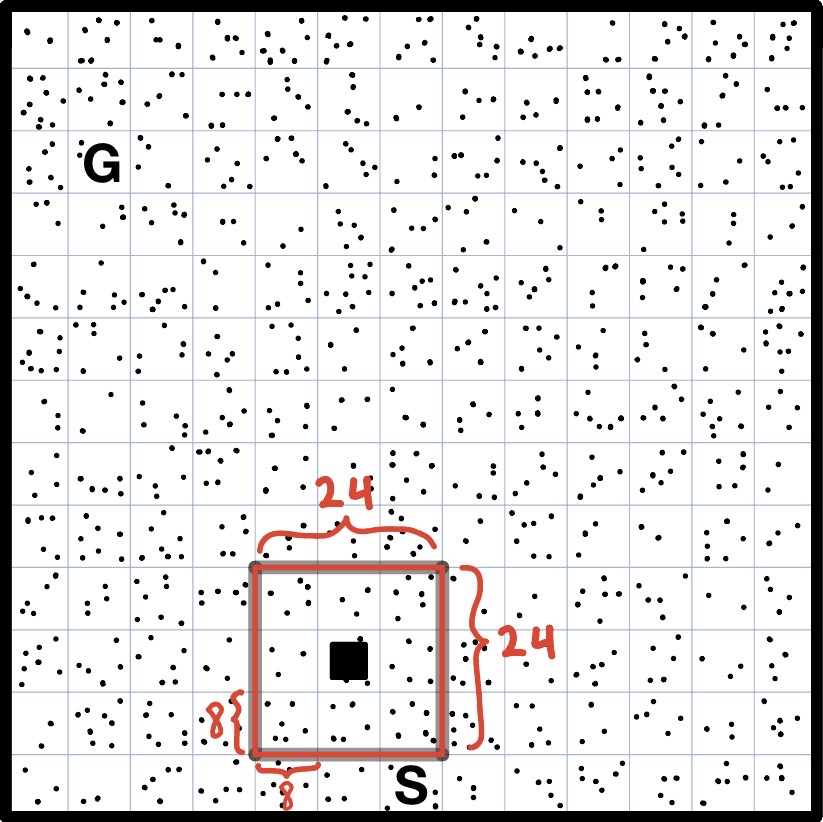}
        \caption{No Path}
        \label{subfig:no_path}
    \end{subfigure}
    \begin{subfigure}[t]{0.488\textwidth}
        \centering
        \includegraphics[width=\textwidth]{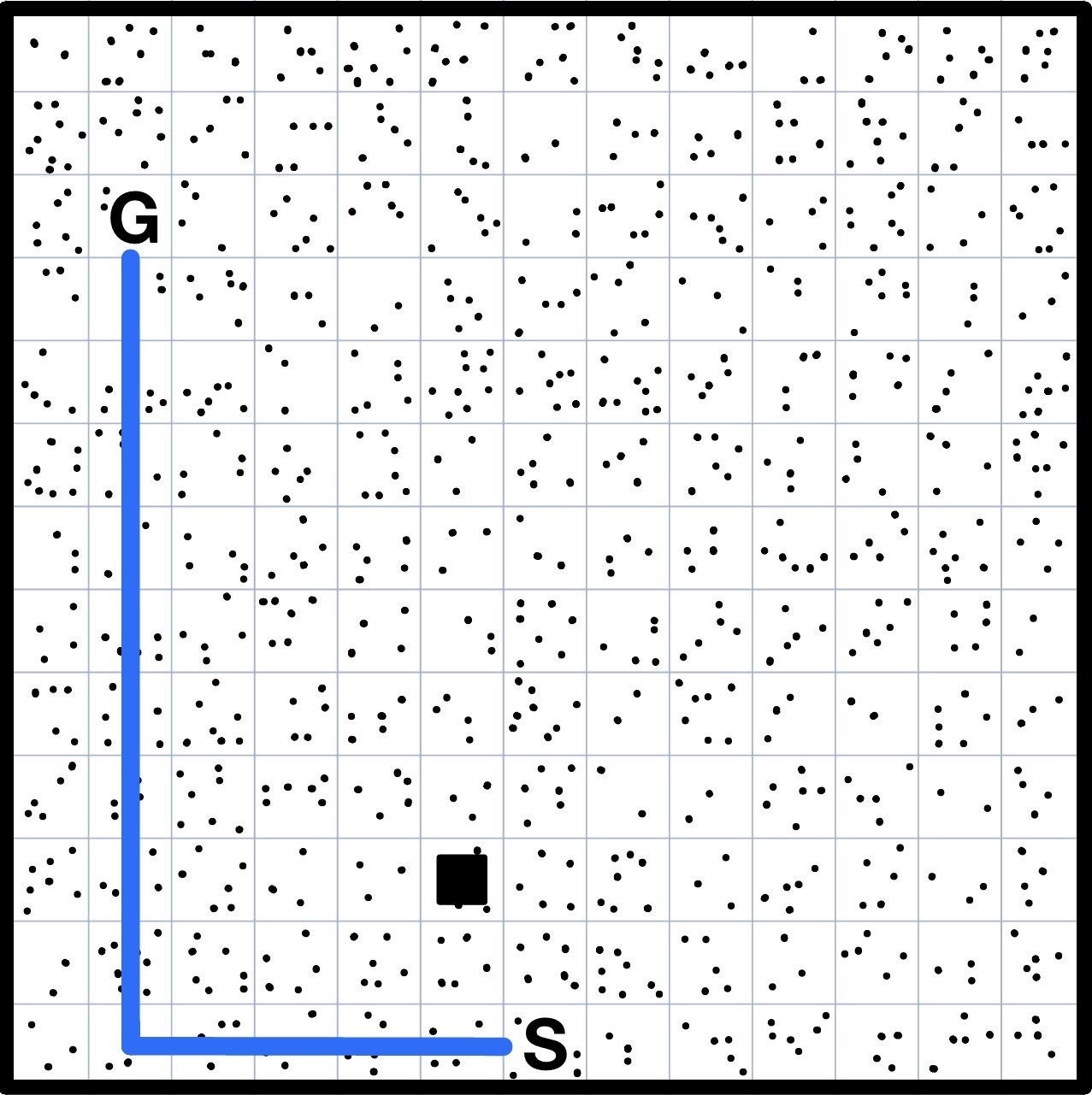}
        \caption{Optimal Path}
        \label{subfig:opt_path}
    \end{subfigure}
    \caption{The base environment used throughout experiments.}
    \label{fig:base_env}
\end{figure}
\begin{figure}[h]
    \centering
    \includegraphics[width=0.3\linewidth]{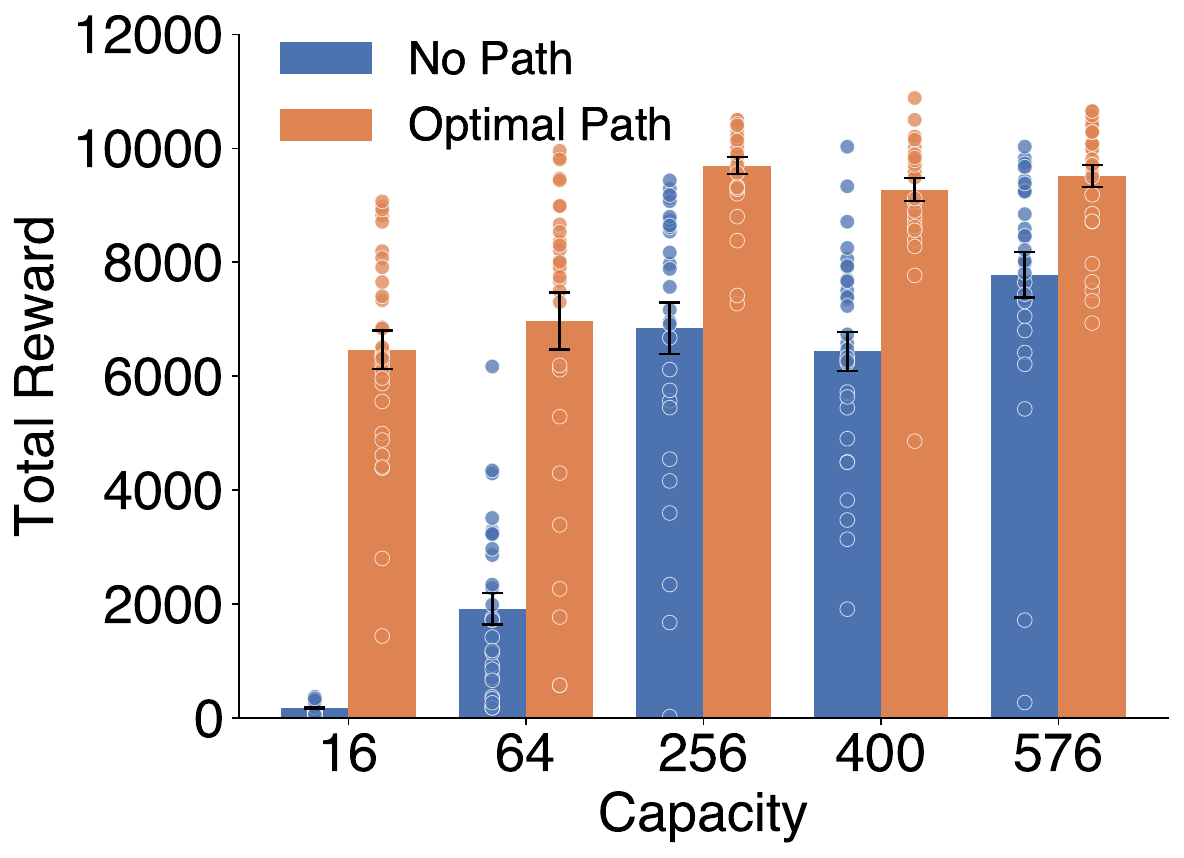}
    \includegraphics[width=.6\linewidth]{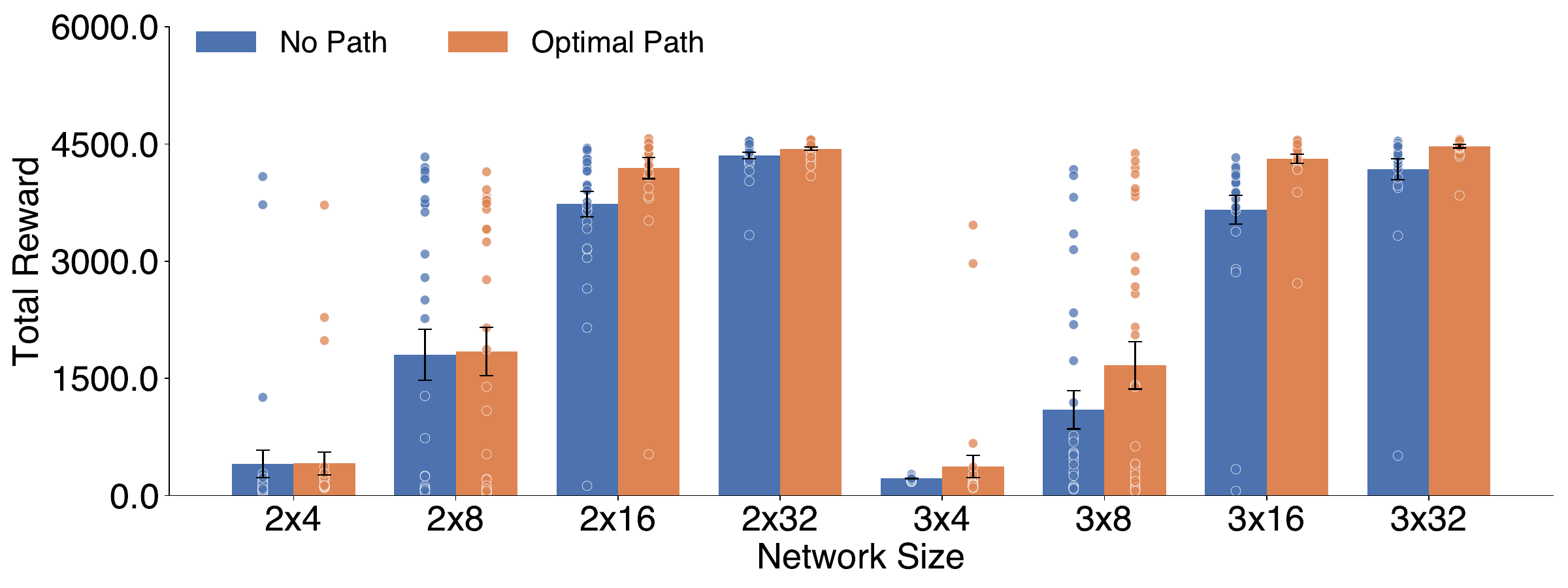}
    \caption{\textbf{Observing a spatial path reduces the necessary capacity to navigate.} Averages of total reward for Linear Q-learning (left) and DQN (right) are shown, along with standard-error bars. We find that learning in presence of a shortest path improves performance for nearly every agent and capacity. Many improvements are statistically significant (see Supplement \ref{app:stat_tests}).  }
    \label{fig:e1-total_reward}
\end{figure}
\begin{figure}[h]
    \centering
    \includegraphics[width=0.49\linewidth]{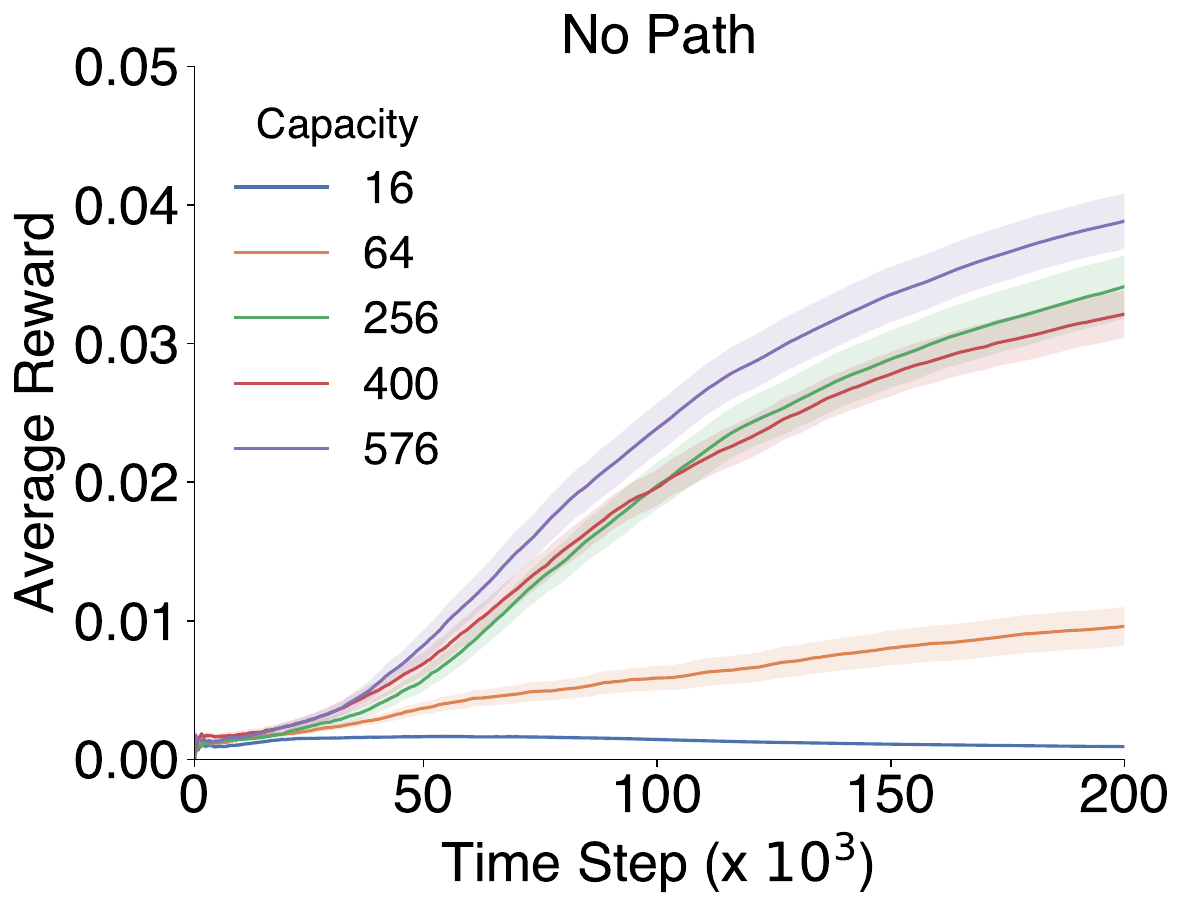}
    \includegraphics[width=0.49\linewidth]{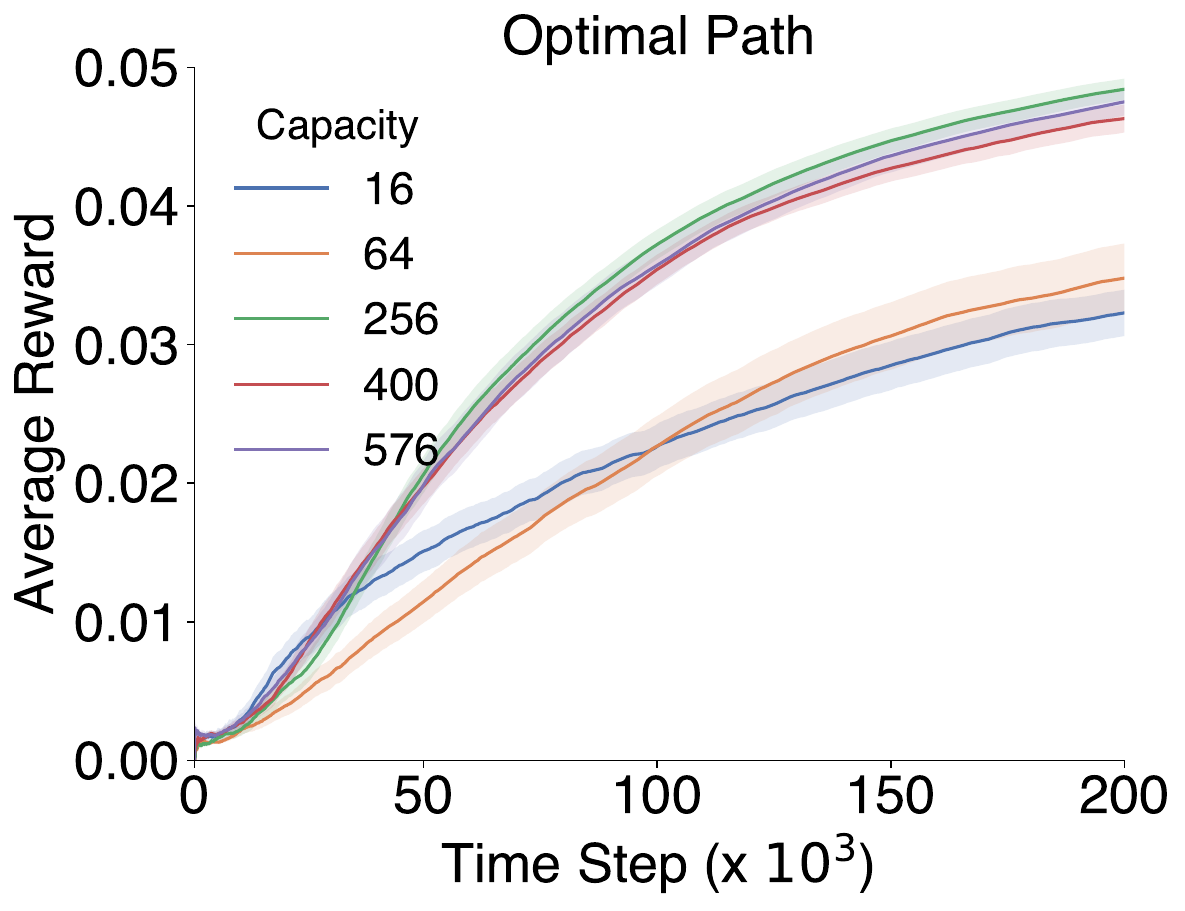}\\
    \includegraphics[width=0.49\linewidth]{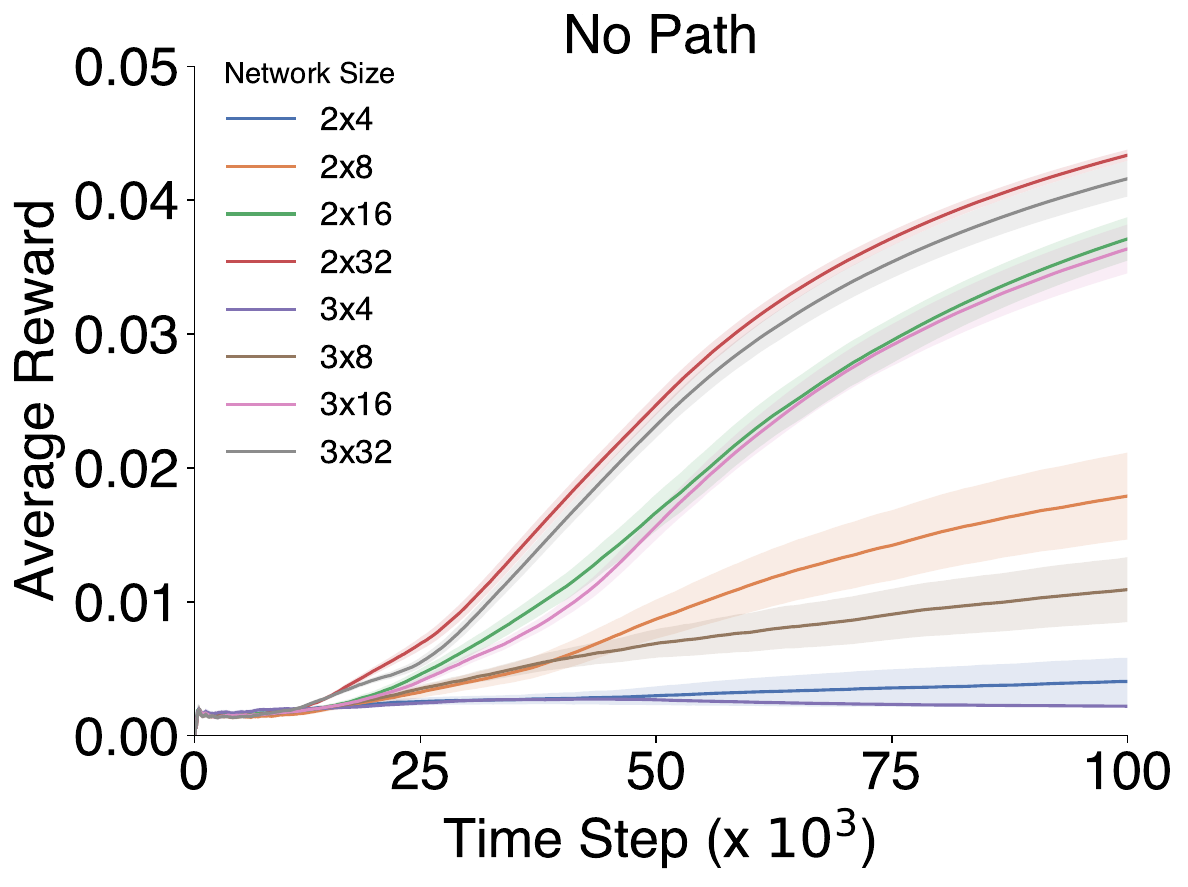}
    \includegraphics[width=0.49\linewidth]{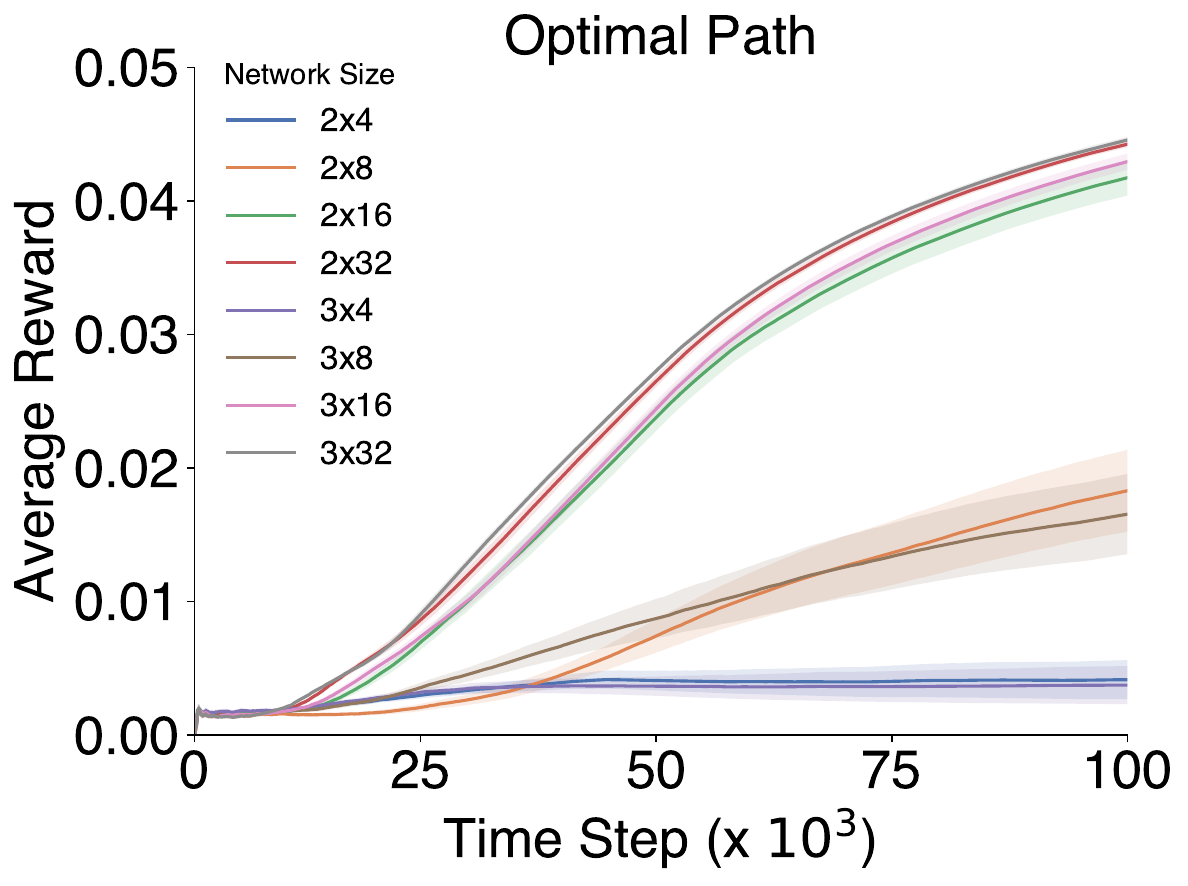}
    \caption{\textbf{Performance improves when agents observe the shortest path:} Average reward tends to increase when the shortest path is visible. This can be observed for nearly every capacity of Linear-Q and DQN; it appears most significant for higher capacity systems, but also has a stark affect in the low capacity regime. Averages and standard error regions are computed with 30 seeds. }
    \label{fig:e1-avg_reward}
\end{figure}
\subsection{Learning in the Presence of a Minimum-length Path}
This experiment compares performance across two domains: one in which the shortest path is visible (Optimal Path, \autoref{subfig:opt_path}), and one in which no path is visible (No Path, \autoref{subfig:no_path}). 

\autoref{fig:e1-total_reward} shows plots of total reward for all agents and capacities. We find \cref{condi:externalization} is satisfied in several cases. Consider the No Path linear agent with $C'=64$ weights per action-value and the Optimal Path agent with $C=16$ weights; we observe $P > P'$ while $C < C'$. In other words, the necessary capacity to achieve comparable performance is reduced when the optimal path is visible. The amount of memory externalized is at most $48=C'-C$ weights per action-value. Interestingly, this occurs below the theoretical minimum of $169=13^2$ dimensions, as predicted by the rank of the matrix with every vectorized image. In this case, the agent can represent a relatively simple policy that indexes on the optimal path; in particular, conjunctions of a few horizontal and vertical pixels suffice to represent a reliably rewarding action. The effect is also apparent with other capacities and in the deep RL setting, e.g. DQN with $3\times 16$ and $3\times 32$. Due to space constraints, we defer further analysis to Supplement \ref{app:exp1}.

Average reward is plotted in \autoref{fig:e1-avg_reward}. This quantity shows how the shortest path affects learning performance throughout an agent's lifetime. The effect is stark in the low-capacity regime; without a path, linear agents are unable to reach the goal with less than 256 weights; with a path, they reliably reach the goal. DQN also experiences a performance boost with near uniformity across the range of the considered network sizes. All the points at which externalization is significant are tabulated in Supplement \ref{app:experimental_details}.

\subsection{Learning in the Presence of Other Fixed Artifacts}
This experiment repeats the previous performance comparison with four additional artifacts. Now we ask whether the impact on performance varies with the quality of behavior expressed by a given path. In addition, we evaluate learning in the presence of a non-behavioral artifact: a set of geometric landmarks. These help us to understand if externalization is possible without an overt behavioral signal, like a goal-directed path. Each artifact is listed below and illustrated in \autoref{fig:e2-artifacts}.
\begin{itemize}
    \item Random: a path generated with uniform random actions (\autoref{subfig:random}).
    \item Suboptimal: a path that reaches the goal with a few more steps than optimal (\autoref{subfig:suboptimal}).
    \item Misleading: a path that steers toward the goal then veers off (\autoref{subfig:misleading}). 
    \item Landmarks: geometric structures of various sizes, shapes, and locations (\autoref{subfig:landmarks}).
\end{itemize}
\begin{figure}[H]
    \centering
    \begin{subfigure}[t]{0.24\linewidth}
        \centering
        \includegraphics[width=\linewidth]{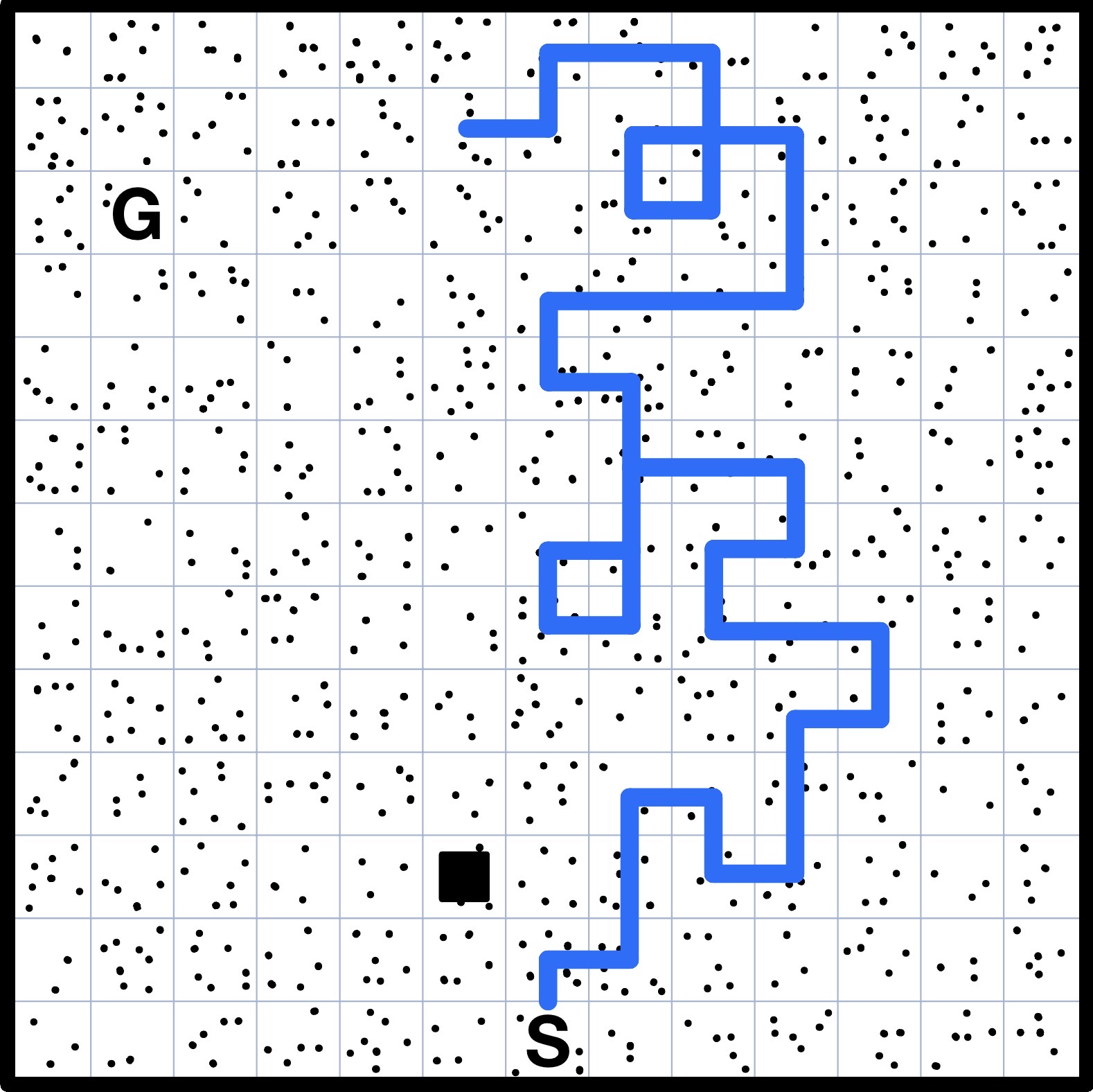}
        \caption{Random}
        \label{subfig:random}
    \end{subfigure}
    \hfill
    \begin{subfigure}[t]{0.24\linewidth}
        \centering
        \includegraphics[width=\linewidth]{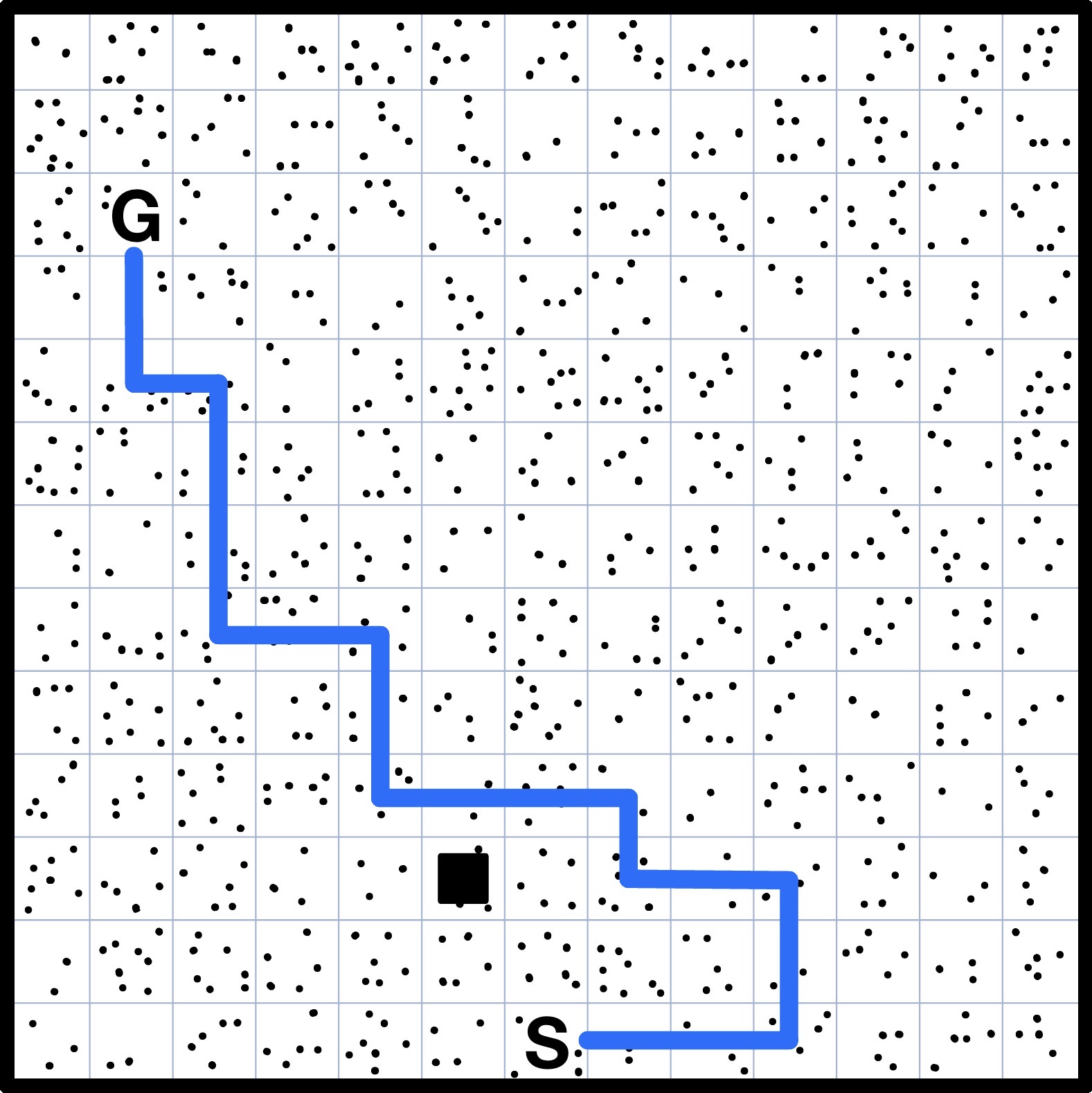}
        \caption{Suboptimal}
        \label{subfig:suboptimal}
    \end{subfigure}
    \hfill
    \begin{subfigure}[t]{0.24\linewidth}
        \centering
        \includegraphics[width=\linewidth]{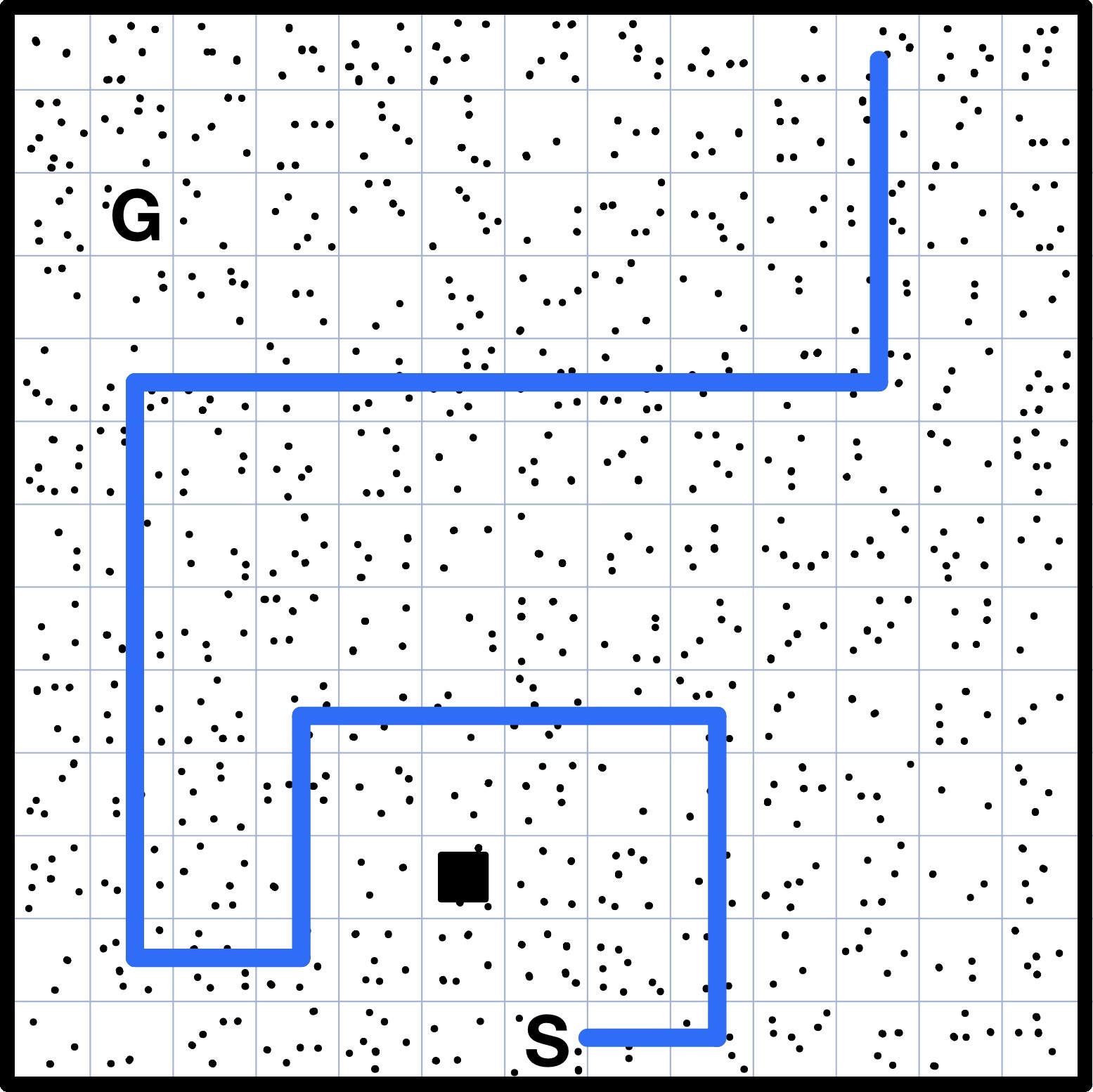}
        \caption{Misleading}
        \label{subfig:misleading}
    \end{subfigure}
    \hfill
    \begin{subfigure}[t]{0.24\linewidth}
        \centering
        \includegraphics[width=\linewidth]{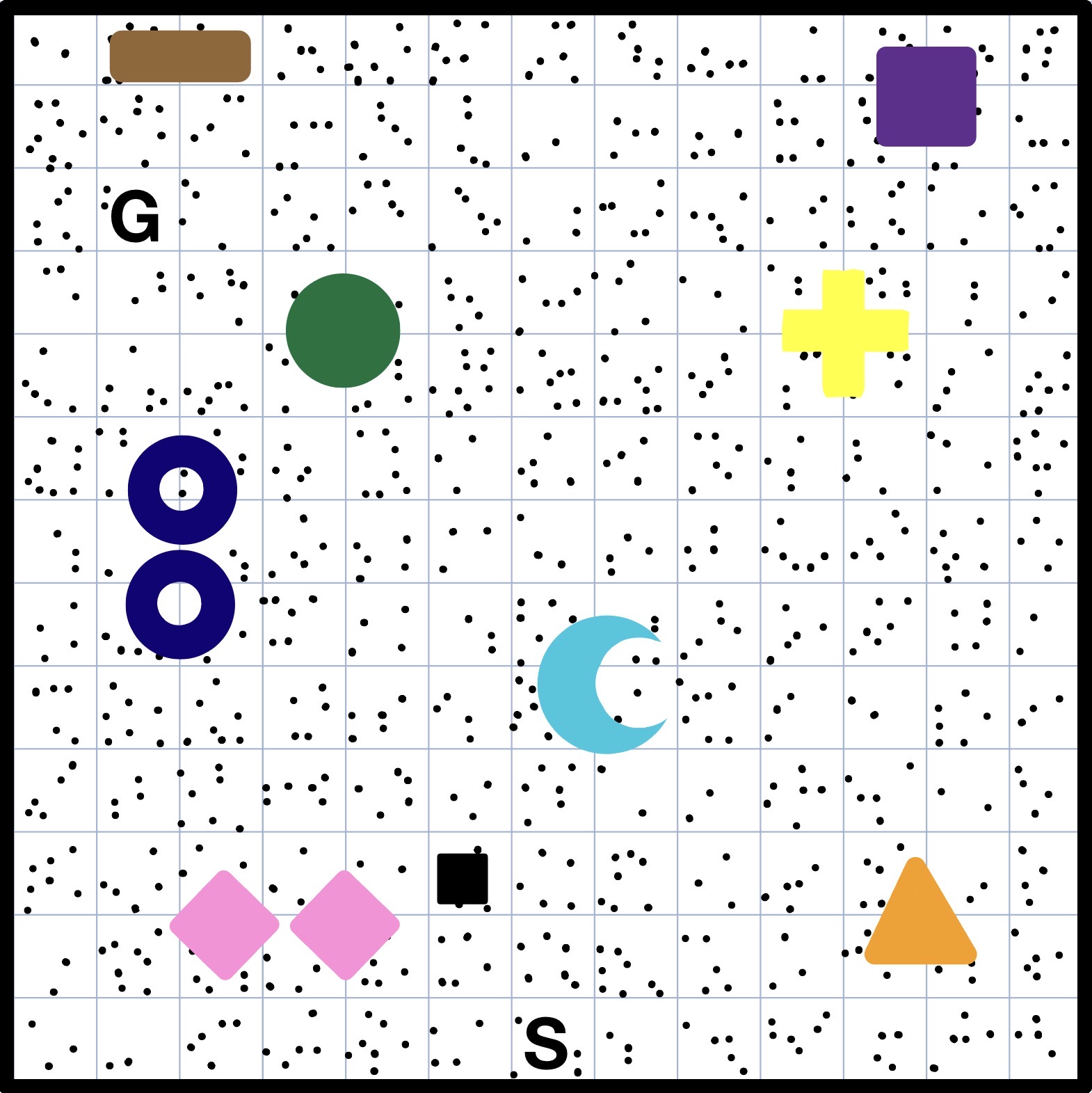}
        \caption{Landmarks}
        \label{subfig:landmarks}
    \end{subfigure}
    \caption{Environments considered for learning in the presence of other fixed artifacts.}
    \label{fig:e2-artifacts}
\end{figure}
\begin{figure}[h]
    \centering
    \includegraphics[width=\linewidth]{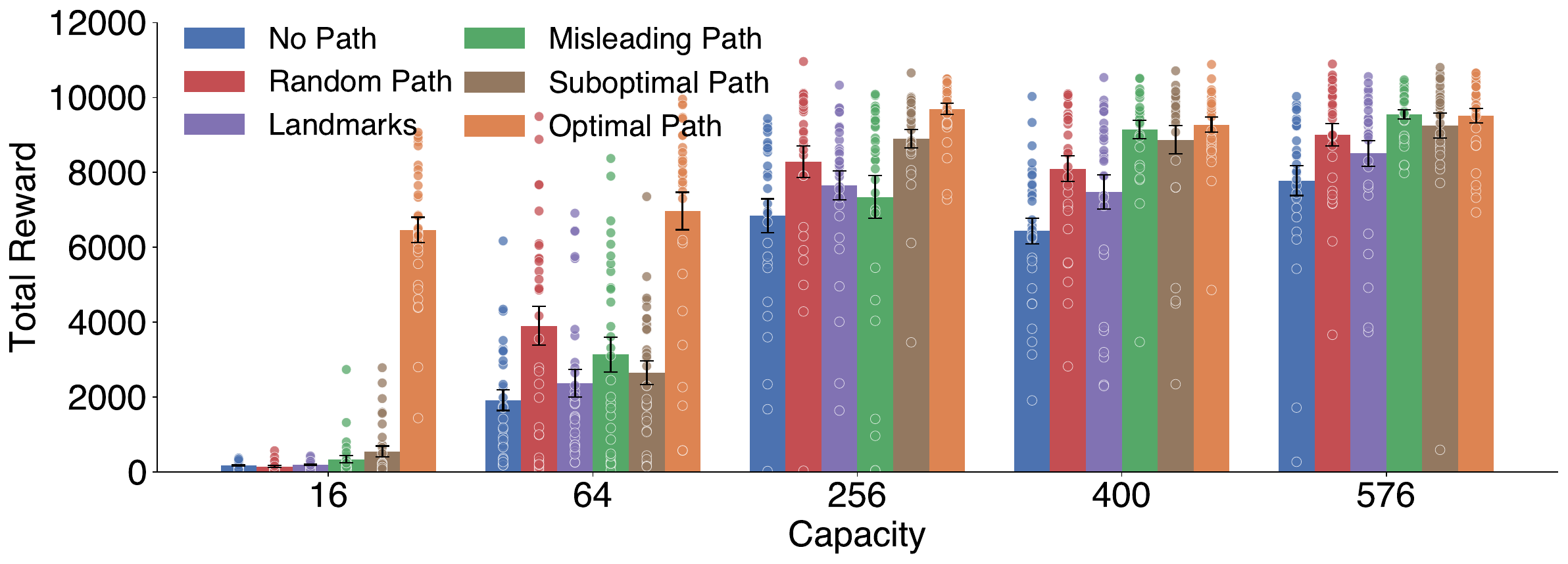}
    \includegraphics[width=\linewidth]{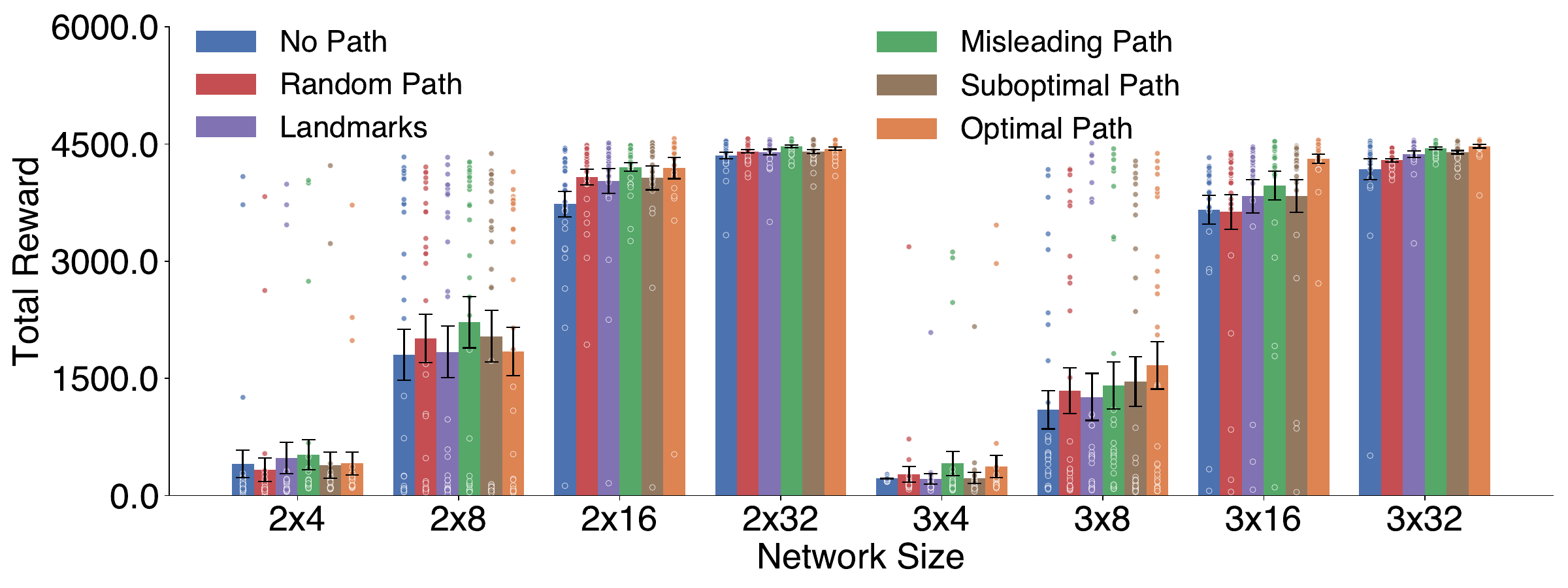}
    \caption{\textbf{Externalization arises with other fixed artifacts:} Average total reward is observed for three paths and one set of geometric landmarks. We find evidence of externalizing memory across all artifacts and with the following number of instances for linear agents: Suboptimal (3), Misleading (2), Random (2), Landmarks (1). For DQN, Suboptimal (2), Landmarks (2), Random (1), Misleading (0). Each bar presents an average and standard-error from the 30 seeds shown.}
    \label{fig:e2-total_reward}
\end{figure}

\autoref{fig:e2-total_reward} shows plots of total reward for every artifact and across the previous range of capacities. We find linear agents externalize memory with all four artifacts, though to varying degrees. For instance, Random Path with $C=256$  vs No Path with $C' =400$. Such examples rule out the simple hypothesis that agents merely follow paths, because performance would otherwise be comparable to No Path. Our Landmarks baseline provides further support for this interpretation: specifically for $C=256$ and $C'=400$. The same story applies to DQN, which we find externalizes memory with the Suboptimal, Random, and Landmarks baselines. We evaluate an expanded range of capacities for the linear agent in \autoref{fig:total_reward_lines}, provide plots of average reward and tabulated results in \autoref{app:experimental_details}.

\subsection{Learning in the Presence of a Dynamic Path}\label{sec:exp3}
This experiment moves to a more naturalistic setting in which the artifact is generated by an agent's own behavior (\autoref{fig:e3-total_reward}). Specifically, as the agent moves through the environment, a noisy path appears at visited locations and gradually fades until it is indistinguishable from the background. Given the non-stationary nature of this environment, we restrict our analysis to linear Q-learning, which is capable of tracking non-stationarity; conventional DQN, relying on a replay buffer, was unable to learn a performant policy in such settings. We repeat our test for externalized memory.

\autoref{fig:e3-total_reward} shows plots of total reward. \cref{condi:externalization} is satisfied at $C=256$ weights per action value. Weaker evidence of externalization occurs for $C=400$ at the 0.11 level. Similar to our other experiments, the presence of an artifact path appears to uniformly increase total reward across the range of capacities. Average reward curves and tabulated results are provided in \autoref{app:experimental_details}.

\begin{figure}[]
    \centering
    \includegraphics[width=0.33\linewidth]{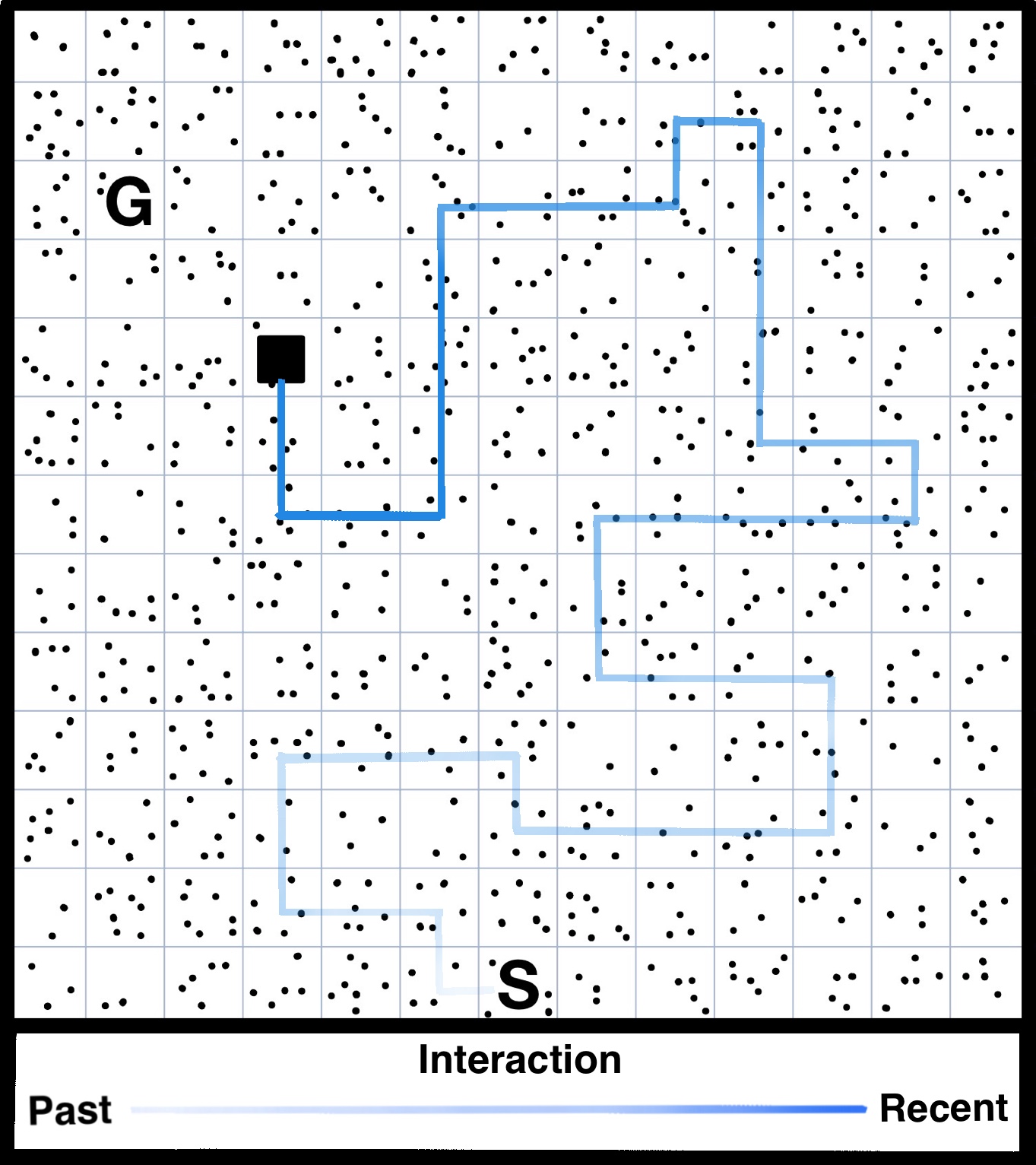}
    \includegraphics[width=0.5\linewidth]{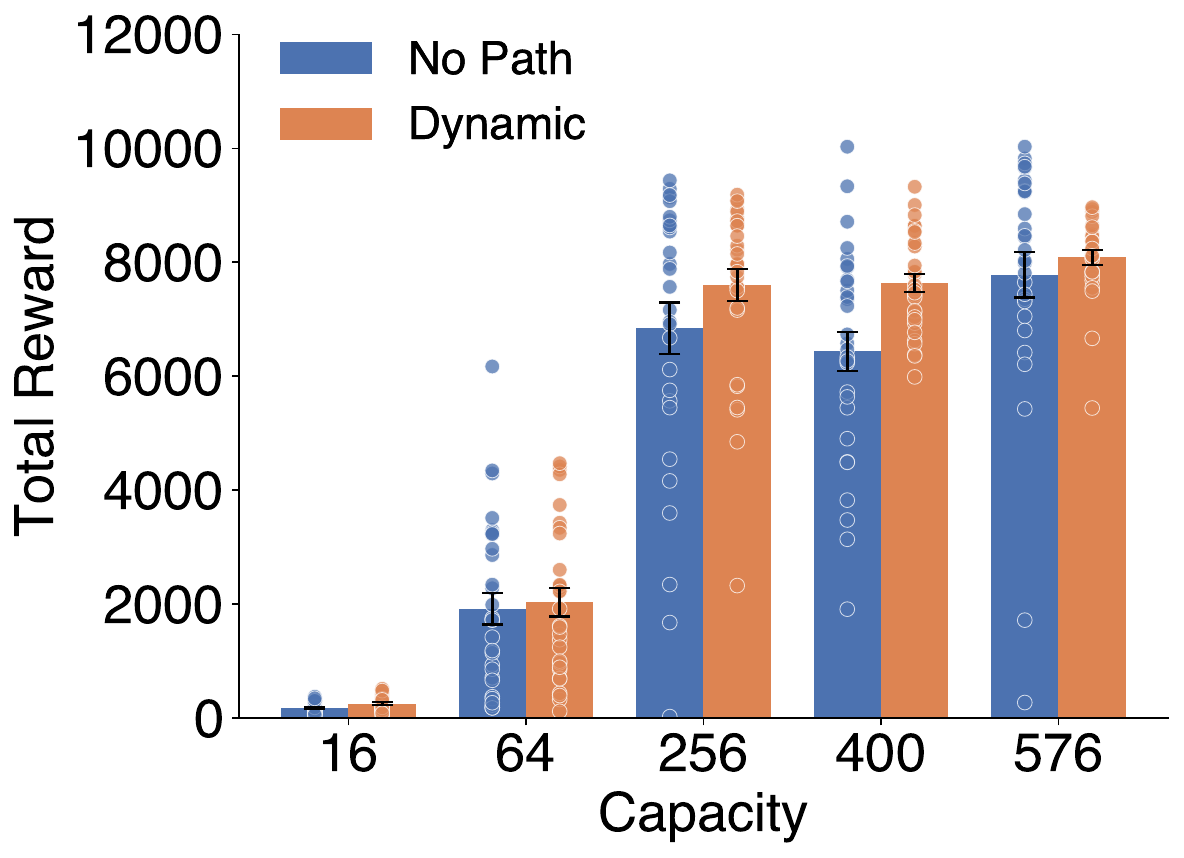}
    \caption{\textbf{Externalization with a dynamic path:} The left panel illustrates the environment, in which the current policy generates a path vanishing over time. We observe performance uniformly increase over nearly all capacities and \cref{condi:externalization} satisfied for $C=256$. }
    \label{fig:e3-total_reward}
\end{figure}

\section{Discussion}\label{sec:criteria}

\paragraph{Artifactual Environments and Classical Memory.} Some readers may wonder how artifactual environments express the basic encode-store-retrieve functionality of classic memory models, such as those \cite{klein2015memory} and \cite{sep-memory} describe. Artifacts meet the first requirement by virtue of being observations that provide information about an agent’s history (\autoref{def:artifact}). In the environments we study, information is anchored to locations in space; data is written upon first visiting a location, then later retrieved by returning to it. In this way, spatial artifacts function as memory records with an access protocol dictated by the environment and interface $(\Ocal,\Acal)$. 

\paragraph{Artifacts as Situated Memory.} Situated accounts of memory enrich the classical model by grounding memory's purpose in service of decision-making \citep{clark1998extended}. In regards to external memory, \cite{michaelian2012external} argues that a model must satisfy certain criteria to capture the essential functionality of natural memory. Michaelian requires agents have constant access to an information-bearing resource and some process to determine the information’s relevance. Following \cite{sims2022externalized}, we summarize these in three points:
\begin{enumerate}
    \item Survival relevant: a memory should bring positive value to decision making.
    \item Susceptible to change: memories are mutable.
    \item Selection: a memory's relevance is determined through some selection process.  
\end{enumerate}
The first requirement underscores the cost of storage. As \cite{sims2022externalized} put it: memory must be ``worth its weight in terms of long-term fitness benefits.'' The second point preserves the basic functionality of the encode-store-retrieve model, while the third requires the existence of a process to determine a memory's relevance in a given scenario.  

\cite{sims2022externalized} use these desiderata to argue the spatial trails left behind by slime mold \citep{reid2012slime} function as external memory. We similarly argue that the artifacts from our empirical study satisfy these desiderata. In support of (1), note that an artifact's value is immediately apparent from total reward (see Figures  \ref{fig:e1-total_reward}, \ref{fig:e2-total_reward}, and \ref{fig:e3-total_reward}); agents in artifactual environments consistently accumulate more reward than in artifactless environments. Support for (2) follows directly from the encode-store-retrieve model, to which artifacts from the Dynamic Path conform (see \autoref{fig:e3-total_reward}). Fixed artifacts provide read-only information and yet still produce an external memory effect, suggesting that reading is more fundamental than writing when learning to navigate and the desiderata may need further refinement. Support for (3) comes from the learning process. Through repeated credit assignment, policies that read and write on each step gradually improve and bias navigation toward goal-relevant locations. With these properties in place, we conclude the artifacts from our study support the same arguments and conclusions as previous accounts of external memory. 

\paragraph{Unintentional Memory.} Our experiments demonstrate that an agent can read and write information to the environment without any explicit objective directing it to do so. In each experiment, agents were given a standard navigation objective: a sparse reward signal providing a bonus for reaching the goal, but no explicit incentive to follow a path. Still, we observe path-following behavior, as performance would otherwise match the No Path baseline. Moreover, in the Dynamic Path environment (\autoref{fig:e3-total_reward}), agents record traces of their previous interactions without explict direction. These artifacts go on to guide future behavior. Remarkably, this form of emergent behavior requires no explicit design or human involvement; it emerges naturally as a consequence of reinforcement learning in a sufficiently complex environment.

\paragraph{Implications for Agent Design.} A popular line of research pursues designs whose performance scales with the number of trainable parameters. This direction is motivated by milestone developments \citep{silver2016mastering, brown2020language, fawzi2022discovering}, historical arguments for the primacy of computation \citep{sutton2019bitter}, and empirical findings of power-law relationships between system capacity and performance \citep{kaplan2020scaling}. Our results hint at another path: rather than scaling system resources, performance gains may instead arise from environments that coevolve with the agent. It is possible that current designs are already sufficient for competent, human-level performance, but require judicious pairing with an appropriate environment to scaffold problem solving \citep{sterelny2010minds}. More work is needed to understand the laws governing the relationship between environment and agent design.

\paragraph{Limitations.}A primary aim of this paper was to offer a purely observational account of external memory. However, our formalism raises the question of whether artifacts can be defined to encompass action as well. Such an extension would be relatively straightforward, though validating it would require considering artifacts that encode previous actions. In a gridworld setting, this amounts to providing directional information, which our bidirectional spatial paths do not capture. 

Some may view the total certainty conveyed by artifacts as a limitation, arguing that a more general theory would incorporate a stochastic condition. We agree that not all forms of external memory need to guarantee knowledge transfer with total certainty. Yet even this simple assumption suffices to prove that artifacts reduce the information needed to represent the past (\autoref{thm:artifact_reduction}). Characterizing what guarantees can be made when artifacts provide only partial information remains an interesting direction for future work.




\section{Related Work}
\textbf{Memory in RL.} Classical memory is any state or process resulting from the sequence of encoding, storing, and retrieving information \citep{klein2015memory, gazzaniga2009cognitive, bernecker2017routledge}. In RL, the term `memory' refers to several distinct concepts and processes, such as the agent's functional dependence on the past \citep{littman1994memoryless, singh1994learning, mnih2015human, abel2023convergence, Icarte2020-bd}, the representational capacity of its internal state \citep{dong2022simple,sutton2018reinforcement, tamborski2025memory}, even the specific use of recurrent architectures \citep{hausknecht2015deep} or replay buffers \citep{lin1992self, schaul2016prioritized}. In other contexts, the term applies more broadly to any learnable parameters or writable storage \citep{oh2016control, khan2017memory}. Our work studies memory as it pertains to the representational capacity of a value function. We argue that a certain, fixed amount of capacity is required to achieve a goal, and when this memory footprint is reduced in the presence of artifacts, the deficit must be compensated for by the environment (Figures \ref{fig:e1-total_reward}, \ref{fig:e2-total_reward}, \ref{fig:e3-total_reward}).

\paragraph{Artifacts are Episodic Memories.} Memory can be taxonomized by the content it provides. Model-free RL agents acquire \textit{procedural memory}, also known as habit memory \citep{russell1921analysis, sep-memory}. We study procedural memory in the context of a value function. In our work, the content of an artifact couples with the weights of a value function to inform the agent how to navigate. Recall an environment is artifactual when a given observation $o$ determines one of the past $\Pbb(O_{t'}=o' \mid O_{t}=o) = 1$. Thus, on their own, artifacts constitute a kind of \textit{episodic memory}: knowledge of ones personal past \citep{Tulving1972}. Several works argue for the centrality of episodic memory in natural agents \citep{gershman2017reinforcement, lengyel2007hippocampal} and in AI systems \citep{blundell2016model, hu2021generalizable, pritzel2017neural, lin2018episodic, zhu2020episodic}. Thus, artifacts are episodic memories that reduce procedural memory.

\paragraph{The Memory of Artificial Agents.}
The notion of individuated memory has long structured computational models of agency. Early models treat agents as unified systems with centralized memory encoding beliefs about the world \citep{newell1990unified, anderson1993rules, rao1995bdi}. Later models reconceive agents as networks of reactive units \citep{brooks1986robust, brooks1991intelligence, arkin1998behavior}. In these architectures, memory is distributed but remains individuated: there is no notion of memory existing outside the agent boundary. Distribution across units reflects an internal architectural choice rather than any dissolution of that boundary. Reinforcement learning agents, similarly, are bounded systems that acquire knowledge through trial and error \citep{kaelbling1998planning, sutton2022quest, dong2022simple, abel2023convergence}. While committing to few assumptions about centralization, they maintain localized memory in the form of value functions, state representations, world models, and other sources of computational overhead.

\textbf{Relation to Stigmergy.} Stigmergy is a mechanism of behavioral coordination arising from interactions between a decision-maker and artifacts in their environment \citep{heylighen2015stigmergy, thierry1995joint, ricci2006cognitive}. The concept of stigmergy was introduced by \cite{grasse1959reconstruction} to explain the highly coordinated behavior of termites. Traditionally, stigmergy research has focused on the self-organization of many simple agents operating under fixed policies. A canonical example is the formation of pheromone trails, which guide ants to food sources \citep{wilson1962chemical, sumpter2003nonlinearity}. \autoref{def:artifactual_env} formalizes a type of environment that bears resemblance to those studied in stigmergic research. In such settings, we study individual RL agents adapting to the context afforded by spatial artifacts. \cite{mariomartin-thesis} and \cite{peshkin1999learning} study RL in stigmergic settings, both assuming that the environment exposes explicit memory states for agents to manipulate through action. In contrast, we study settings where memory effects are implicitly realized through environment dynamics and the conventional, unprivileged sensory stream.

\paragraph{The Artifacts of Other Agents.}
The setting from our fixed-artifact experiments is similar to that studied by \cite{Borsa2017-ze}, in which a deep RL agent learns from observing another agent without explicitly modeling it or having access to its internal state. The fixed artifacts we consider could plausibly originate from another agent's behavior. In our dynamic-path experiments, by contrast, agents learn from traces of their own prior behavior. Importantly, we focus on how behavioral artifacts affect memory, and therefore restrict our analysis to architectures that condition only on the current observation; \cite{Borsa2017-ze}, by contrast, employs recurrent architectures.


\section{Conclusion}
We investigate the relationship between an RL agent's computational resources and its environment, providing evidence that, in artifactual settings, the environment can serve as a form of external memory. This effect holds across multiple capacities, algorithms, and environments. We formalize artifactual environments and prove that such settings afford a reduction in the number of observations required to represent a history. We further argue that the externalized memory observed in our experiments satisfies the criteria established by prior accounts of external memory. Together, our results suggest that a memory process is not confined to one side of the conventional agent-environment boundary; rather, its data and functionality can cut across this boundary and reside in the environment.

Future work can progress naturally along two directions: agent design and agent-environment relationships. A concrete next step would be to investigate whether agents can adapt their capacity to the presence of artifacts---to modulate plasticity between different processes. Our work studied unintentional externalization, so a natural question is whether agents can intentionally generate artifacts from which they later benefit. One might also formalize artifacts differently; we studied only artifacts that determine a single past observation, but in principle artifacts are richer structures that may determine multiple past observations. More broadly, we anticipate that further work in this area could reveal principled ways to exploit the environment as a substitute for explicit internal memory, and shed additional light on what memory means for artificial agents.

\appendix


\subsubsection*{Acknowledgments}
\label{sec:ack}
The authors would like to thank Will Dabney, Shruti Mishra, Joseph Modayil, and Matt Taylor for their comments on an early draft of this paper, and the members of the Openmind Research Institute for thoughtful discussions that guided this project.


\bibliography{ref}

\begin{thebibliography}{71}
\providecommand{\natexlab}[1]{#1}
\providecommand{\url}[1]{\texttt{#1}}
\expandafter\ifx\csname urlstyle\endcsname\relax
  \providecommand{\doi}[1]{DOI: #1}\else
  \providecommand{\doi}{DOI: \begingroup \urlstyle{rm}\Url}\fi

\bibitem[Abel et~al.(2023)Abel, Barreto, van Hasselt, Van~Roy, Precup, and Singh]{abel2023convergence}
David Abel, Andr{\'e} Barreto, Hado van Hasselt, Benjamin Van~Roy, Doina Precup, and Satinder Singh.
\newblock On the convergence of bounded agents.
\newblock \emph{arXiv preprint arXiv:2307.11044}, 2023.

\bibitem[Anderson(1993)]{anderson1993rules}
John~R. Anderson.
\newblock \emph{Rules of the Mind}.
\newblock Lawrence Erlbaum Associates, Hillsdale, NJ, 1993.
\newblock ISBN 978-0805812343.

\bibitem[Arkin(1998)]{arkin1998behavior}
Ronald~C. Arkin.
\newblock \emph{Behavior-based Robotics}.
\newblock The MIT Press, Cambridge, MA, 1998.

\bibitem[Bernecker \& Michaelian(2017)Bernecker and Michaelian]{bernecker2017routledge}
Sven Bernecker and Kourken Michaelian (eds.).
\newblock \emph{The Routledge Handbook of Philosophy of Memory}.
\newblock Routledge, New York, 1st edition, 2017.
\newblock ISBN 9781138909366.

\bibitem[Bickel \& Doksum(2015)Bickel and Doksum]{bickel2015mathematical}
Peter~J. Bickel and Kjell~A. Doksum.
\newblock \emph{Mathematical Statistics: Basic Ideas and Selected Topics, Volumes I--II Package}.
\newblock Chapman and Hall/CRC, 1st edition, 2015.
\newblock ISBN 9781498740319.

\bibitem[Blundell et~al.(2016)Blundell, Uria, Pritzel, Li, Ruderman, Leibo, Rae, Wierstra, and Hassabis]{blundell2016model}
Charles Blundell, Benigno Uria, Alexander Pritzel, Yazhe Li, Avraham Ruderman, Joel~Z. Leibo, Jack Rae, Daan Wierstra, and Demis Hassabis.
\newblock Model-free episodic control.
\newblock \emph{arXiv preprint arXiv:1606.04460}, 2016.

\bibitem[Boncelet(2005)]{boncelet2005saltandpepper}
Charles~G. Boncelet.
\newblock Image noise models.
\newblock In Al~Bovik (ed.), \emph{Handbook of Image and Video Processing}, pp.\  397--409. Academic Press, 2nd edition, 2005.
\newblock ISBN 9780121197926.
\newblock \doi{10.1016/B978-012119792-6/50087-5}.

\bibitem[Borsa et~al.(2019)Borsa, Heess, Piot, Liu, Hasenclever, Munos, and Pietquin]{Borsa2017-ze}
Diana Borsa, Nicolas Heess, Bilal Piot, Siqi Liu, Leonard Hasenclever, Remi Munos, and Olivier Pietquin.
\newblock Observational learning by reinforcement learning.
\newblock In \emph{Proceedings of the 18th International Conference on Autonomous Agents and Multi-agent Systems}, AAMAS '19, pp.\  1117–1124, Richland, SC, 2019. International Foundation for Autonomous Agents and Multiagent Systems.
\newblock ISBN 9781450363099.

\bibitem[Bowling et~al.(2023)Bowling, Martin, Abel, and Dabney]{bowling2023settling}
Michael Bowling, John~D Martin, David Abel, and Will Dabney.
\newblock Settling the reward hypothesis.
\newblock In Andreas Krause, Emma Brunskill, Kyunghyun Cho, Barbara Engelhardt, Sivan Sabato, and Jonathan Scarlett (eds.), \emph{Proceedings of the 40th International Conference on Machine Learning}, volume 202 of \emph{Proceedings of Machine Learning Research}, pp.\  3003--3020. PMLR, 23--29 Jul 2023.

\bibitem[Bradtke(1992)]{bradtke1992reinforcement}
Steven Bradtke.
\newblock Reinforcement learning applied to linear quadratic regulation.
\newblock In S.~Hanson, J.~Cowan, and C.~Giles (eds.), \emph{Advances in Neural Information Processing Systems}, volume~5. Morgan-Kaufmann, 1992.

\bibitem[Brockett \& Liberzon(2000)Brockett and Liberzon]{brockett2000quantized}
Roger~W. Brockett and Daniel Liberzon.
\newblock Quantized feedback stabilization of linear systems.
\newblock \emph{IEEE Transactions on Automatic Control}, 45\penalty0 (7):\penalty0 1279--1289, July 2000.
\newblock \doi{10.1109/9.867021}.

\bibitem[Brooks(1986)]{brooks1986robust}
Rodney~A. Brooks.
\newblock A robust layered control system for a mobile robot.
\newblock \emph{IEEE Journal on Robotics and Automation}, 2\penalty0 (1):\penalty0 14--23, 1986.
\newblock \doi{10.1109/JRA.1986.1087032}.

\bibitem[Brooks(1991)]{brooks1991intelligence}
Rodney~A. Brooks.
\newblock Intelligence without representation.
\newblock \emph{Artificial Intelligence}, 47\penalty0 (1):\penalty0 139--159, 1991.
\newblock ISSN 0004-3702.
\newblock \doi{https://doi.org/10.1016/0004-3702(91)90053-M}.

\bibitem[Brown et~al.(2020)Brown, Mann, Ryder, Subbiah, Kaplan, Dhariwal, Neelakantan, Shyam, Sastry, Askell, et~al.]{brown2020language}
Tom Brown, Benjamin Mann, Nick Ryder, Melanie Subbiah, Jared Kaplan, Prafulla Dhariwal, Arvind Neelakantan, Pranav Shyam, Girish Sastry, Amanda Askell, et~al.
\newblock Language models are few-shot learners.
\newblock \emph{Advances in Neural Information Processing Systems}, 33:\penalty0 1877--1901, 2020.

\bibitem[Clark(1998)]{clark1998being}
Andy Clark.
\newblock \emph{Being there: Putting brain, body, and world together again}.
\newblock MIT press, 1998.

\bibitem[Clark \& Chalmers(1998)Clark and Chalmers]{clark1998extended}
Andy Clark and David Chalmers.
\newblock The extended mind.
\newblock \emph{Analysis}, 58\penalty0 (1):\penalty0 7--19, 1998.

\bibitem[Delchamps(1990)]{delchamps1990stabilizing}
David~F. Delchamps.
\newblock Stabilizing a linear system with quantized state feedback.
\newblock \emph{IEEE Transactions on Automatic Control}, 35\penalty0 (8):\penalty0 916--924, August 1990.

\bibitem[Dong et~al.(2022)Dong, Van~Roy, and Zhou]{dong2022simple}
Shi Dong, Benjamin Van~Roy, and Zhengyuan Zhou.
\newblock Simple agent, complex environment: efficient reinforcement learning with agent states.
\newblock \emph{Journal of Machine Learning Research}, 23\penalty0 (1), January 2022.
\newblock ISSN 1532-4435.

\bibitem[Fawzi et~al.(2022)Fawzi, Balog, Huang, Hubert, Romera-Paredes, Barekatain, Novikov, R.~Ruiz, Schrittwieser, Swirszcz, et~al.]{fawzi2022discovering}
Alhussein Fawzi, Matej Balog, Aja Huang, Thomas Hubert, Bernardino Romera-Paredes, Mohammadamin Barekatain, Alexander Novikov, Francisco~J R.~Ruiz, Julian Schrittwieser, Grzegorz Swirszcz, et~al.
\newblock Discovering faster matrix multiplication algorithms with reinforcement learning.
\newblock \emph{Nature}, 610\penalty0 (7930):\penalty0 47--53, 2022.

\bibitem[Gazzaniga(2009)]{gazzaniga2009cognitive}
Michael~S. Gazzaniga (ed.).
\newblock \emph{The Cognitive Neurosciences}.
\newblock The MIT Press, 9 2009.
\newblock ISBN 9780262303101.
\newblock \doi{10.7551/mitpress/8029.001.0001}.

\bibitem[Gershman \& Daw(2017)Gershman and Daw]{gershman2017reinforcement}
Samuel~J. Gershman and Nathaniel~D. Daw.
\newblock Reinforcement learning and episodic memory in humans and animals: An integrative framework.
\newblock \emph{Annual Review of Psychology}, 68\penalty0 (Volume 68, 2017):\penalty0 101--128, 2017.
\newblock ISSN 1545-2085.
\newblock \doi{https://doi.org/10.1146/annurev-psych-122414-033625}.

\bibitem[Grass{\'e}(1959)]{grasse1959reconstruction}
Pierre-Paul Grass{\'e}.
\newblock La reconstruction du nid et les coordinations interindividuelles chez bellicositermes natalensis et cubitermes sp. la th{\'e}orie de la stigmergie: Essai d'interpr{\'e}tation du comportement des termites constructeurs.
\newblock \emph{Insectes Sociaux}, 6\penalty0 (1):\penalty0 41--80, March 1959.
\newblock ISSN 1420-9098.
\newblock \doi{10.1007/BF02223791}.

\bibitem[Hausknecht \& Stone(2015)Hausknecht and Stone]{hausknecht2015deep}
Matthew Hausknecht and Peter Stone.
\newblock Deep recurrent q-learning for partially observable mdps.
\newblock In \emph{AAAI Fall Symposium Series: Sequential Decision Making for Intelligent Agents}, pp.\  29--37, 2015.

\bibitem[Heersmink(2021)]{heersmink2021varieties}
Richard Heersmink.
\newblock Varieties of artifacts: Embodied, perceptual, cognitive, and affective.
\newblock \emph{Topics in Cognitive Science}, 13\penalty0 (4):\penalty0 573--596, 2021.
\newblock \doi{10.1111/tops.12549}.

\bibitem[Heylighen(2015)]{heylighen2015stigmergy}
Francis Heylighen.
\newblock Stigmergy as a universal coordination mechanism: Components, varieties and applications.
\newblock \url{https://pespmc1.vub.ac.be/Papers/Stigmergy-Springer.pdf}, 2015.

\bibitem[Hu et~al.(2021)Hu, Ye, Zhu, Ren, and Zhang]{hu2021generalizable}
Hao Hu, Jianing Ye, Guangxiang Zhu, Zhizhou Ren, and Chongjie Zhang.
\newblock Generalizable episodic memory for deep reinforcement learning.
\newblock In Marina Meila and Tong Zhang (eds.), \emph{Proceedings of the 38th International Conference on Machine Learning}, volume 139 of \emph{Proceedings of Machine Learning Research}, pp.\  4380--4390. PMLR, 18--24 Jul 2021.

\bibitem[Hutchins(1995)]{hutchins1995cognition}
Edwin Hutchins.
\newblock \emph{Cognition in the Wild}.
\newblock The MIT Press, 02 1995.
\newblock ISBN 9780262275972.

\bibitem[Hutchins(2001)]{hutchins2001cognitive}
Edwin Hutchins.
\newblock Cognitive artifacts.
\newblock In \emph{The {MIT} Encyclopedia of the Cognitive Sciences}, pp.\  126--127. MIT Press, Cambridge, MA, 2001.

\bibitem[Hutter(2005)]{hutter2004universal}
Marcus Hutter.
\newblock \emph{Universal Artificial Intelligence}.
\newblock Texts in Theoretical Computer Science. An EATCS Series. Springer, Berlin, Heidelberg, 1st edition, 2005.
\newblock ISBN 978-3-540-22139-5.
\newblock \doi{10.1007/b138233}.

\bibitem[Icarte et~al.(2020)Icarte, Valenzano, Klassen, Christoffersen, massoud Farahmand, and McIlraith]{Icarte2020-bd}
Rodrigo~Toro Icarte, Richard Valenzano, Toryn~Q. Klassen, Phillip Christoffersen, Amir massoud Farahmand, and Sheila~A. McIlraith.
\newblock The act of remembering: a study in partially observable reinforcement learning, 2020.

\bibitem[Jaeger(2000)]{jaeger2000observable}
Herbert Jaeger.
\newblock Observable operator models for discrete stochastic time series.
\newblock \emph{Neural Computation}, 12\penalty0 (6):\penalty0 1371--1398, 06 2000.
\newblock ISSN 0899-7667.
\newblock \doi{10.1162/089976600300015411}.

\bibitem[Kaelbling et~al.(1998)Kaelbling, Littman, and Cassandra]{kaelbling1998planning}
Leslie~Pack Kaelbling, Michael~L. Littman, and Anthony~R. Cassandra.
\newblock Planning and acting in partially observable stochastic domains.
\newblock \emph{Artificial Intelligence}, 101\penalty0 (1-2):\penalty0 99--134, 1998.
\newblock ISSN 0004-3702.
\newblock \doi{https://doi.org/10.1016/S0004-3702(98)00023-X}.

\bibitem[Kaplan et~al.(2020)Kaplan, McCandlish, Henighan, Brown, Chess, Child, Gray, Radford, Wu, and Amodei]{kaplan2020scaling}
Jared Kaplan, Sam McCandlish, Tom Henighan, Tom~B. Brown, Benjamin Chess, Rewon Child, Scott Gray, Alec Radford, Jeffrey Wu, and Dario Amodei.
\newblock Scaling laws for neural language models.
\newblock \emph{arXiv preprint arXiv:2001.08361}, 2020.

\bibitem[Khan et~al.(2018)Khan, Zhang, Atanasov, Karydis, Kumar, and Lee]{khan2017memory}
Arbaaz Khan, Clark Zhang, Nikolay Atanasov, Konstantinos Karydis, Vijay Kumar, and Daniel~D. Lee.
\newblock Memory augmented control networks.
\newblock In \emph{International Conference on Learning Representations}, 2018.

\bibitem[Klein(2015)]{klein2015memory}
Stanley~B. Klein.
\newblock What memory is.
\newblock \emph{WIREs Cognitive Science}, 6\penalty0 (1):\penalty0 1--38, 2015.
\newblock \doi{https://doi.org/10.1002/wcs.1333}.

\bibitem[Lengyel \& Dayan(2007)Lengyel and Dayan]{lengyel2007hippocampal}
M\'{a}t\'{e} Lengyel and Peter Dayan.
\newblock Hippocampal contributions to control: The third way.
\newblock In J.~Platt, D.~Koller, Y.~Singer, and S.~Roweis (eds.), \emph{Advances in Neural Information Processing Systems}, volume~20, pp.\  889--896, 2007.

\bibitem[Lin(1992)]{lin1992self}
Long-Ji Lin.
\newblock Self-improving reactive agents based on reinforcement learning, planning and teaching.
\newblock \emph{Machine Learning}, 8\penalty0 (3):\penalty0 293--321, May 1992.
\newblock ISSN 1573-0565.
\newblock \doi{10.1007/BF00992699}.

\bibitem[Lin et~al.(2018)Lin, Zhao, Yang, and Zhang]{lin2018episodic}
Zichuan Lin, Tianqi Zhao, Guangwen Yang, and Lintao Zhang.
\newblock Episodic memory deep q-networks.
\newblock In \emph{Proceedings of the Twenty-Seventh International Joint Conference on Artificial Intelligence, {IJCAI-18}}, pp.\  2433--2439. International Joint Conferences on Artificial Intelligence Organization, 7 2018.
\newblock \doi{10.24963/ijcai.2018/337}.

\bibitem[Littman(1994)]{littman1994memoryless}
Michael~L. Littman.
\newblock Memoryless policies: theoretical limitations and practical results.
\newblock In \emph{Proceedings of the Third International Conference on Simulation of Adaptive Behavior: From Animals to Animats 3: From Animals to Animats 3}, SAB94, pp.\  238–245, Cambridge, MA, USA, 1994. The MIT Press.
\newblock ISBN 0262531224.

\bibitem[Littman et~al.(2001)Littman, Sutton, and Singh]{littman2001predictive}
Michael~L. Littman, Richard~S. Sutton, and Satinder Singh.
\newblock Predictive representations of state.
\newblock In T.~Dietterich, S.~Becker, and Z.~Ghahramani (eds.), \emph{Advances in Neural Information Processing Systems}, volume~14, pp.\  1555--1561, 2001.

\bibitem[Mart{\'\i}n~Mu{\~n}oz(1998)]{mariomartin-thesis}
Mario Mart{\'\i}n~Mu{\~n}oz.
\newblock \emph{Reinforcement Learning for Embedded Agents Facing Complex Tasks}.
\newblock Phd thesis, Universitat Polit{\`e}cnica de Catalunya, Barcelona, Spain, 1998.

\bibitem[Menary(2010)]{menary2010cognitive}
Richard Menary.
\newblock Cognitive integration and the extended mind.
\newblock \emph{The extended mind}, pp.\  227--243, 2010.

\bibitem[Michaelian(2012)]{michaelian2012external}
Kourken Michaelian.
\newblock Is external memory memory? biological memory and extended mind.
\newblock \emph{Consciousness and Cognition}, 21\penalty0 (3):\penalty0 1154--1165, 2012.
\newblock ISSN 1053-8100.
\newblock \doi{https://doi.org/10.1016/j.concog.2012.04.008}.

\bibitem[Michaelian \& Sutton(2017)Michaelian and Sutton]{sep-memory}
Kourken Michaelian and John Sutton.
\newblock {Memory}.
\newblock In Edward~N. Zalta (ed.), \emph{The {Stanford} Encyclopedia of Philosophy}. Metaphysics Research Lab, Stanford University, {S}ummer 2017 edition, 2017.

\bibitem[Mnih et~al.(2015)Mnih, Kavukcuoglu, Silver, Rusu, Veness, Bellemare, Graves, Riedmiller, Fidjeland, Ostrovski, Petersen, Beattie, Sadik, Antonoglou, King, Kumaran, Wierstra, Legg, and Hassabis]{mnih2015human}
Volodymyr Mnih, Koray Kavukcuoglu, David Silver, Andrei~A. Rusu, Joel Veness, Marc~G. Bellemare, Alex Graves, Martin Riedmiller, Andreas~K. Fidjeland, Georg Ostrovski, Stig Petersen, Charles Beattie, Amir Sadik, Ioannis Antonoglou, Helen King, Dharshan Kumaran, Daan Wierstra, Shane Legg, and Demis Hassabis.
\newblock Human-level control through deep reinforcement learning.
\newblock \emph{Nature}, 518\penalty0 (7540):\penalty0 529--533, February 2015.
\newblock ISSN 1476-4687.
\newblock \doi{10.1038/nature14236}.

\bibitem[Newell(1990)]{newell1990unified}
Allen Newell.
\newblock \emph{Unified Theories of Cognition}.
\newblock Harvard University Press, Cambridge, MA, 1990.
\newblock ISBN 9780674920996.

\bibitem[Oh et~al.(2016)Oh, Chockalingam, Singh, and Lee]{oh2016control}
Junhyuk Oh, Valliappa Chockalingam, Satinder Singh, and Honglak Lee.
\newblock Control of memory, active perception, and action in minecraft.
\newblock In Maria~Florina Balcan and Kilian~Q. Weinberger (eds.), \emph{Proceedings of The 33rd International Conference on Machine Learning}, volume~48 of \emph{Proceedings of Machine Learning Research}, pp.\  2790--2799, New York, New York, USA, 20--22 Jun 2016. PMLR.

\bibitem[Patterson et~al.(2024)Patterson, Neumann, White, and White]{patterson2024empirical}
Andrew Patterson, Samuel Neumann, Martha White, and Adam White.
\newblock Empirical design in reinforcement learning.
\newblock \emph{Journal of Machine Learning Research}, 25\penalty0 (318):\penalty0 1--63, 2024.
\newblock URL \url{https://jmlr.org/papers/v25/23-0183.html}.

\bibitem[Peshkin et~al.(1999)Peshkin, Meuleau, and Kaelbling]{peshkin1999learning}
Leonid Peshkin, Nicolas Meuleau, and Leslie~Pack Kaelbling.
\newblock Learning policies with external memory.
\newblock In \emph{Proceedings of the 16th International Conference on Machine Learning}, ICML '99, pp.\  307–314, San Francisco, CA, USA, 1999. Morgan Kaufmann Publishers Inc.
\newblock ISBN 1558606122.

\bibitem[Pritzel et~al.(2017)Pritzel, Uria, Srinivasan, Badia, Vinyals, Hassabis, Wierstra, and Blundell]{pritzel2017neural}
Alexander Pritzel, Benigno Uria, Sriram Srinivasan, Adri{\`a}~Puigdom{\`e}nech Badia, Oriol Vinyals, Demis Hassabis, Daan Wierstra, and Charles Blundell.
\newblock Neural episodic control.
\newblock In Doina Precup and Yee~Whye Teh (eds.), \emph{Proceedings of the 34th International Conference on Machine Learning}, volume~70 of \emph{Proceedings of Machine Learning Research}, pp.\  2827--2836. PMLR, 06--11 Aug 2017.

\bibitem[Rao \& Georgeff(1995)Rao and Georgeff]{rao1995bdi}
Anand~S. Rao and Michael~P. Georgeff.
\newblock {BDI} agents: From theory to practice.
\newblock In \emph{Proceedings of the First International Conference on Multi-Agent Systems (ICMAS-95)}, pp.\  312--319, 1995.

\bibitem[Reid et~al.(2012)Reid, Latty, Dussutour, and Beekman]{reid2012slime}
Chris~R Reid, Tanya Latty, Audrey Dussutour, and Madeleine Beekman.
\newblock Slime mold uses an externalized spatial “memory” to navigate in complex environments.
\newblock \emph{Proceedings of the National Academy of Sciences}, 109\penalty0 (43):\penalty0 17490--17494, 2012.

\bibitem[Ricci et~al.(2007)Ricci, Omicini, Viroli, Gardelli, and Oliva]{ricci2006cognitive}
Alessandro Ricci, Andrea Omicini, Mirko Viroli, Luca Gardelli, and Enrico Oliva.
\newblock Cognitive stigmergy: Towards a framework based on agents and artifacts.
\newblock In Danny Weyns, H.~Van~Dyke Parunak, and Fabien Michel (eds.), \emph{Environments for Multi-Agent Systems III}, pp.\  124--140. Springer, 2007.
\newblock ISBN 978-3-540-71103-2.

\bibitem[Russell(1921)]{russell1921analysis}
Bertrand Russell.
\newblock \emph{The Analysis of Mind}.
\newblock G. Allen \& Unwin, London, 1921.

\bibitem[Schaul et~al.(2016)Schaul, Quan, Antonoglou, and Silver]{schaul2016prioritized}
Tom Schaul, John Quan, Ioannis Antonoglou, and David Silver.
\newblock Prioritized experience replay.
\newblock In \emph{International Conference on Learning Representations}, 2016.

\bibitem[Silver et~al.(2016)Silver, Huang, Maddison, Guez, Sifre, Van Den~Driessche, Schrittwieser, Antonoglou, Panneershelvam, Lanctot, et~al.]{silver2016mastering}
David Silver, Aja Huang, Chris~J Maddison, Arthur Guez, Laurent Sifre, George Van Den~Driessche, Julian Schrittwieser, Ioannis Antonoglou, Veda Panneershelvam, Marc Lanctot, et~al.
\newblock Mastering the game of go with deep neural networks and tree search.
\newblock \emph{nature}, 529\penalty0 (7587):\penalty0 484--489, 2016.

\bibitem[Sims \& Kiverstein(2022)Sims and Kiverstein]{sims2022externalized}
Matthew Sims and Julian Kiverstein.
\newblock Externalized memory in slime mould and the extended (non-neuronal) mind.
\newblock \emph{Cognitive Systems Research}, 73:\penalty0 26--35, 2022.
\newblock ISSN 1389-0417.
\newblock \doi{https://doi.org/10.1016/j.cogsys.2021.12.001}.

\bibitem[Singh et~al.(2004)Singh, James, and Rudary]{singh2004predictive}
Satinder Singh, Michael~R. James, and Matthew~R. Rudary.
\newblock Predictive state representations: a new theory for modeling dynamical systems.
\newblock In \emph{Proceedings of the 20th Conference on Uncertainty in Artificial Intelligence}, UAI '04, pp.\  512–519, Arlington, Virginia, USA, 2004. AUAI Press.
\newblock ISBN 0974903906.

\bibitem[Singh et~al.(1994)Singh, Jaakkola, and Jordan]{singh1994learning}
Satinder~P. Singh, Tommi Jaakkola, and Michael~I. Jordan.
\newblock Learning without state-estimation in partially observable markovian decision processes.
\newblock In William~W. Cohen and Haym Hirsh (eds.), \emph{Machine Learning Proceedings 1994}, pp.\  284--292. Morgan Kaufmann, San Francisco (CA), 1994.
\newblock ISBN 978-1-55860-335-6.
\newblock \doi{https://doi.org/10.1016/B978-1-55860-335-6.50042-8}.

\bibitem[Sterelny(2010)]{sterelny2010minds}
Kim Sterelny.
\newblock Minds: extended or scaffolded?
\newblock \emph{Phenomenology and the Cognitive Sciences}, 9\penalty0 (4):\penalty0 465--481, December 2010.
\newblock ISSN 1572-8676.
\newblock \doi{10.1007/s11097-010-9174-y}.

\bibitem[Sumpter \& Beekman(2003)Sumpter and Beekman]{sumpter2003nonlinearity}
David~J.T Sumpter and Madeleine Beekman.
\newblock From nonlinearity to optimality: pheromone trail foraging by ants.
\newblock \emph{Animal Behaviour}, 66\penalty0 (2):\penalty0 273--280, 2003.
\newblock ISSN 0003-3472.
\newblock \doi{https://doi.org/10.1006/anbe.2003.2224}.

\bibitem[Sutton(2003)]{sutton2003constructive}
John Sutton.
\newblock Constructive memory and distributed cognition: Towards an interdisciplinary framework.
\newblock In B.~Kokinov and W.~Hirst (eds.), \emph{Constructive Memory}, pp.\  290--303. New Bulgarian University, 2003.

\bibitem[Sutton(2019)]{sutton2019bitter}
Richard~S. Sutton.
\newblock The bitter lesson.
\newblock \emph{Incomplete Ideas (blog)}, 2019.
\newblock URL \url{http://www.incompleteideas.net/IncIdeas/BitterLesson.html}.

\bibitem[Sutton(2022)]{sutton2022quest}
Richard~S. Sutton.
\newblock The quest for a common model of the intelligent decision maker.
\newblock In \emph{Proceedings of the 5th Multi-disciplinary Conference on Reinforcement Learning and Decision Making (RLDM 2022)}, Providence, Rhode Island, USA, 2022.

\bibitem[Sutton \& Barto(2018)Sutton and Barto]{sutton2018reinforcement}
Richard~S. Sutton and Andrew~G. Barto.
\newblock \emph{Reinforcement Learning: An Introduction}.
\newblock The MIT Press, Cambridge, MA, 2nd edition, 2018.

\bibitem[Tamborski \& Abel(2025)Tamborski and Abel]{tamborski2025memory}
Massimiliano Tamborski and David Abel.
\newblock Memory allocation in resource-constrained reinforcement learning.
\newblock \emph{arXiv preprint arXiv:2506.17263}, 2025.

\bibitem[Thierry et~al.(1995)Thierry, Theraulaz, Gautier, and Stiegler]{thierry1995joint}
B.~Thierry, G.~Theraulaz, J.Y. Gautier, and B.~Stiegler.
\newblock Joint memory.
\newblock \emph{Behavioural Processes}, 35\penalty0 (1):\penalty0 127--140, 1995.
\newblock ISSN 0376-6357.
\newblock \doi{https://doi.org/10.1016/0376-6357(95)00039-9}.
\newblock Cognition and Evolution.

\bibitem[Tulving(1972)]{Tulving1972}
Endel Tulving.
\newblock Episodic and semantic memory.
\newblock In \emph{Organization of Memory}, pp.\  381--403. Academic Press, London, UK, 1972.

\bibitem[Watkins \& Dayan(1992)Watkins and Dayan]{watkins1992q}
Christopher J. C.~H. Watkins and Peter Dayan.
\newblock Q-learning.
\newblock \emph{Machine Learning}, 8\penalty0 (3--4):\penalty0 279--292, May 1992.
\newblock ISSN 1573-0565.
\newblock \doi{10.1007/BF00992698}.

\bibitem[Wilson(1962)]{wilson1962chemical}
Edward~O. Wilson.
\newblock Chemical communication among workers of the fire ant solenopsis saevissima (fr. smith) 1. the organization of mass-foraging.
\newblock \emph{Animal Behaviour}, 10\penalty0 (1):\penalty0 134--147, 1962.
\newblock ISSN 0003-3472.
\newblock \doi{https://doi.org/10.1016/0003-3472(62)90141-0}.

\bibitem[Zhu* et~al.(2020)Zhu*, Lin*, Yang, and Zhang]{zhu2020episodic}
Guangxiang Zhu*, Zichuan Lin*, Guangwen Yang, and Chongjie Zhang.
\newblock Episodic reinforcement learning with associative memory.
\newblock In \emph{International Conference on Learning Representations}, 2020.

\end{thebibliography}
\bibliographystyle{rlj}

\beginSupplementaryMaterials
\section{Proofs}\label{app:proofs}

\prob*
\begin{proof}
    Necessity: Suppose $\xi$ is artifactual. Then $\Omega_\xi\neq \emptyset$, meaning there exists at least one $o\in\Omega_\xi$ satisfying \autoref{def:artifact}. According to \autoref{def:artifact}, for all times $t$ where $O_t =o$, there exists some non-zero $t' < t$ and $o'\neq o$ such that $O_{t'}=o'$. 
    
    In terms of probability, if the event $O_t = o$ logically necessitates that $O_{t'} = o'$, then the sample space is constrained such that there are no outcomes where $O_t = o$ and $O_{t'} \neq o'$. Mathematically, this implies:
    \begin{align*}
        \Pbb(O_{t'} = o' | O_t = o) = \frac{\Pbb(O_{t'} = o' \cap O_t = o)}{\Pbb(O_t = o)}
    \end{align*}
    Since $O_t = o$ implies $O_{t'} = o'$, the intersection $O_{t'} = o' \cap O_t = o$ is simply the event $O_t = o$:
    \begin{align*}
        \Pbb(O_{t'} = o' | O_t = o) = \frac{O_t = o)}{\Pbb(O_t = o)} = 1.
    \end{align*}

    Sufficiency: Suppose that for any $t > 0$, there exist distinct $o, o'$ and $t' < t$ such that $\Pbb(O_{t'} = o' | O_t = o) = 1$. This implies that the occurrence of $o$ at time $t$ provides total certainty about the occurrence of $o'$ at time $t'$. If the probability of the past state $o'$ given the current state $o$ is 1, then for every realization of the process where $O_t = o$, it must be that $O_{t'} = o'$. Since $o'$ is distinct from $o$, this satisfies the requirement in \autoref{def:artifact}. Because such an $o$ exists, the set $\Omega_\xi$ is non-empty. Per \autoref{def:artifactual_env}, the environment is therefore artifactual.    
\end{proof}

\orderreduction*
\begin{proof}
Since $\xi$ is artifactual, by \autoref{lemma:prob}, there exist distinct observations $o, o'$ and non-zero times $t,t'$ with $t' < t$ such that $\Pbb(O_{t'}=o' \mid O_t=o) = 1$. 

Let $H = O_{t},A_{t}, O_{t-1},A_{t-1}, \ldots, O_{t-m+1},A_{t-m+1}$ be an $m$-length history ending at time $t$. Assume $H$ contains both the observation $o'$ at time $t'=t-k$ and its artifact $o$ at time $t$, where $1 \leq k < m$.

Construct $H'$ by removing the observation at time $t-k$:
\begin{align*}
H' = O_{t},A_{t},\ldots O_{t-k+1},A_{t-k+1},A_{t-k}, O_{t-k-1},A_{t-k-1},\ldots  O_{t-m+1},A_{t-m+1}.
\end{align*}
Note that $H'$ has $m - 1$ observations.

By the chain rule for mutual information:
\begin{align*}
\Ibb(O_{t+1}; H) &= \Ibb(O_{t+1}; H') + \Ibb(O_{t+1}; O_{t-k} \mid H').
\end{align*}

Since $O_t = o$ is in $H'$ and $\Pbb(O_{t-k}=o' \mid O_t=o) = 1$, the observation $O_t = o$ completely determines $O_{t-k} = o'$. Therefore, $O_{t-k}$ is a deterministic function of the information already present in $H'$, which implies $\Ibb(O_{t+1}; O_{t-k} \mid H') = 0.$ Thus,
\begin{align*}
\Ibb(O_{t+1}; H) = \Ibb(O_{t+1}; H').
\end{align*}

This equality holds for all histories $H$ with positive probability under the environment dynamics where observation $o'$ precedes its artifact $o$.
\end{proof}

\indet*
\begin{proof}
        Let $\xi$ be an artifactual environment. We construct an artifactless copy $\xi'$ by adding noise to all the artifactual relationships.
        
        Consider the set of all artifactual relationships in $\xi$:
        \begin{align*}
            \Zcal = \{(i,j,k,t) : i,j\in\Ocal, i \neq j, k \geq 1, t > k, \Pbb(X_{t-k}=i \mid X_{t}=j) = 1\}
        \end{align*}
        For each $(i,j,k,t) \in \Zcal$, modify the transition dynamics to ensure that 
        when $X_t = j$, there is a small but non-zero probability $\epsilon > 0$ that 
        $X_{t-k} \neq i$. Specifically, define $\xi'$ such that:
        \begin{align*}
        \Pbb_{\xi'}(X_{t-k}=i \mid X_{t}=j) = 
        \begin{cases}
        1 - \epsilon &: (i,j,k,t) \in \Zcal \\
        \Pbb_\xi(X_{t-k}=i \mid X_{t}=j) &: \text{otherwise}
        \end{cases}
        \end{align*}
        where the remaining probability mass $\epsilon$ is distributed among other possible 
        values of $X_{t-k}$.
        
        By construction, for all distinct observations $i,j$ and all integers $k \geq 1$, $t$ with $t > k$:
        $$\Pbb_{\xi'}(X_{t-k}=i \mid X_{t}=j) \leq 1 - \epsilon < 1$$
        
        Therefore, $\xi'$ is non-artifactual, and the artifacts of $\xi$ are obscured in $\xi'$.
        \end{proof}

\subsection{Additional Results}
\begin{corollary}\label{corr:multiple_reduction}
Suppose $H$ is an $m$-length history containing $k<m$ artifacts. There exists a history $H'$ containing $m-k$ observations such that 
\begin{align*}
\Ibb(O_{t+1}; H) = \Ibb(O_{t+1}; H').
\end{align*}
\begin{proof}
    Starting with an $m$-length history $H_m$, define a history $H_{m-1}$ by removing the observation associated with a single artifact. According to the Artifact Reduction Theorem (\autoref{thm:artifact_reduction}), $H_m$ and $H_{m-1}$ have the same mutual information with $O_{t+1}$. Apply \autoref{thm:artifact_reduction} starting from the reduced history of the previous application. After $k$ applications, let $H=H_m$ and $H'=H_{m-k}$. We have  
    \begin{align*}
        \Ibb(O_{t+1}; H) = \Ibb(O_{t+1}; H').
    \end{align*}
\end{proof}
\end{corollary}

\section{Experimental Details} \label{app:experimental_details}
\subsection{Hyperparameter Selection}
Our experiments treat step-size as a hyperparameter, sweeping over a finite set of candidate values to select the value yielding the highest total reward averaged across 30 seeds. To correct for maximization bias, we use a two-stage approach \citep{patterson2024empirical}. Specifically, we report the average and standard error of 30 different seeds for the hyperparameters with maximum performance. Figures \ref{fig:exp1_stepsize_sweeps}, \ref{fig:exp2_sweep}, and \ref{fig:exp3_stepsize_dyn_lin} show the average total reward across the full range of step-sizes considered in each experiment. Note the plots exhibit distinct maximums most capacities.

\subsection{Statistical Tests} \label{app:stat_tests}
We use a two-sample model to test whether the mean total reward from one agent is higher than another. Let $\Pcal_i$ be a random sample of total rewards from an agent that learns in an artifactual environment. The subscript $i$ indexes the agent's capacity, e.g. $2\times 16$. Similarly, let $\Qcal_j$ denote the random sample of agent $j$ in the artifactless environment (No Path). The empirical averages of each sample are respectively denoted $\bar{p}_i$ and $\bar{q}_j$. We test every pair $(i,j)$ according to the null hypothesis $H_0(i,j): \bar{p}_i \leq \bar{q}_j$ at a significance-level of $\alpha=0.05$  For more information on two-sample models see \cite{bickel2015mathematical}.

\subsection{Environment Details} \label{sec:environment_details}
\begin{figure}[h]
    \centering
    \includegraphics[width=\linewidth]{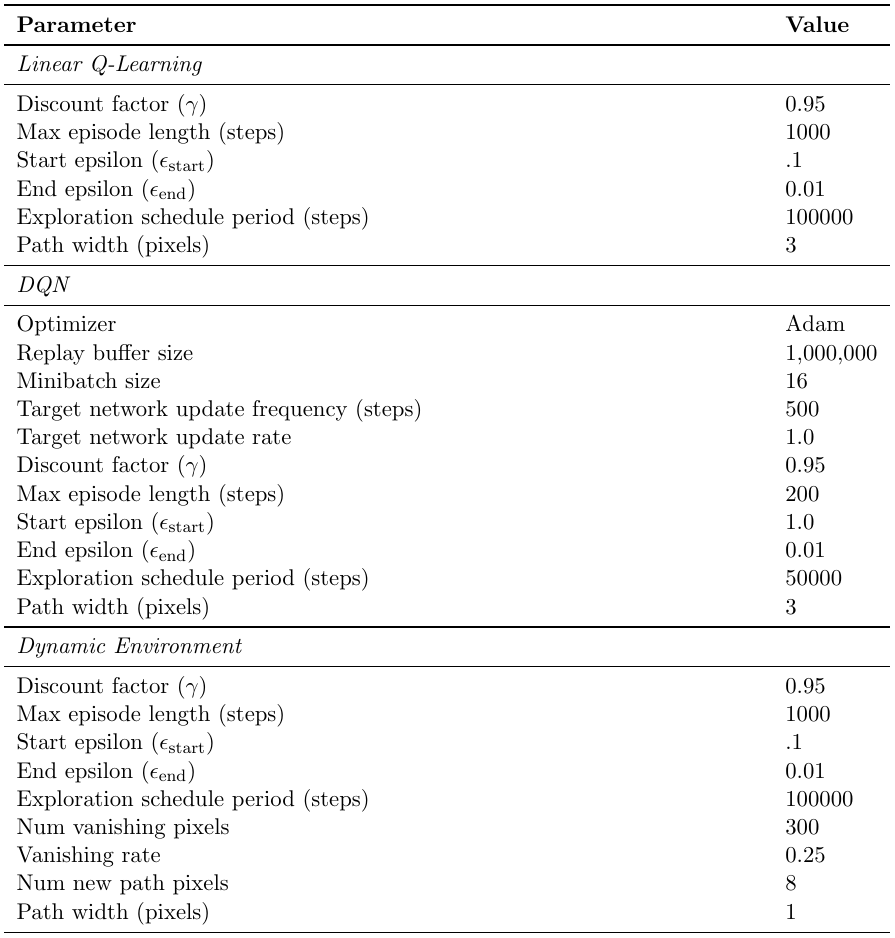}  
    \caption{\textbf{Experiment Configuration Details.} Hyperparameters used for linear agent (above), DQN agents (middle), dynamic experiments (below).}
    \label{fig:experiment_details}
\end{figure}
\paragraph{Observations}
\begin{itemize} 
    \item Tile textures: randomly generated black-and-white patterns with 10\% black pixels.
    \item The agent observes a 3$\times$3 portion of the grid centered on its current cell.
    \item Walls are not explicitly visible; they are only apparent through self-looping transitions.
    \item The agent's position and goal location are not visible in observations.
\end{itemize}

\paragraph{Artifacts}
\begin{itemize}
    \item \textbf{Optimal Path:} Shortest path from start to goal.
    \item \textbf{Suboptimal Path:} A path 8 steps longer than optimal.
    \item \textbf{Misleading Path:} Hand-crafted path that does not lead to the goal.
    \item \textbf{Random Path:} A fixed path generated via random walk.
    \item \textbf{Landmarks:} Multi-cell shapes (diamond, donut, circle, rectangle, triangle, square) distributed throughout the environment.
\end{itemize}
\paragraph{Dynamic Path Environment.}
Here, we describe the path dynamics (\autoref{fig:experiment_details}). A path is drawn at every transition: from the current cell to the next. Let $\Pcal$ be the set of pixels connecting the centers of both cells with a given thickness. Furthermore, let $\Qcal$ be set of all pixels. At every timestep, a subset of $\Pcal$ (Num new path pixels) is drawn without replacement and uniformly at random. These pixels are assigned values of one to mark the path. These values persist into the future, until they are randomly selected for removal. At each step, a subset of $\Qcal$ (Num vanishing pixels) is drawn uniformly at random and without replacement. With a fixed probability (Vanishing rate), these pixels are assigned values of zero. The path is applied as a mask to the full image maintained by the environment. Thus, from the agent's perspective, a value of zero removes the path and reveals the original background image of salt-and-pepper noise.

\subsection{Experiment 1. Learning in the Presence of a Shortest Path}\label{app:exp1}
\begin{figure}[H]
    \centering
    \includegraphics[width=0.49\linewidth]{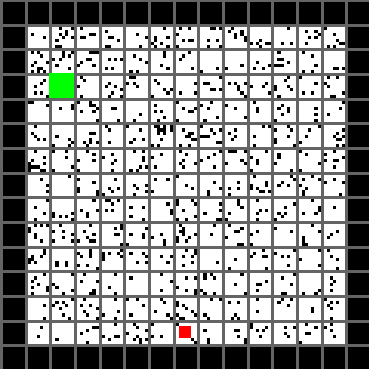}
    \includegraphics[width=0.49\linewidth]{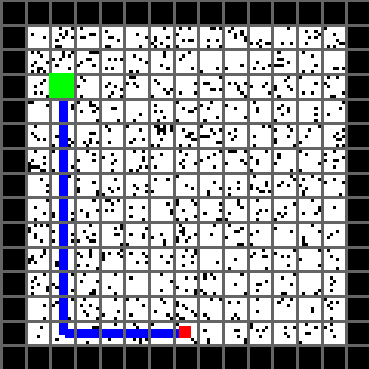}
    \caption{Environments considered in Experiment 1.}
    \label{fig:e1-envs-app}
\end{figure}

\begin{figure}[H]
    \centering
    \includegraphics[width=0.49\linewidth]{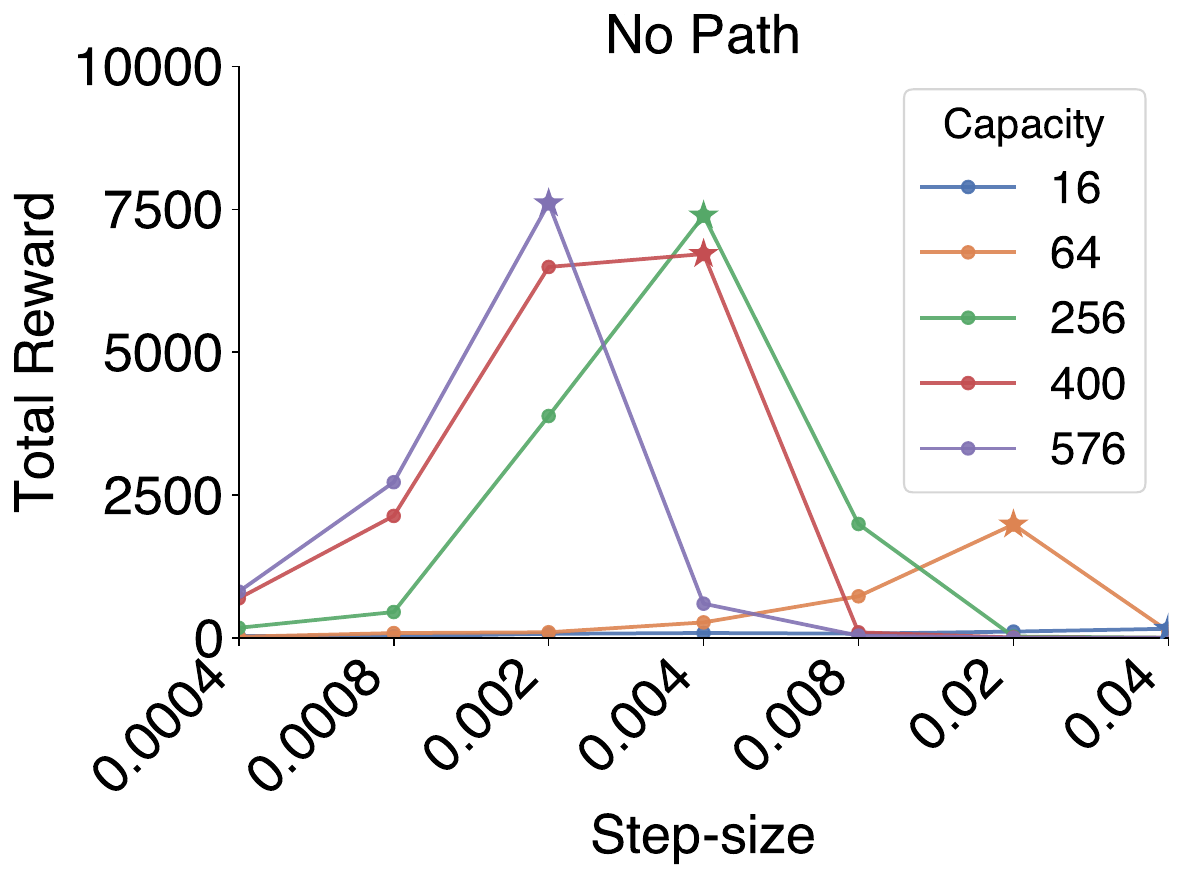}
    \includegraphics[width=0.49\linewidth]{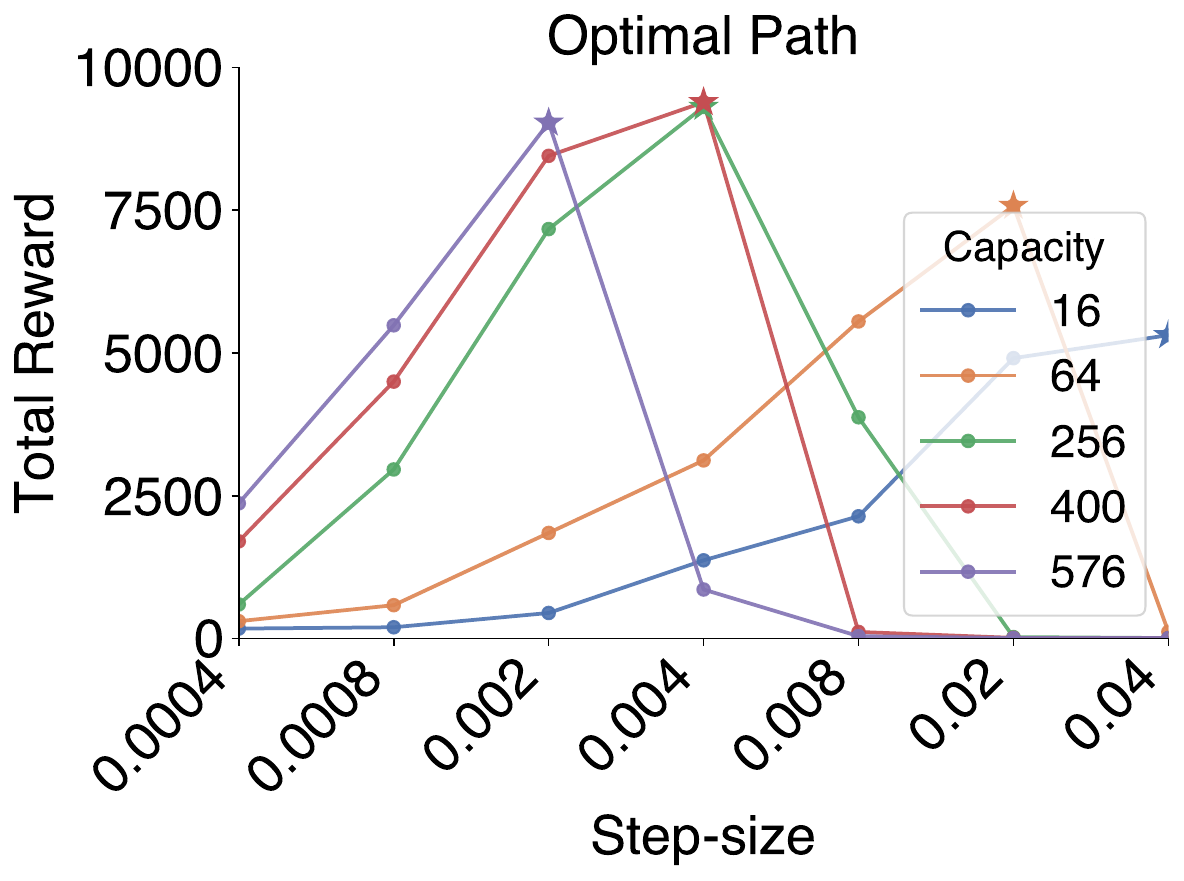}
    \includegraphics[width=0.49\linewidth]{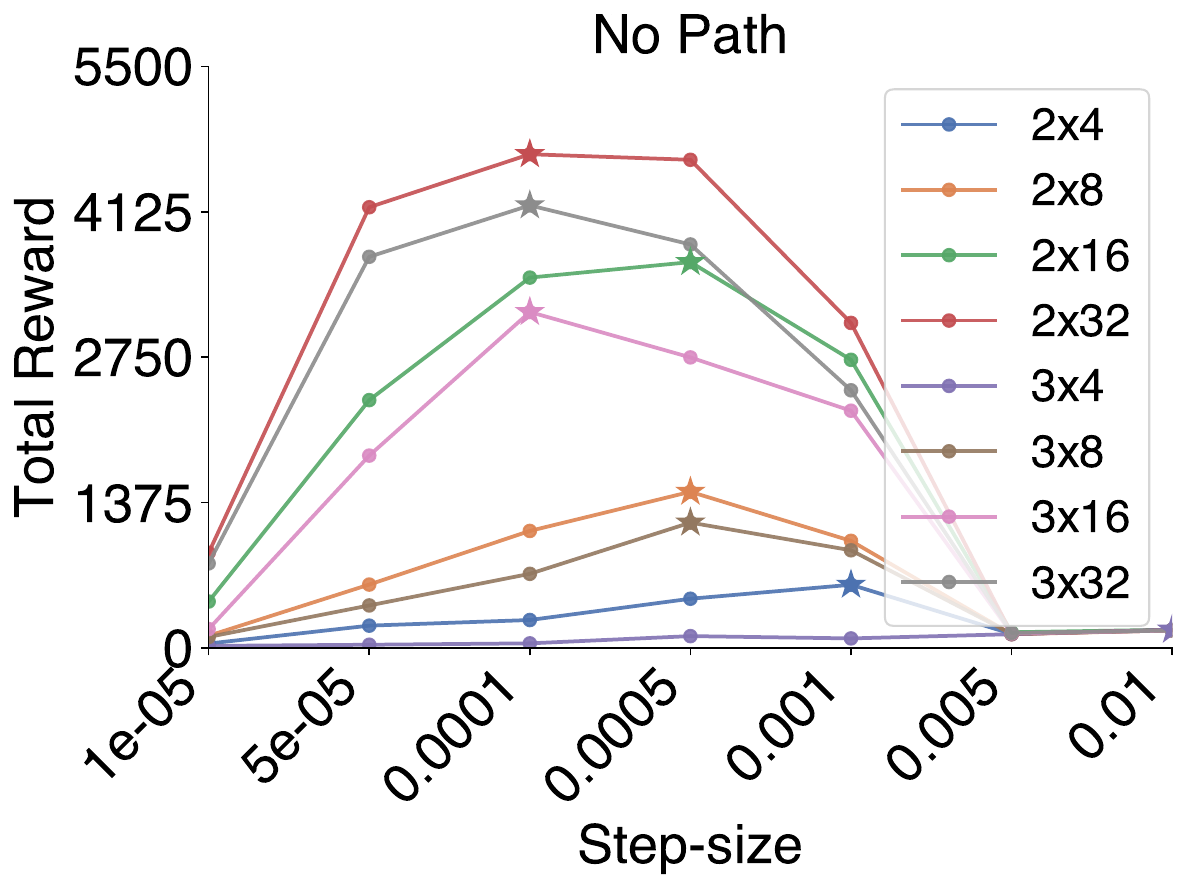}
    \includegraphics[width=0.49\linewidth]{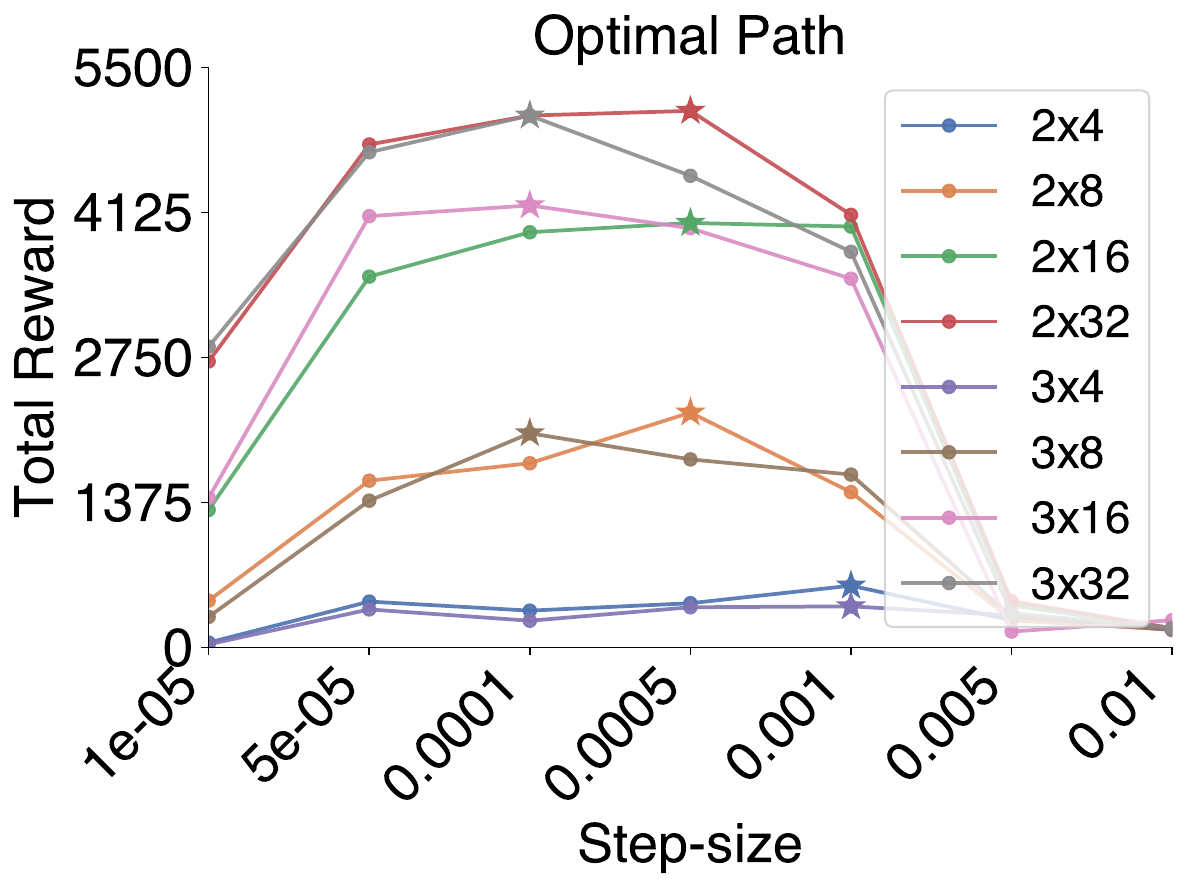}
    \caption{\textbf{Experiment 1. Step-size sweeps.} Linear Q-learning (top half), DQN (bottom half). Selected step-size  is marked with a star.}
    \label{fig:exp1_stepsize_sweeps}
\end{figure}

\begin{figure}[H]
    \centering
    \includegraphics[width=0.49\linewidth]{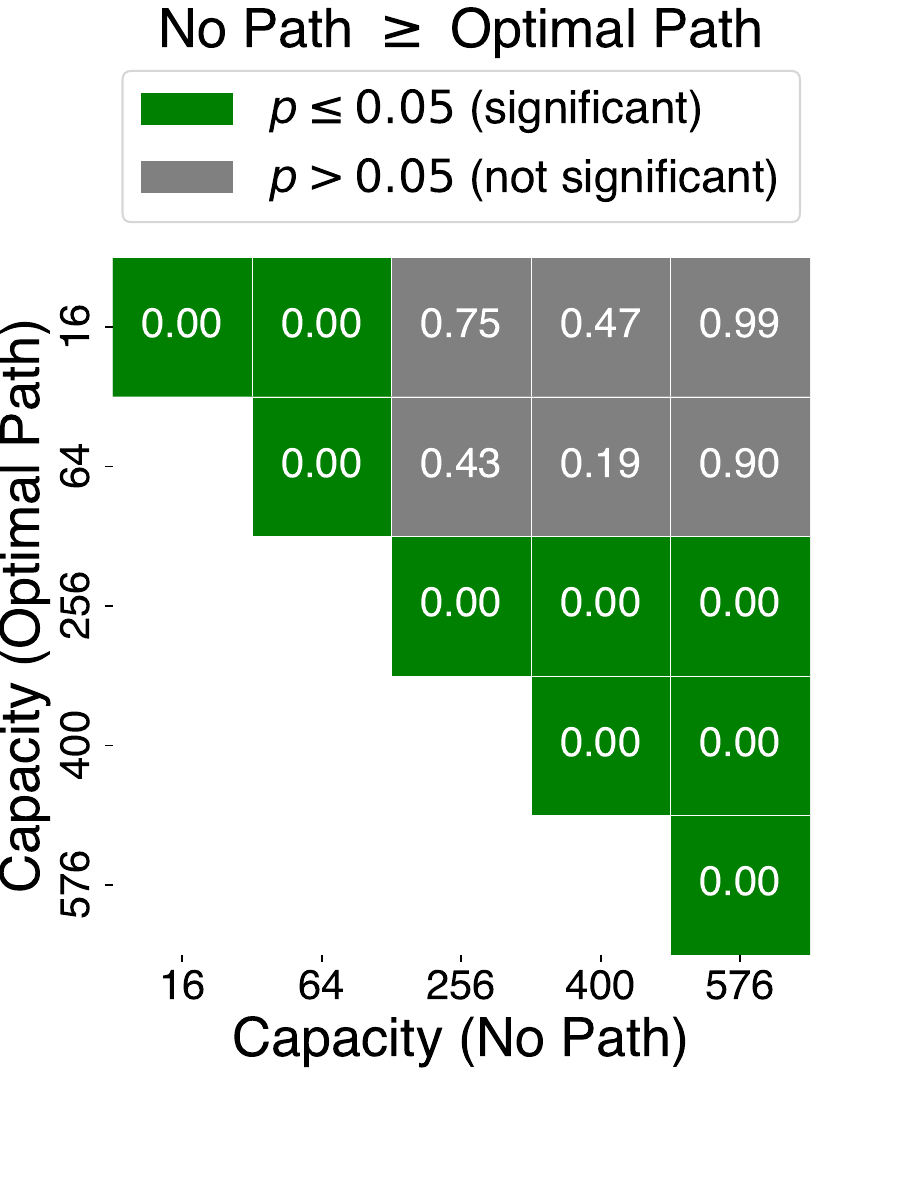}
    \includegraphics[width=0.49\linewidth]{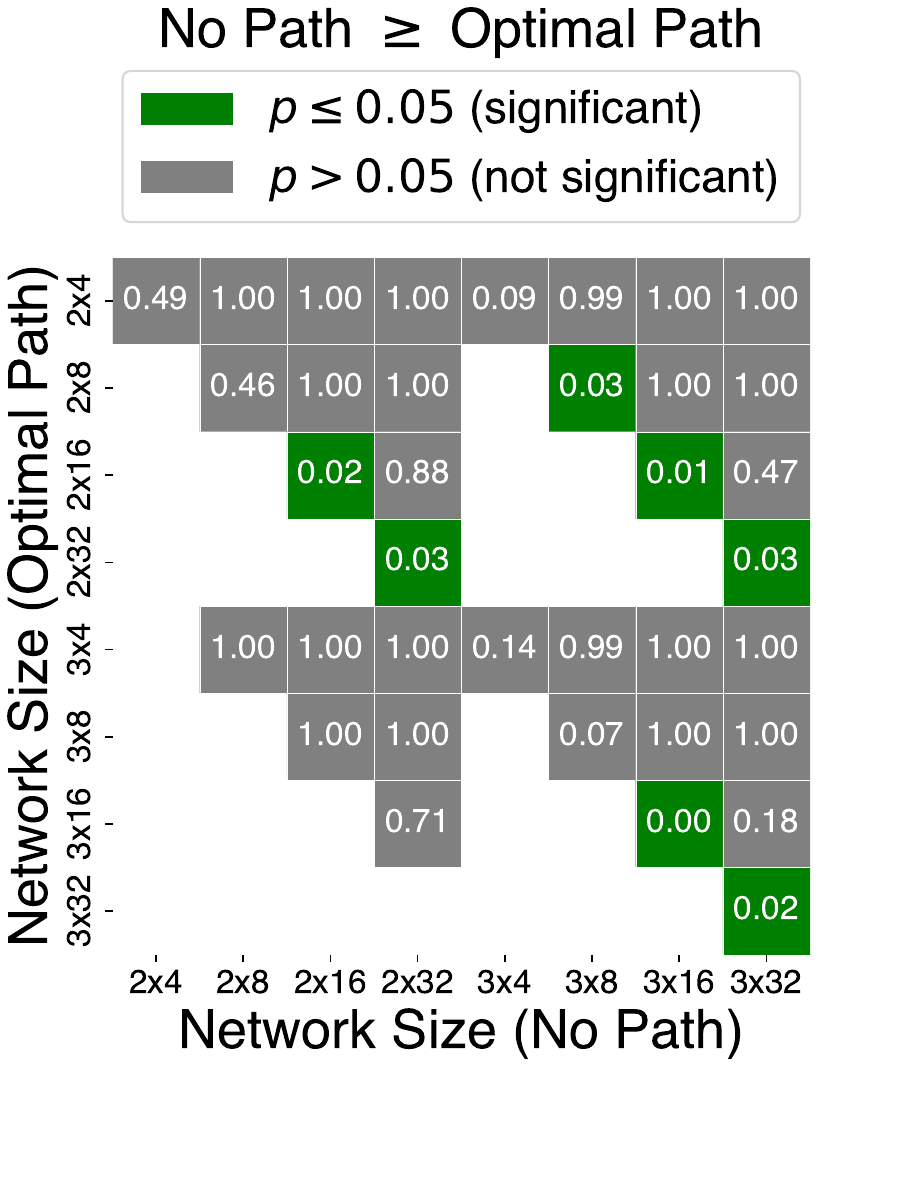}
    \caption{\textbf{Experiment 1. Significance Tests.} Let $P_i$ and $P_j$ be the performances associated with capacities of the row and column. The plot should be read row-wise: when the $(i,j)$-cell is green, $P_i$ is significantly higher than $P_j$.}
    \label{fig:exp1_sigtest}
\end{figure}

\subsection{Experiment 2. Learning in the Presence of Other Fixed Artifacts}
\begin{figure}[H]
    \centering
    \includegraphics[width=1.0\linewidth]{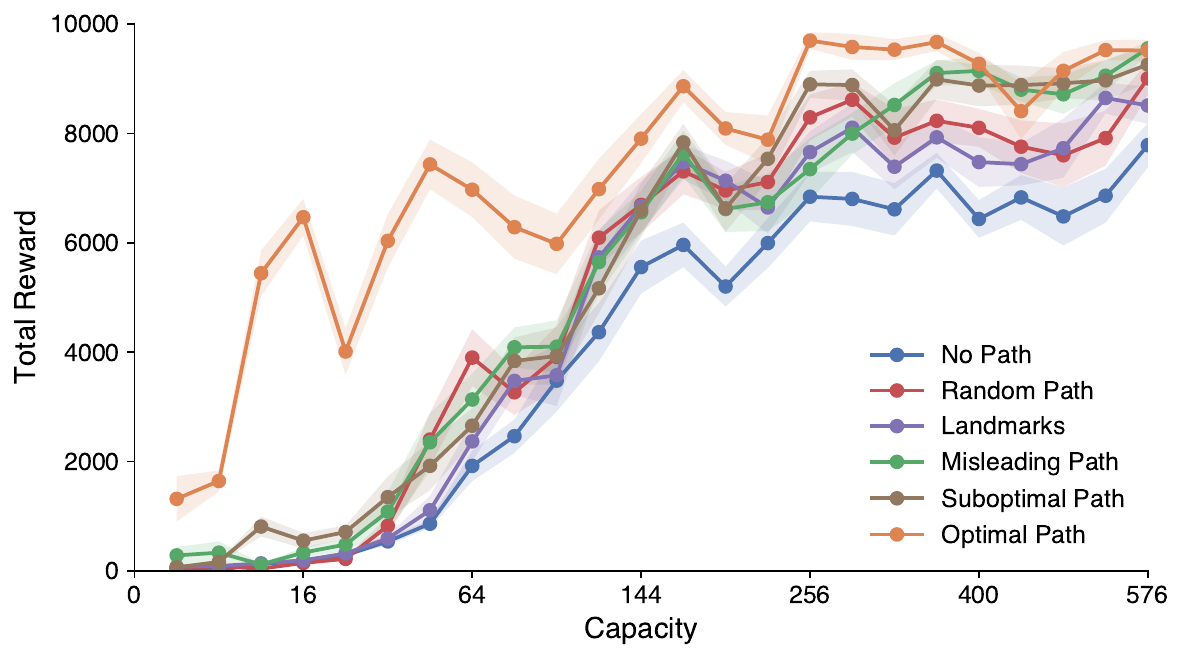}
    \caption{\textbf{Capacity vs performance of Linear-Q in the presence of fixed artifacts.} Total reward for each artifact type is shown. Each data point presents an average and standard-error from 30 seeds. Capacity ranges from $1^2$ to $24^2$ (1 to 576). }
    \label{fig:total_reward_lines}
\end{figure}

\begin{figure}[H]
    \centering
    \includegraphics[width=0.49\linewidth]{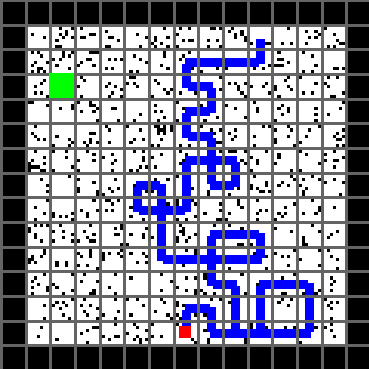}
    \includegraphics[width=0.49\linewidth]{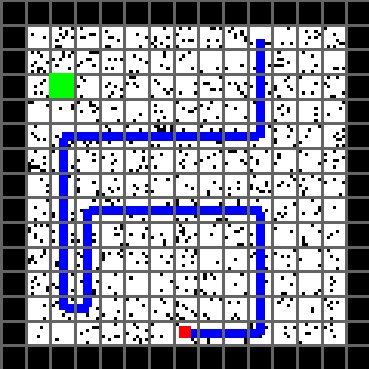}
    \includegraphics[width=0.49\linewidth]{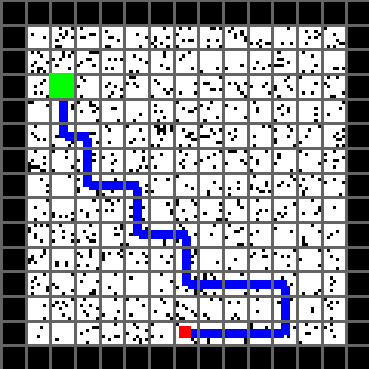}
    \includegraphics[width=0.49\linewidth]{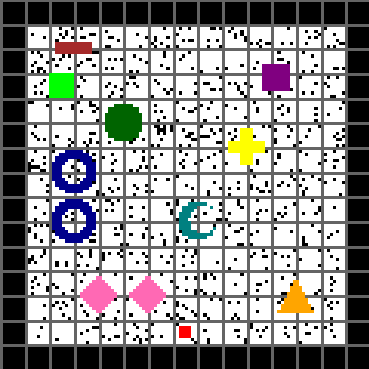}
    \caption{Environments considered in Experiment 2.: Random (top left), Misleading (top right), Suboptimal (bottom left), Landmarks (bottom right).}
    \label{fig:e2-envs-app}
\end{figure}

\begin{table}[H]
    \includegraphics[width=\linewidth]{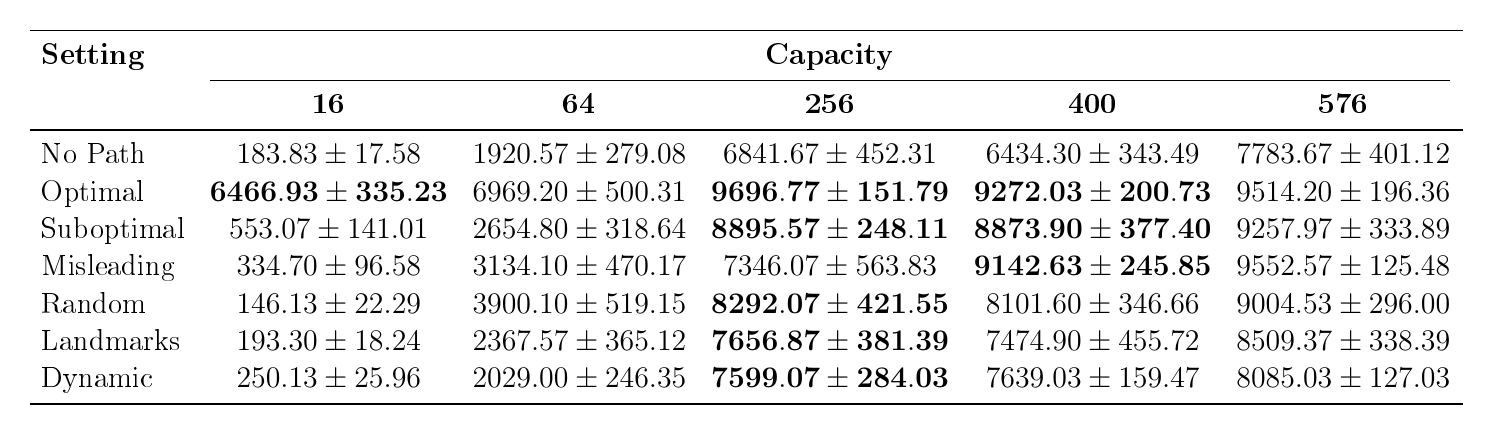}
    \includegraphics[width=\linewidth]{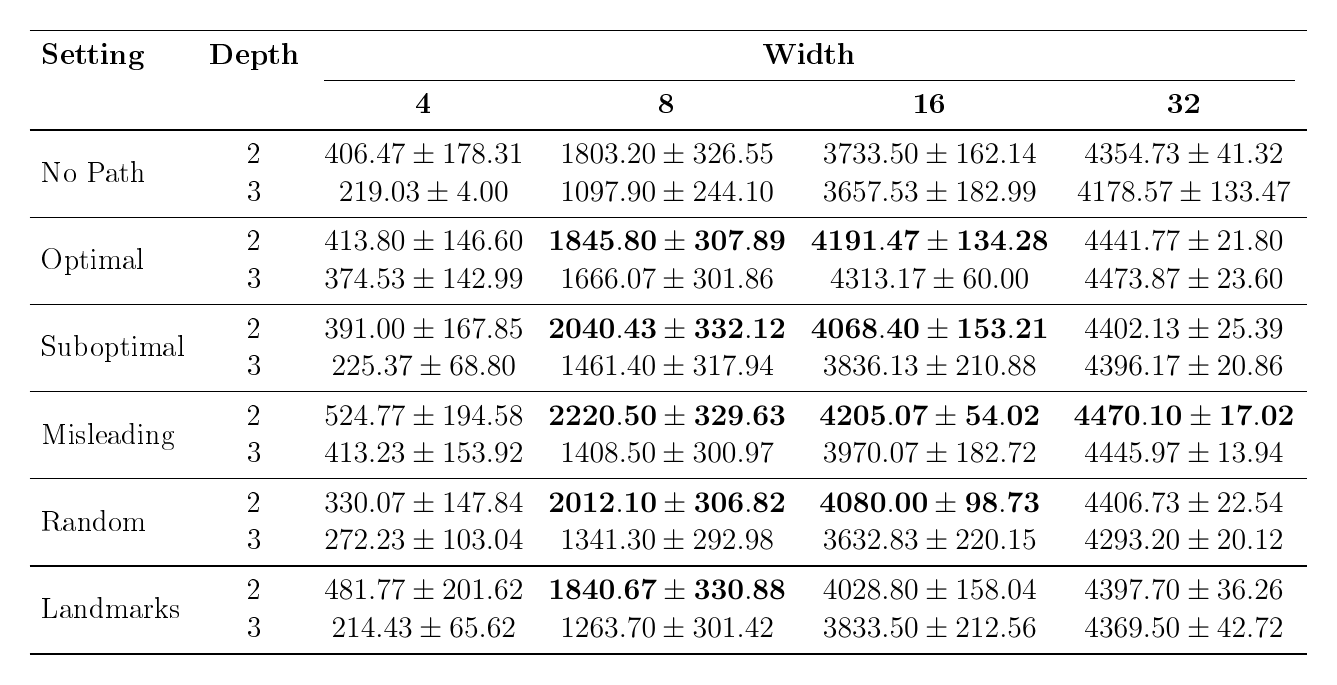}
    \caption{Average total reward with standard errors over thirty independent seeds, for all settings considered. Bold text indicates satisfaction of our empirical condition.}
    \label{tab:total_reward}
\end{table}

\begin{figure}[H]
    \centering
    \includegraphics[width=0.49\linewidth]{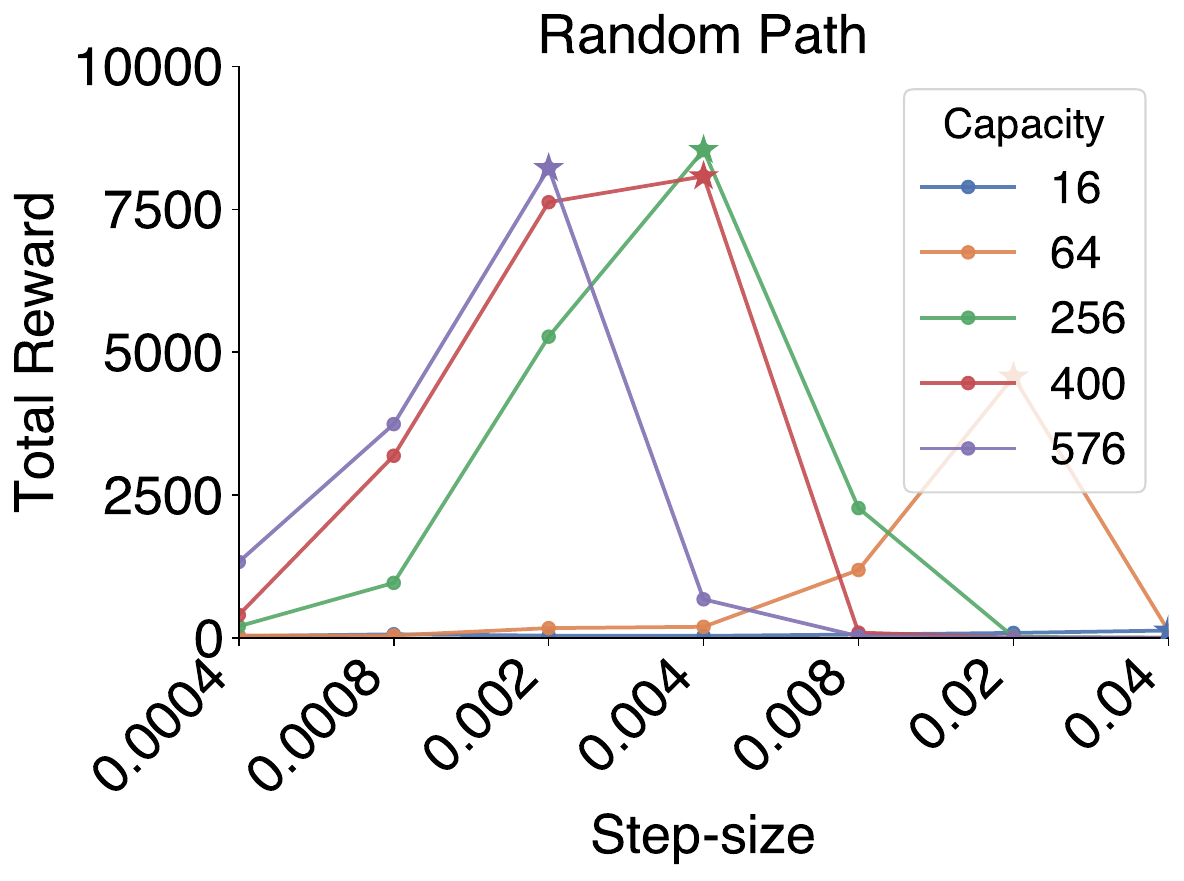}
    \includegraphics[width=0.49\linewidth]{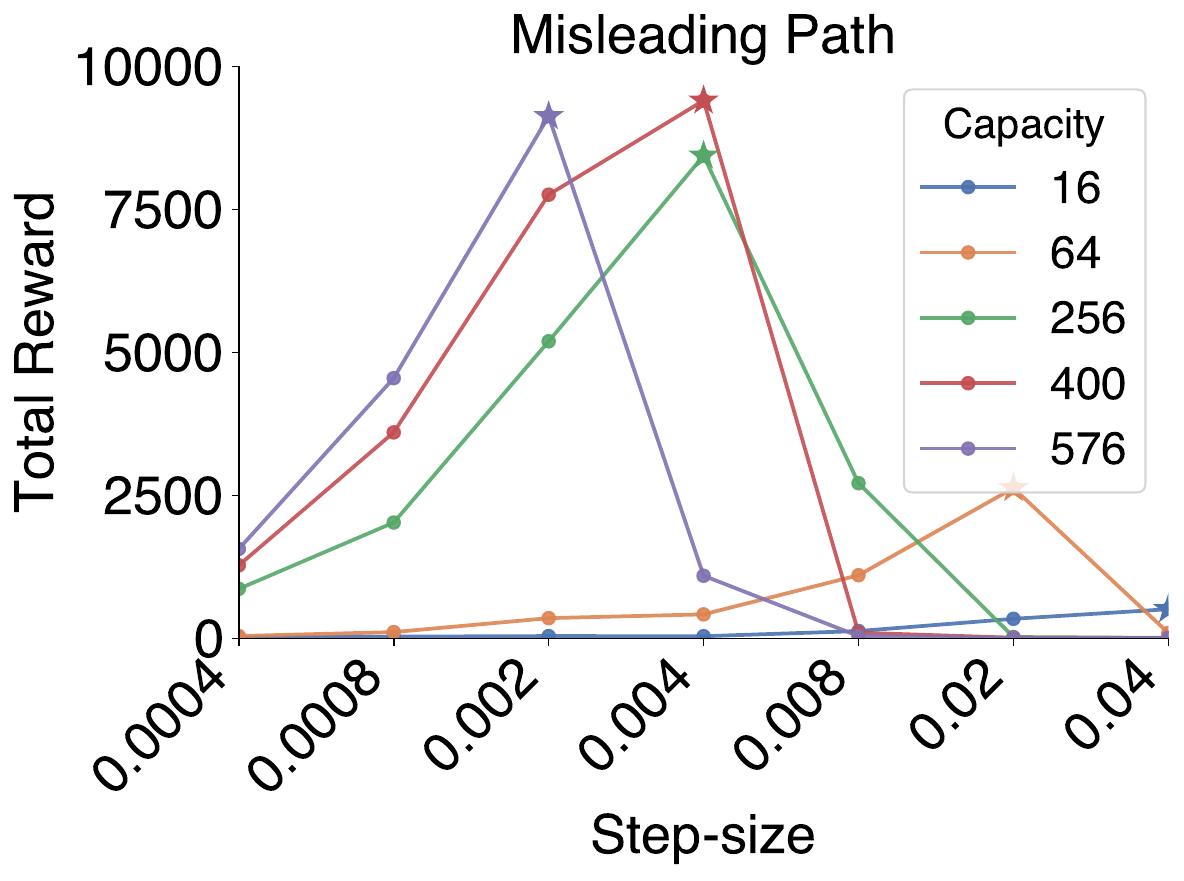}
    \includegraphics[width=0.49\linewidth]{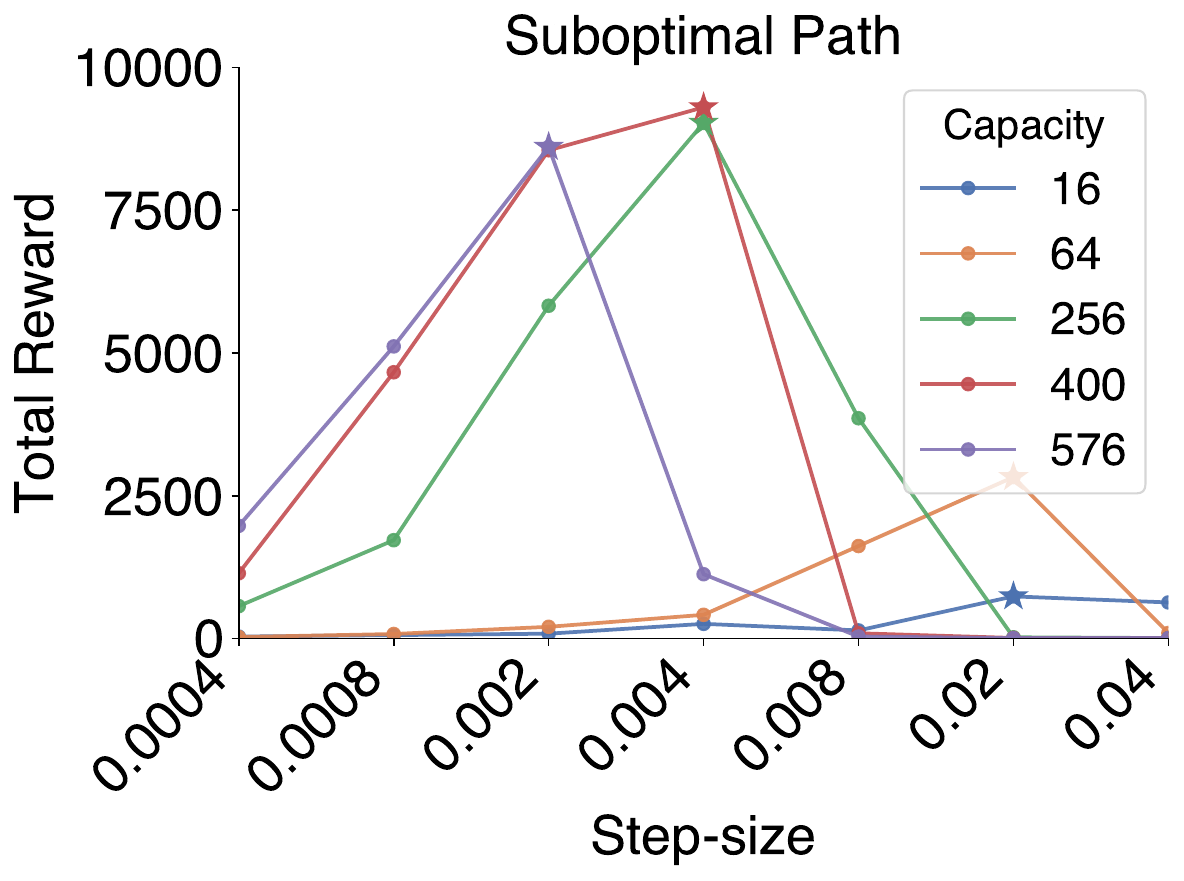}
    \includegraphics[width=0.49\linewidth]{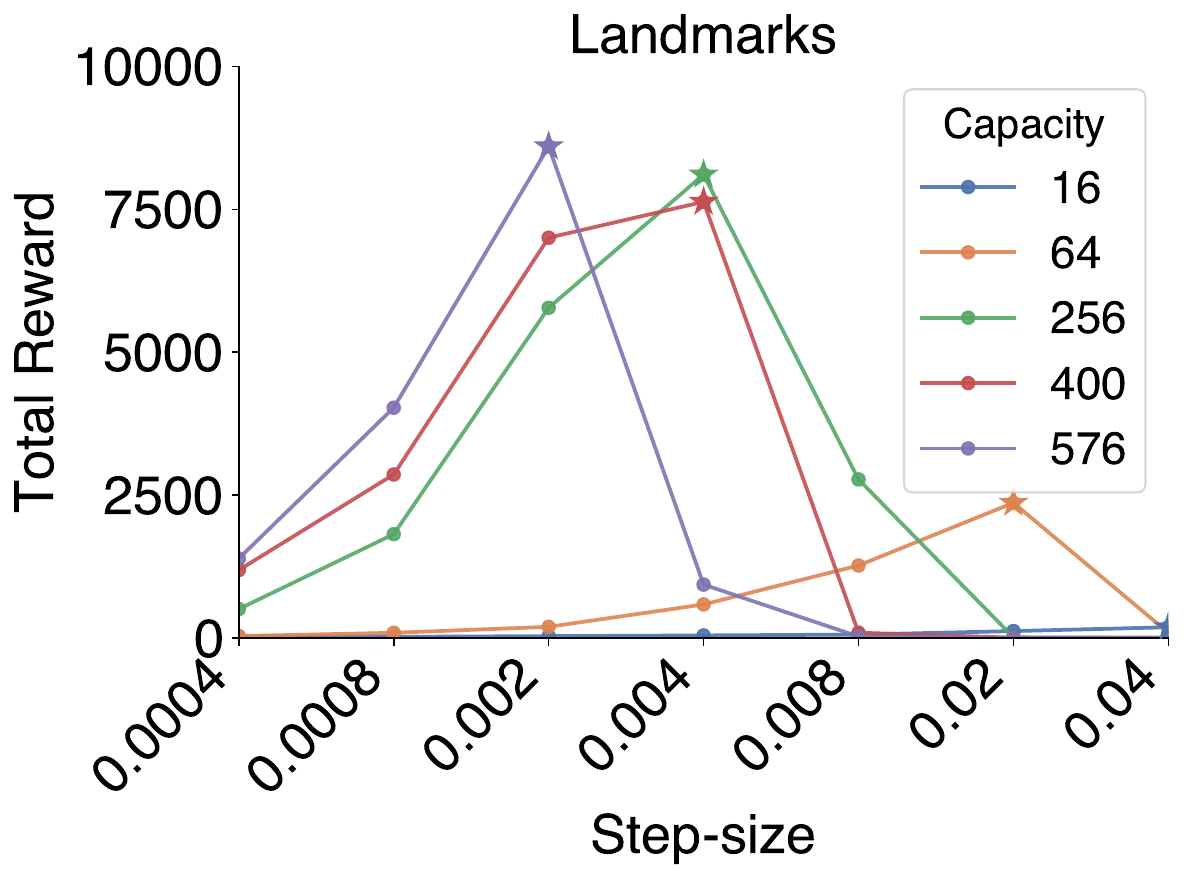}
    \includegraphics[width=0.49\linewidth]{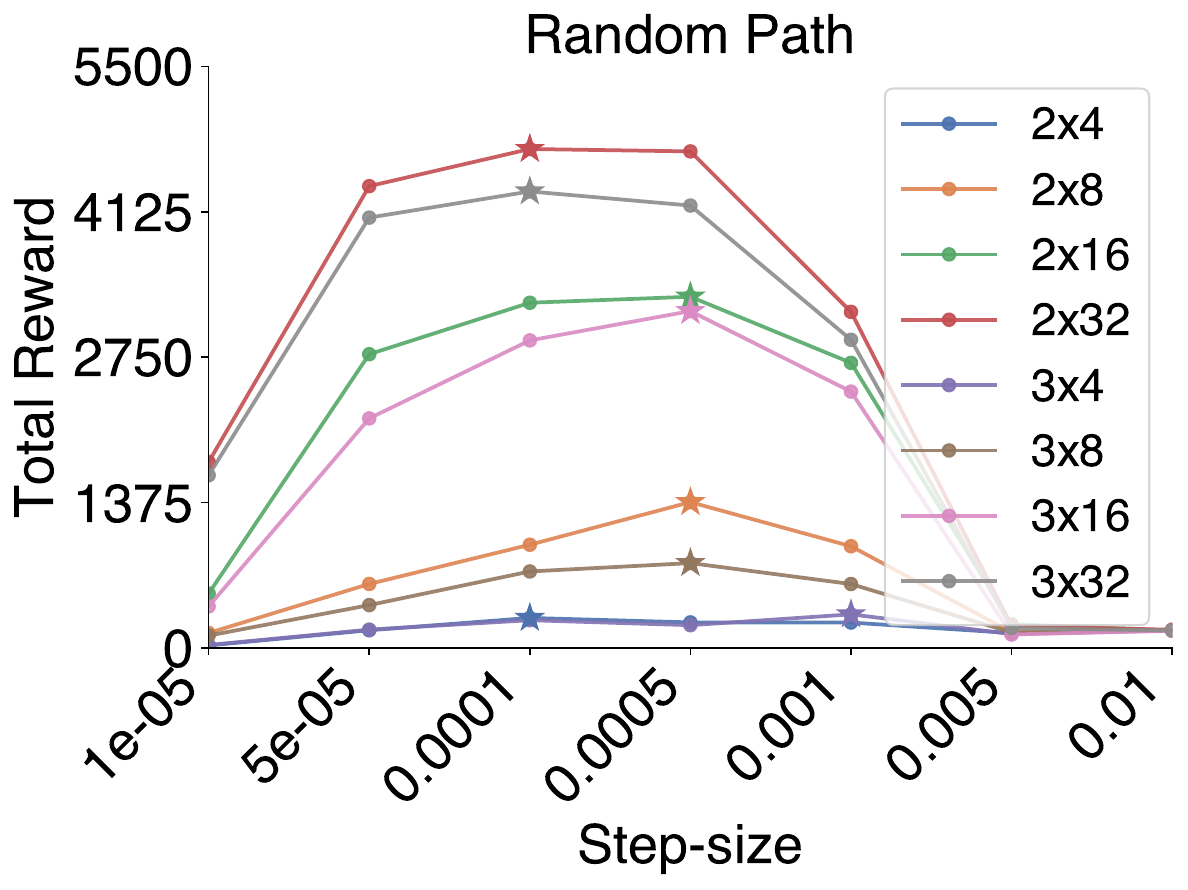}
    \includegraphics[width=0.49\linewidth]{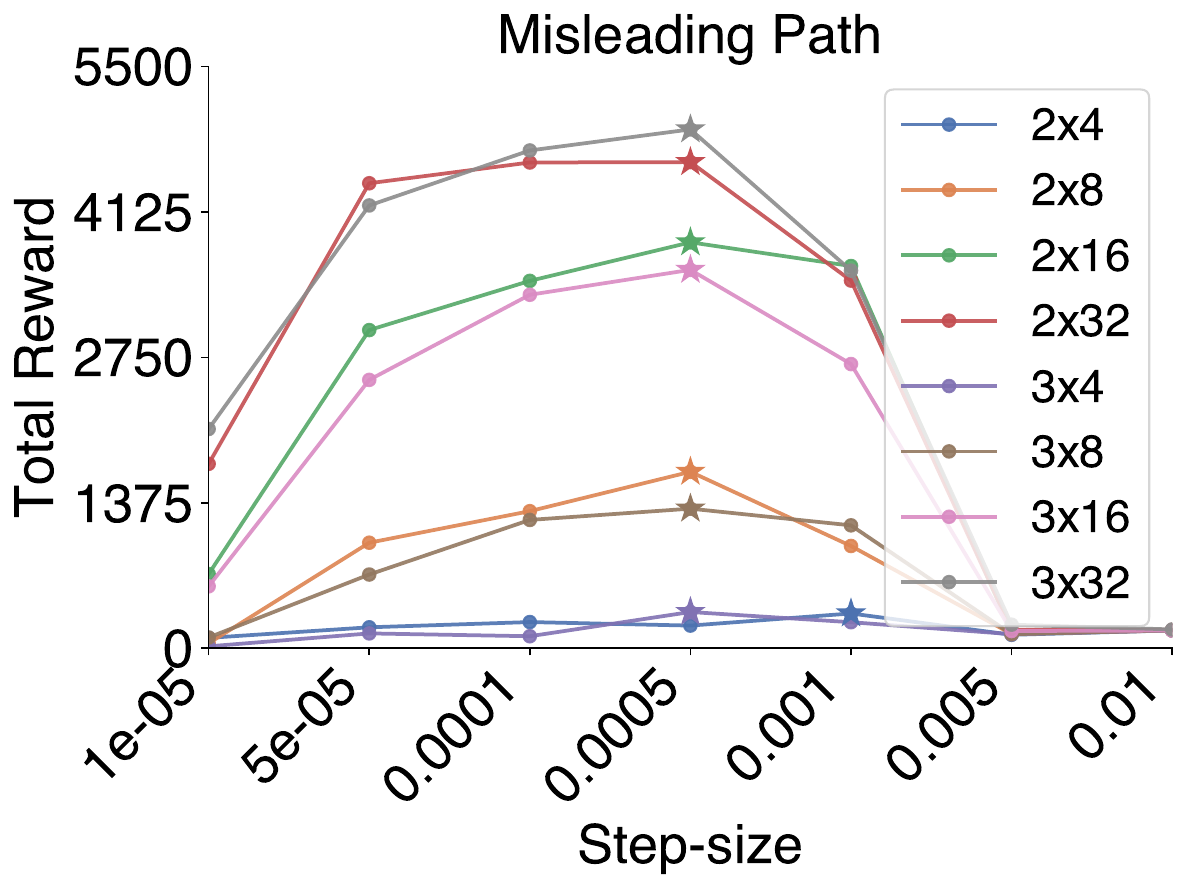}
    \includegraphics[width=0.49\linewidth]{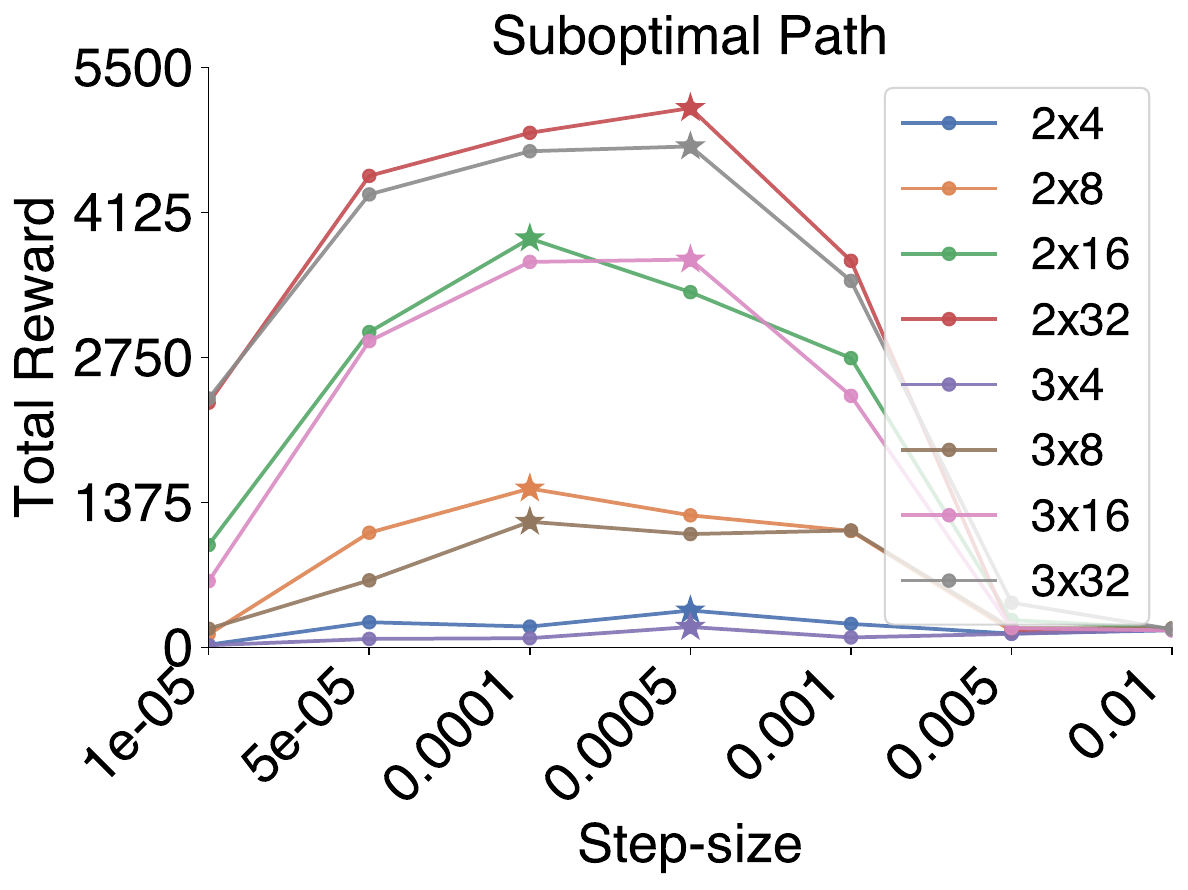}
    \includegraphics[width=0.49\linewidth]{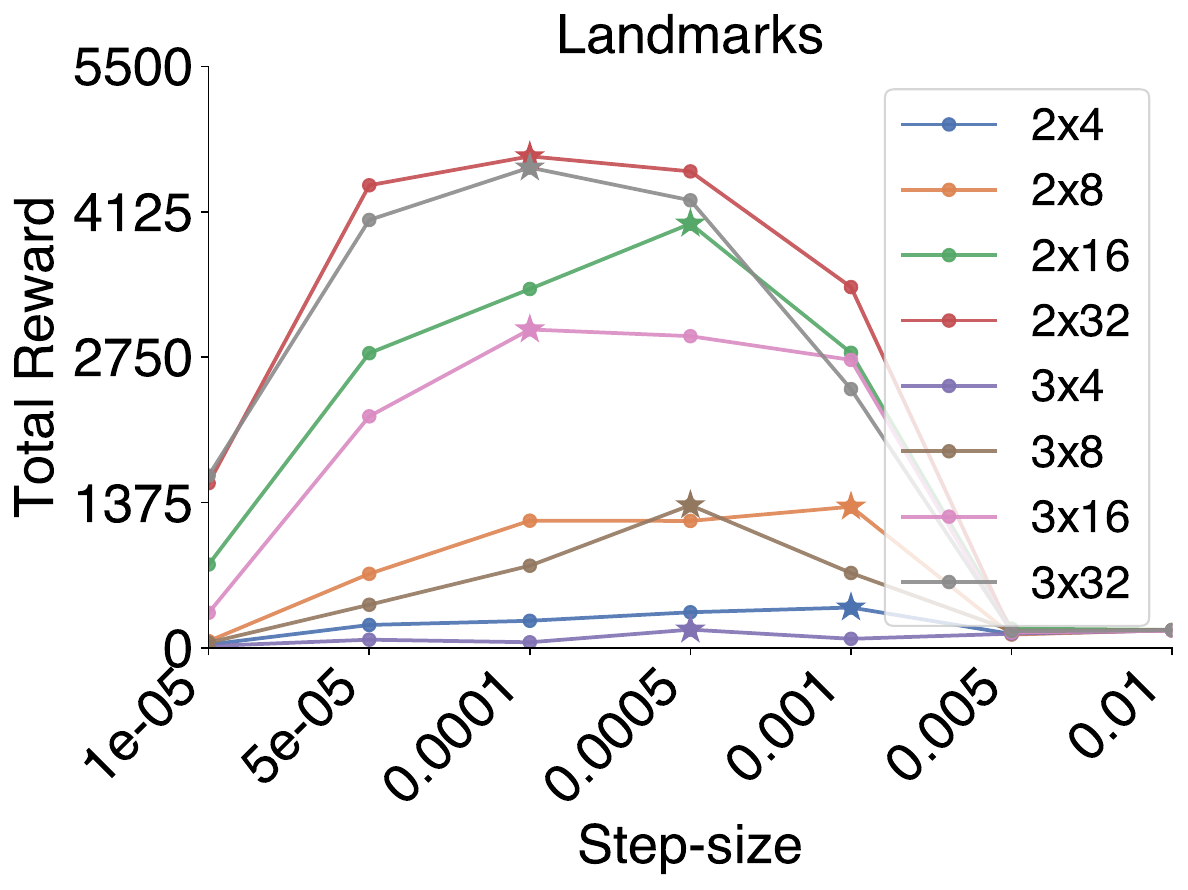}
    \caption{\textbf{Experiment 2. Step-size sweeps.} Linear Q-learning (top half), DQN (lower half).}
    \label{fig:exp2_sweep}
\end{figure}

\begin{figure}[H]
    \centering
    \includegraphics[width=0.49\linewidth]{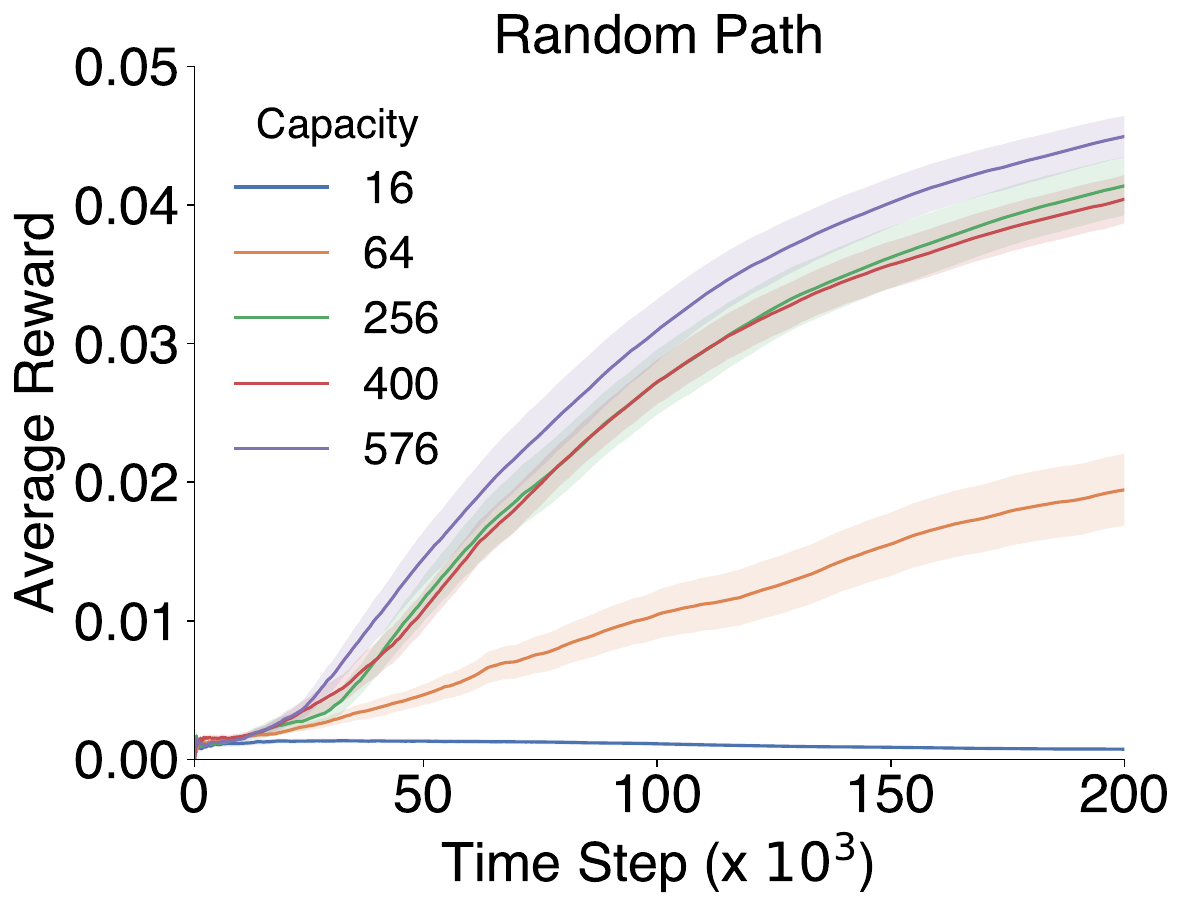}
    \includegraphics[width=0.49\linewidth]{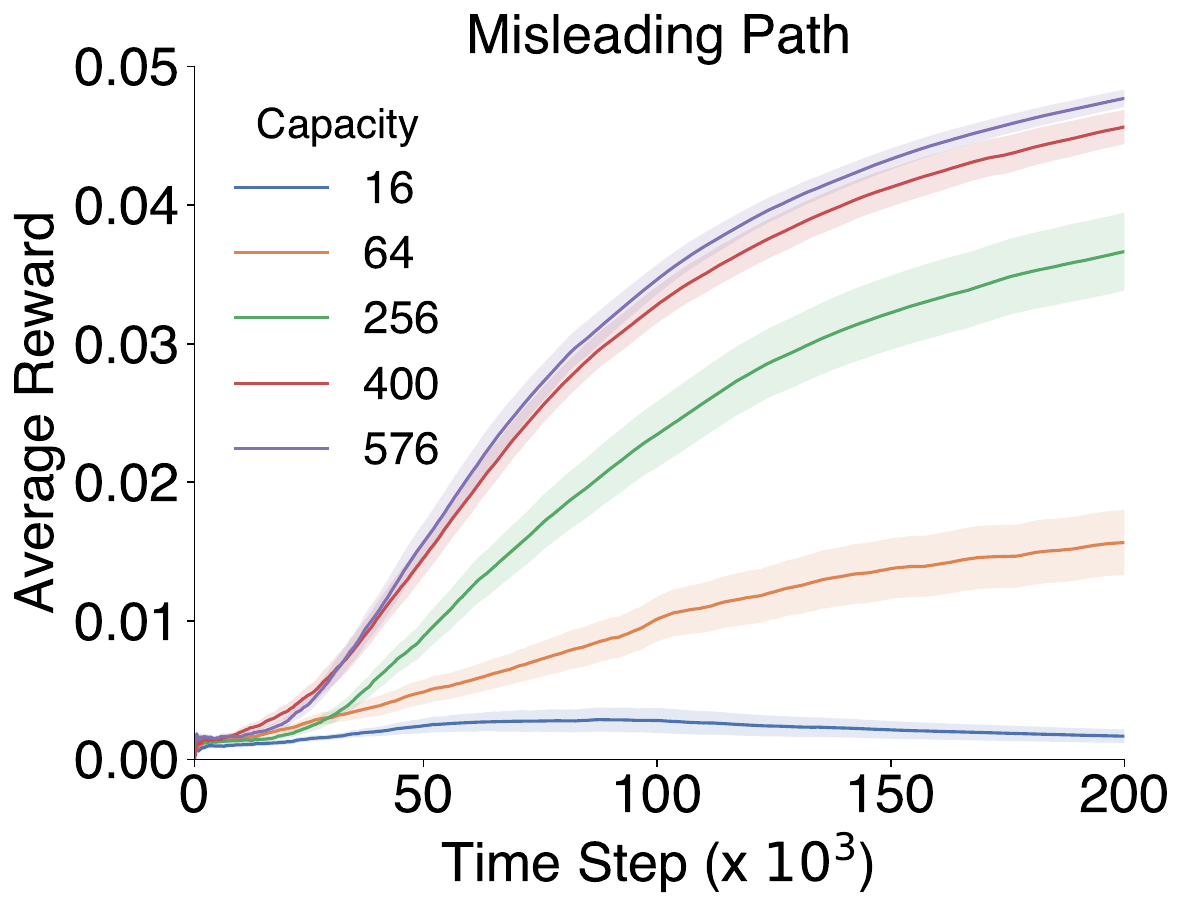}
    \includegraphics[width=0.49\linewidth]{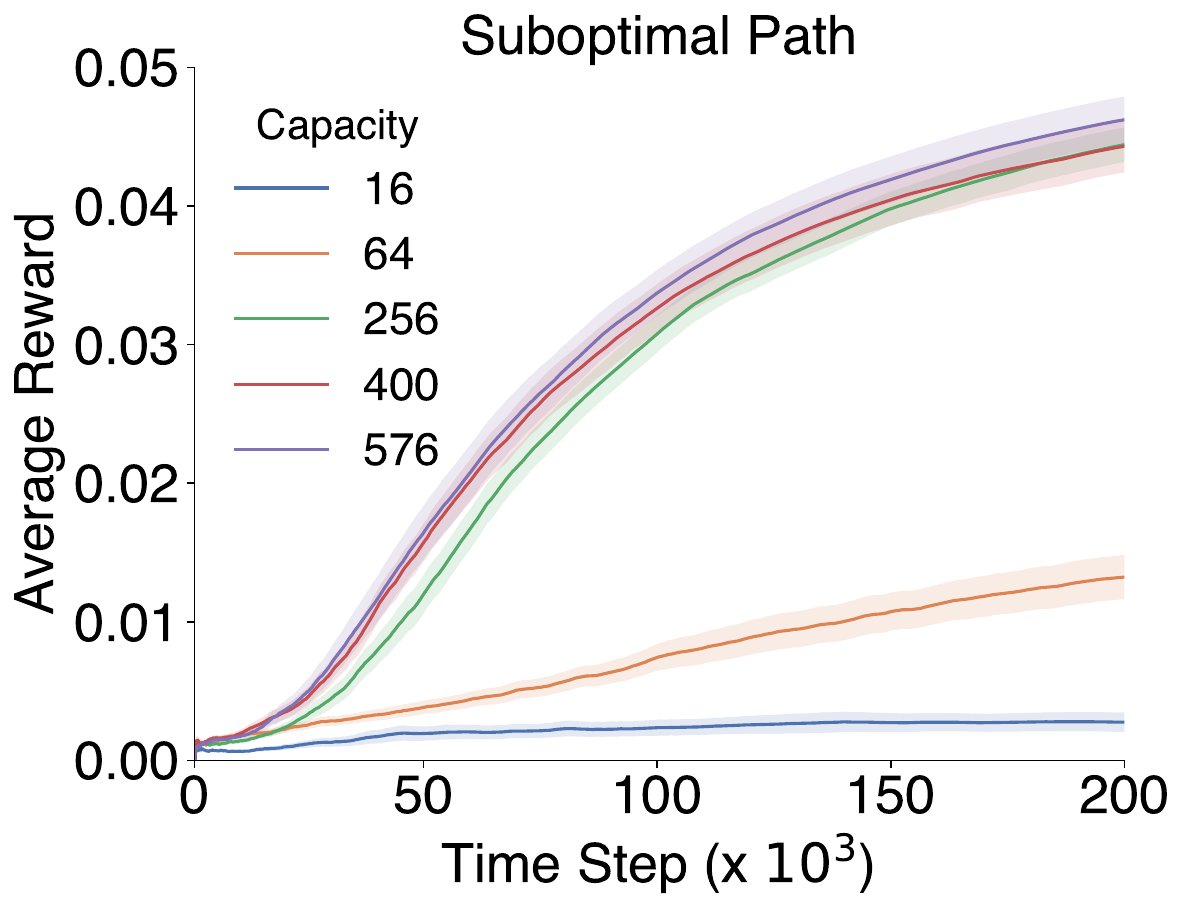}
    \includegraphics[width=0.49\linewidth]{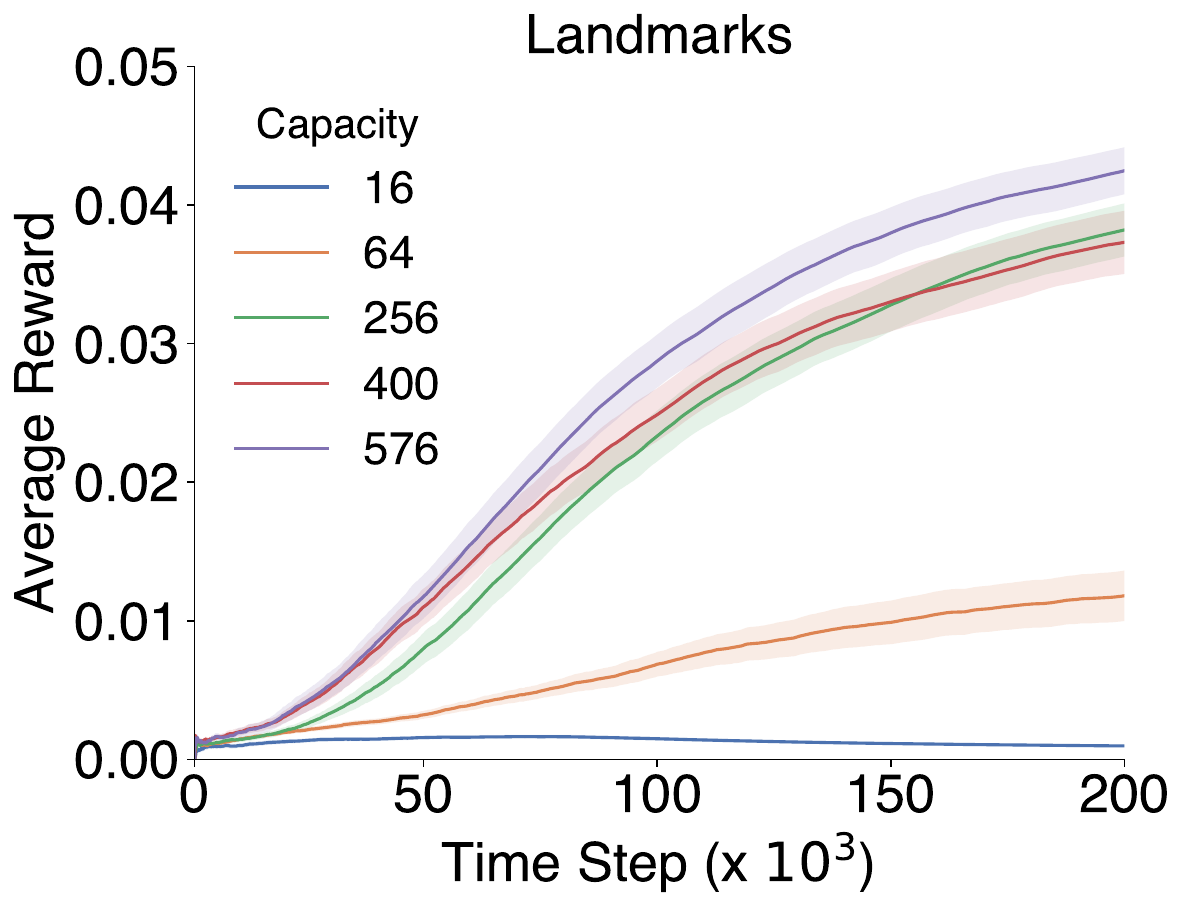}
    \includegraphics[width=0.49\linewidth]{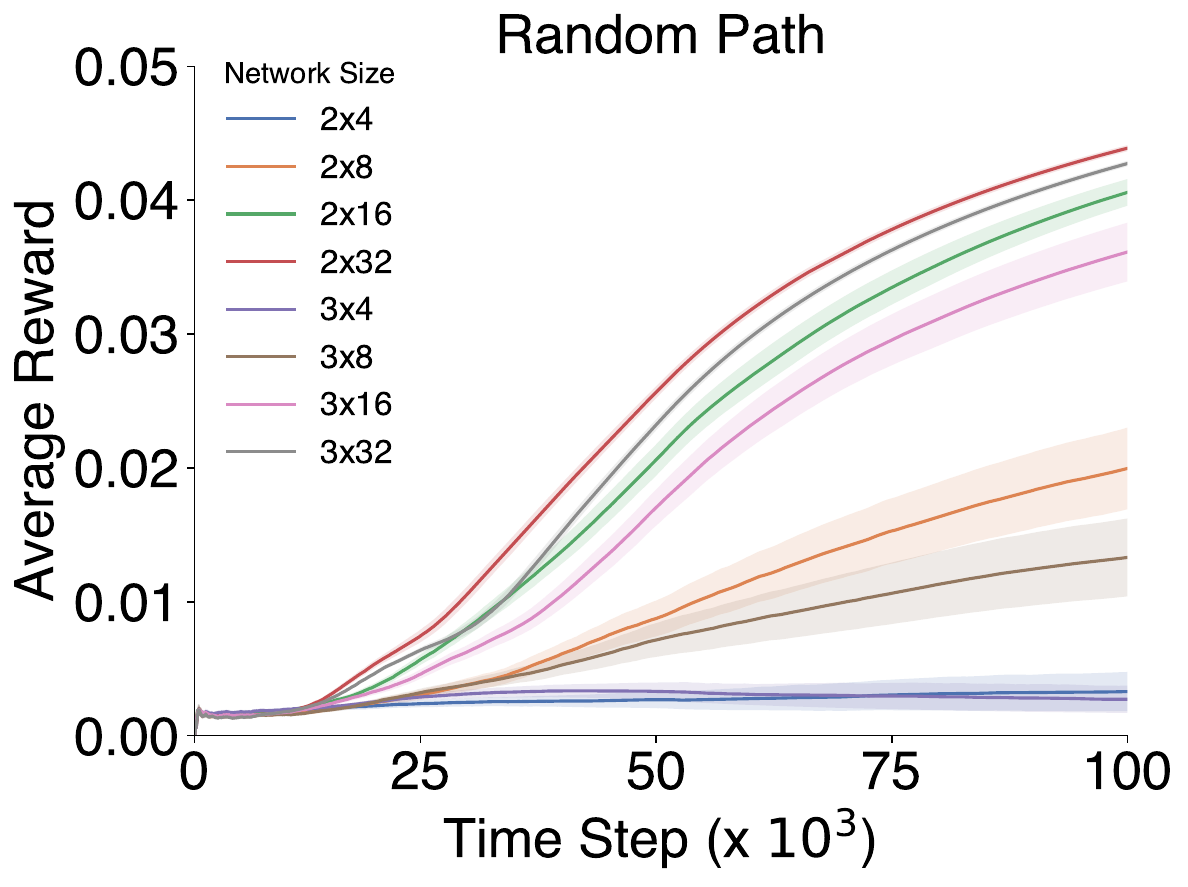}
    \includegraphics[width=0.49\linewidth]{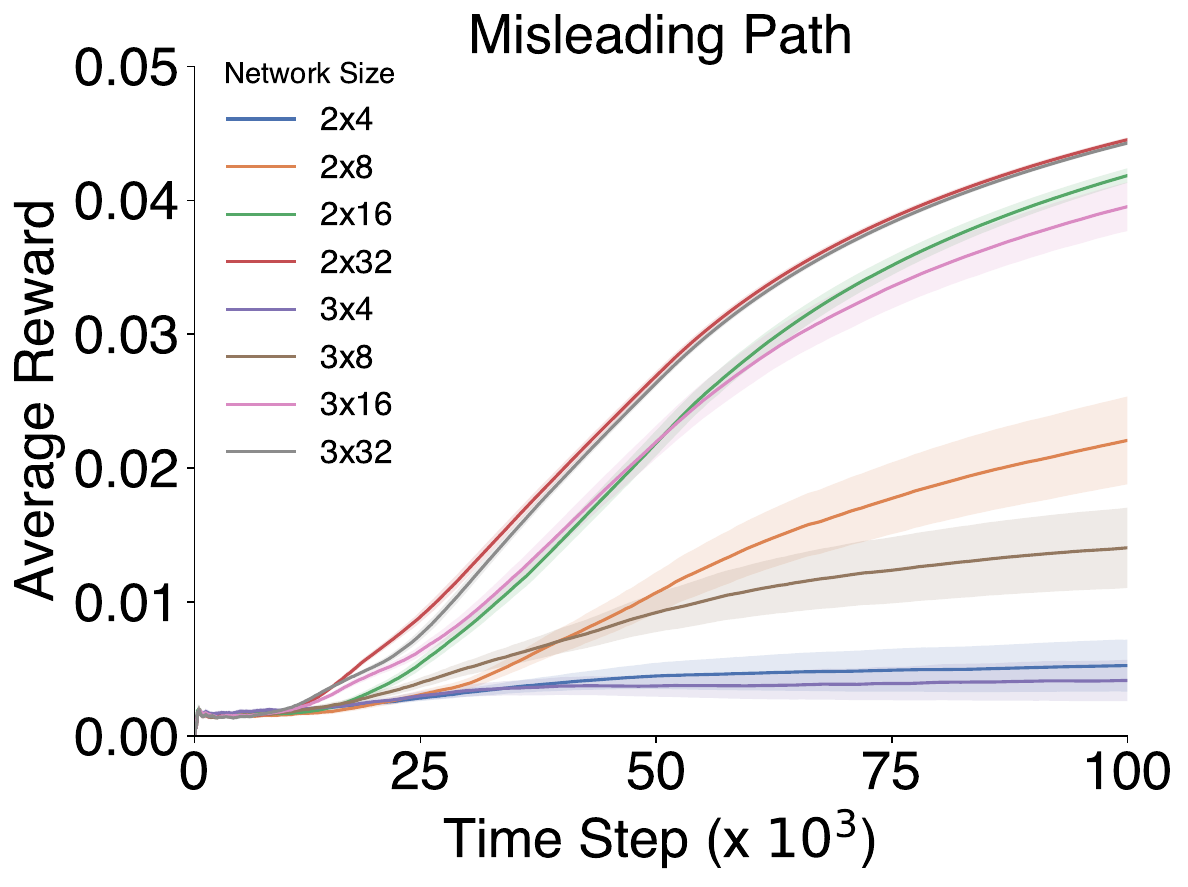}
    \includegraphics[width=0.49\linewidth]{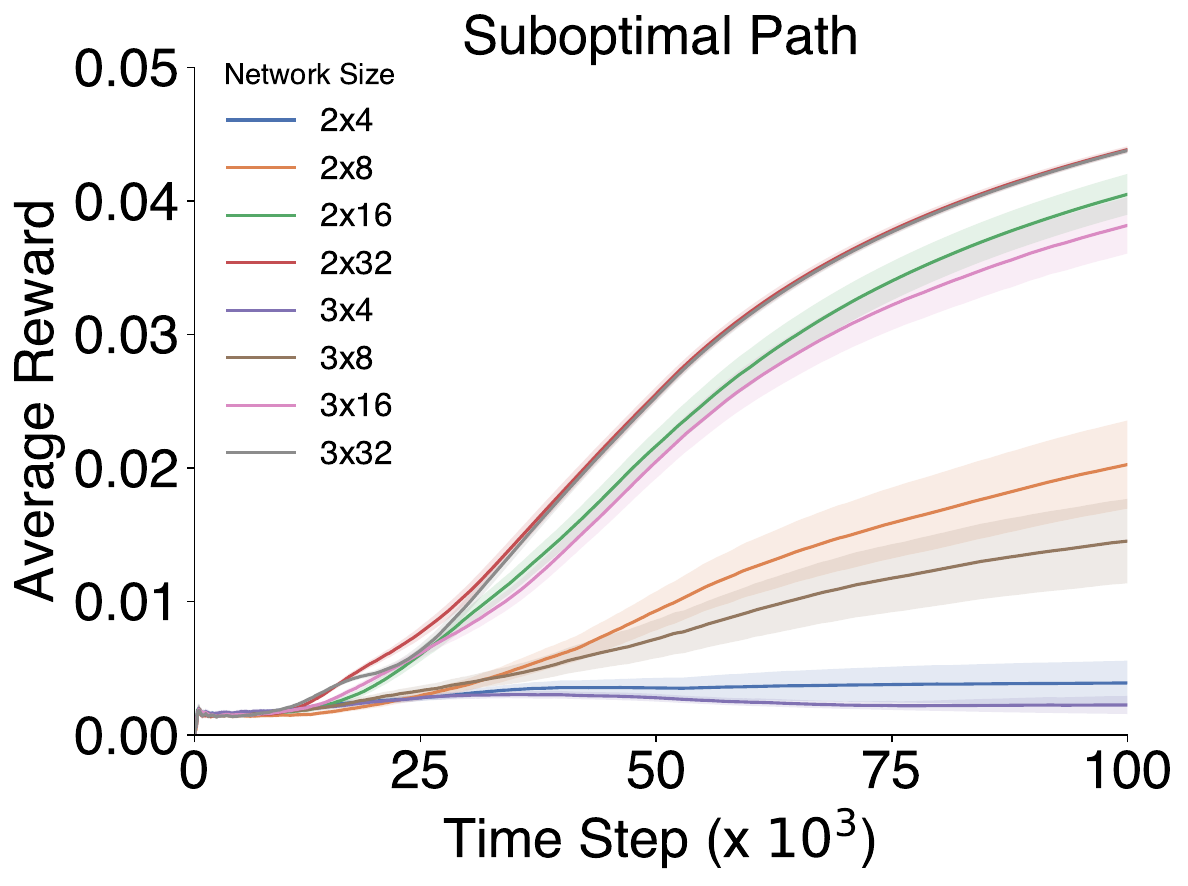}
    \includegraphics[width=0.49\linewidth]{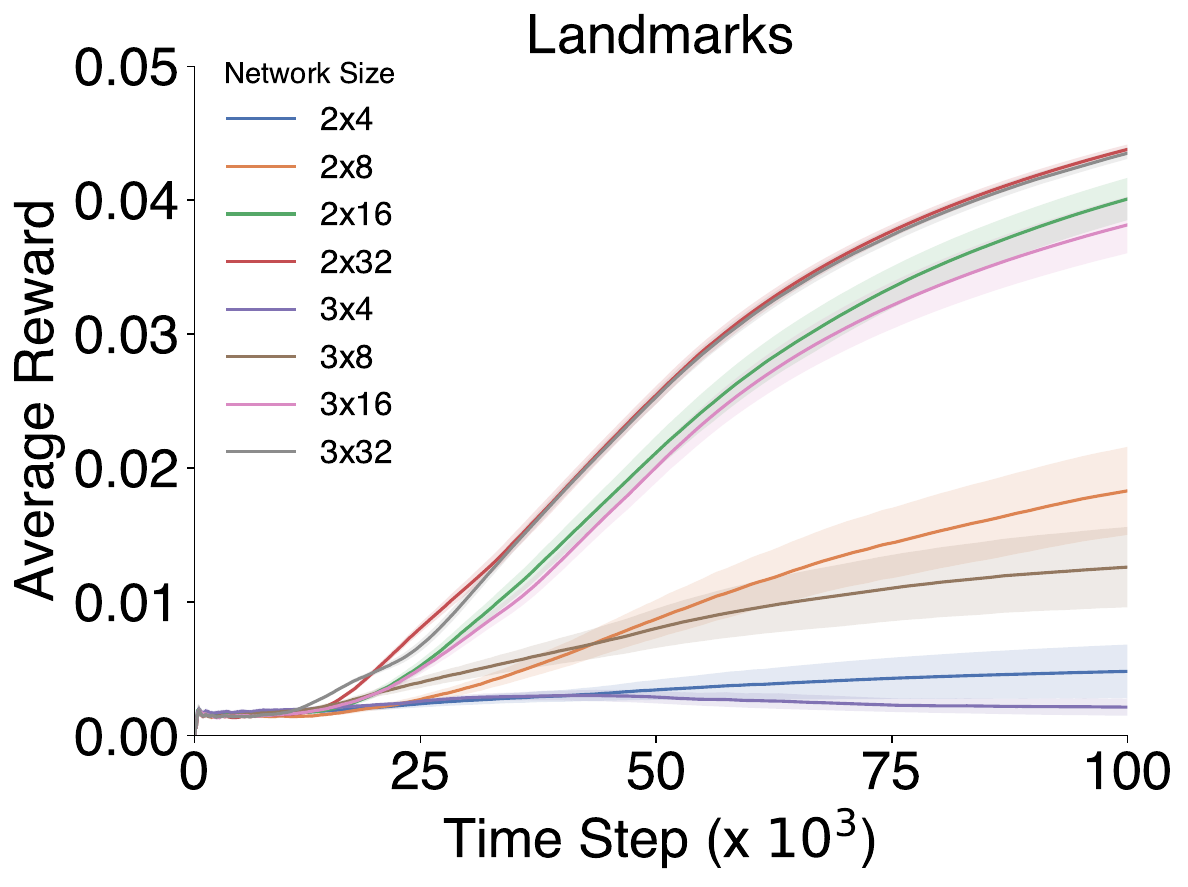}
    \caption{\textbf{Experiment 2. Average Reward.} Linear-Q (top half), DQN (bottom half).}
    \label{fig:exp2_stepsize}
\end{figure}

\begin{figure}[H]
    \centering
    \includegraphics[width=0.49\linewidth]{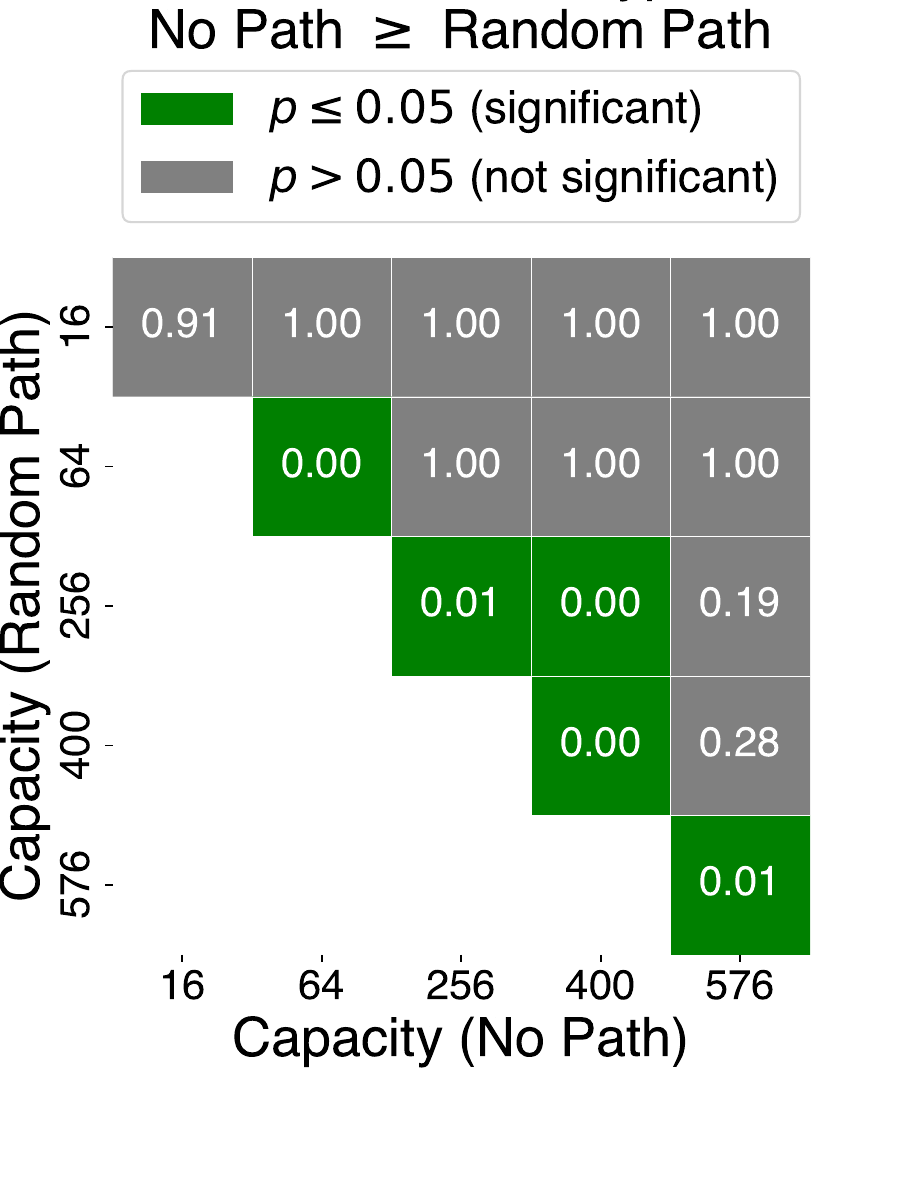}
    \includegraphics[width=0.49\linewidth]{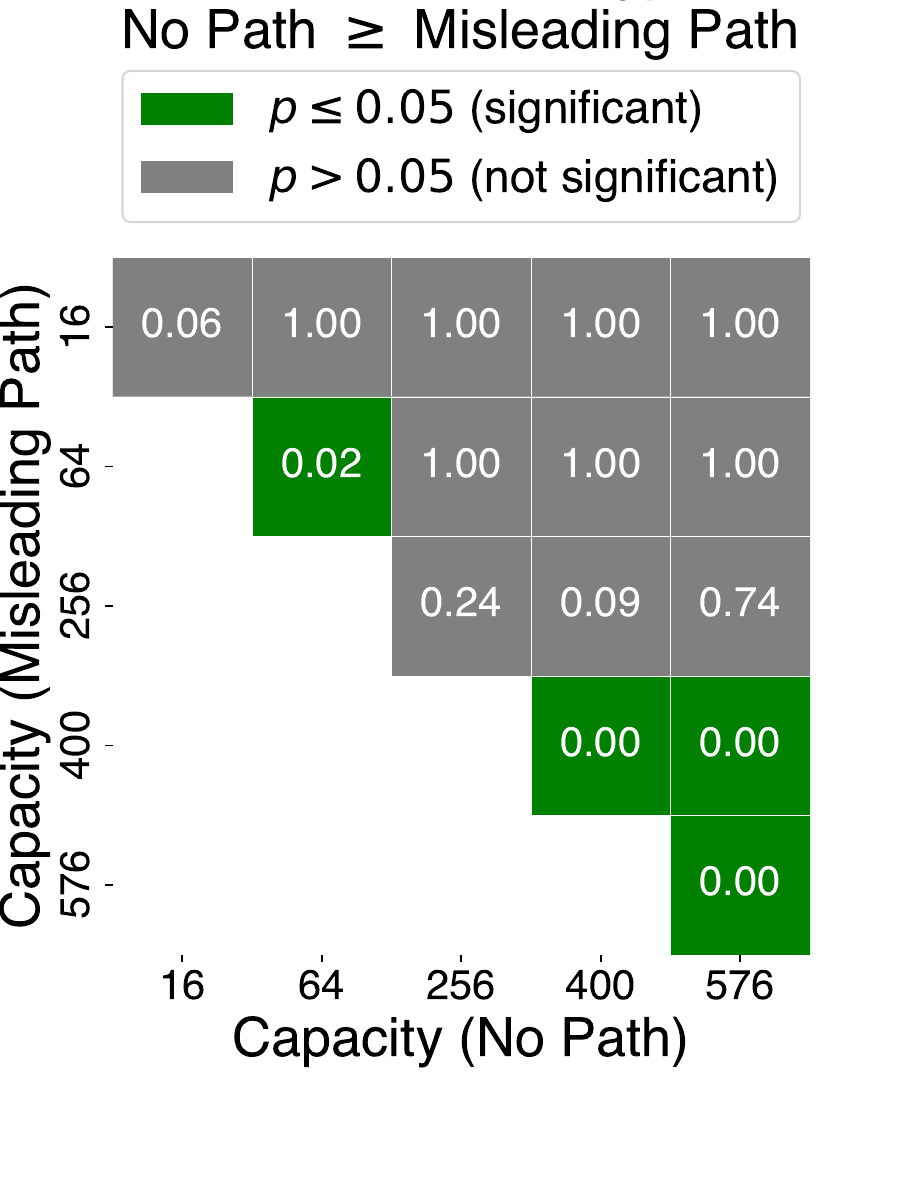}
    \includegraphics[width=0.49\linewidth]{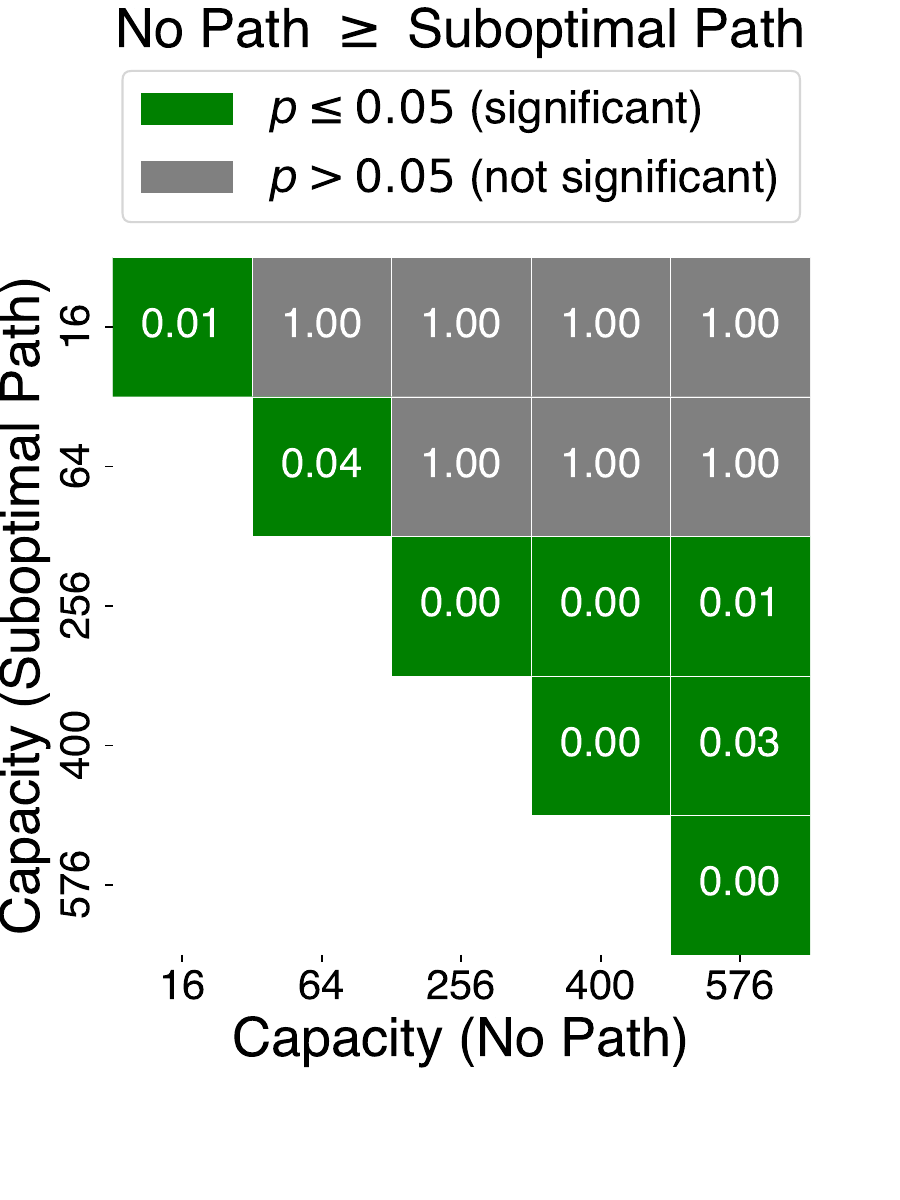}
    \includegraphics[width=0.49\linewidth]{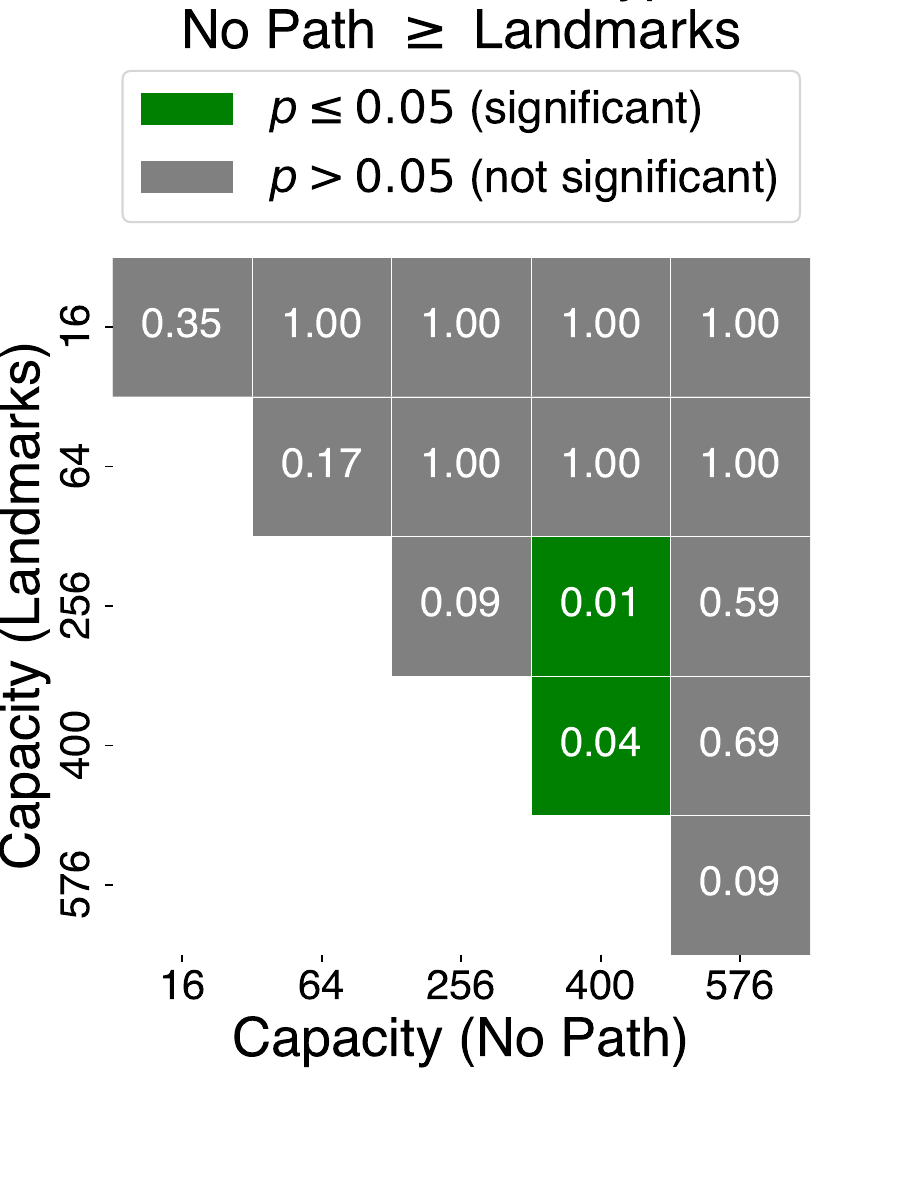}
    \caption{\textbf{Experiment 2. Linear Significant Tests.} Let $P_i$ and $P_j$ be the performances associated with capacities of the row and column. The plot should be read row-wise: when the $(i,j)$-cell is green, $P_i$ is significantly higher than $P_j$.}
\end{figure}

\begin{figure}[H]
    \includegraphics[width=0.49\linewidth]{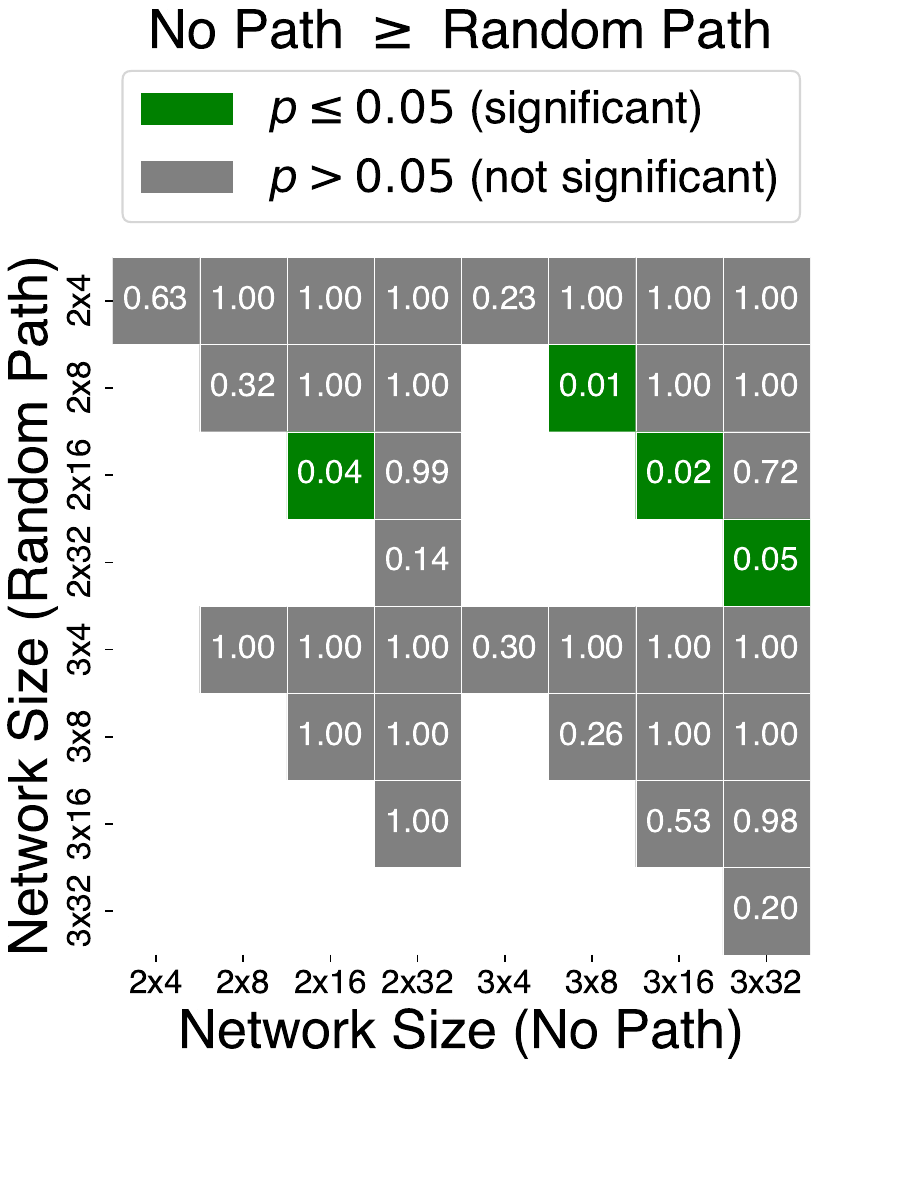}
    \includegraphics[width=0.49\linewidth]{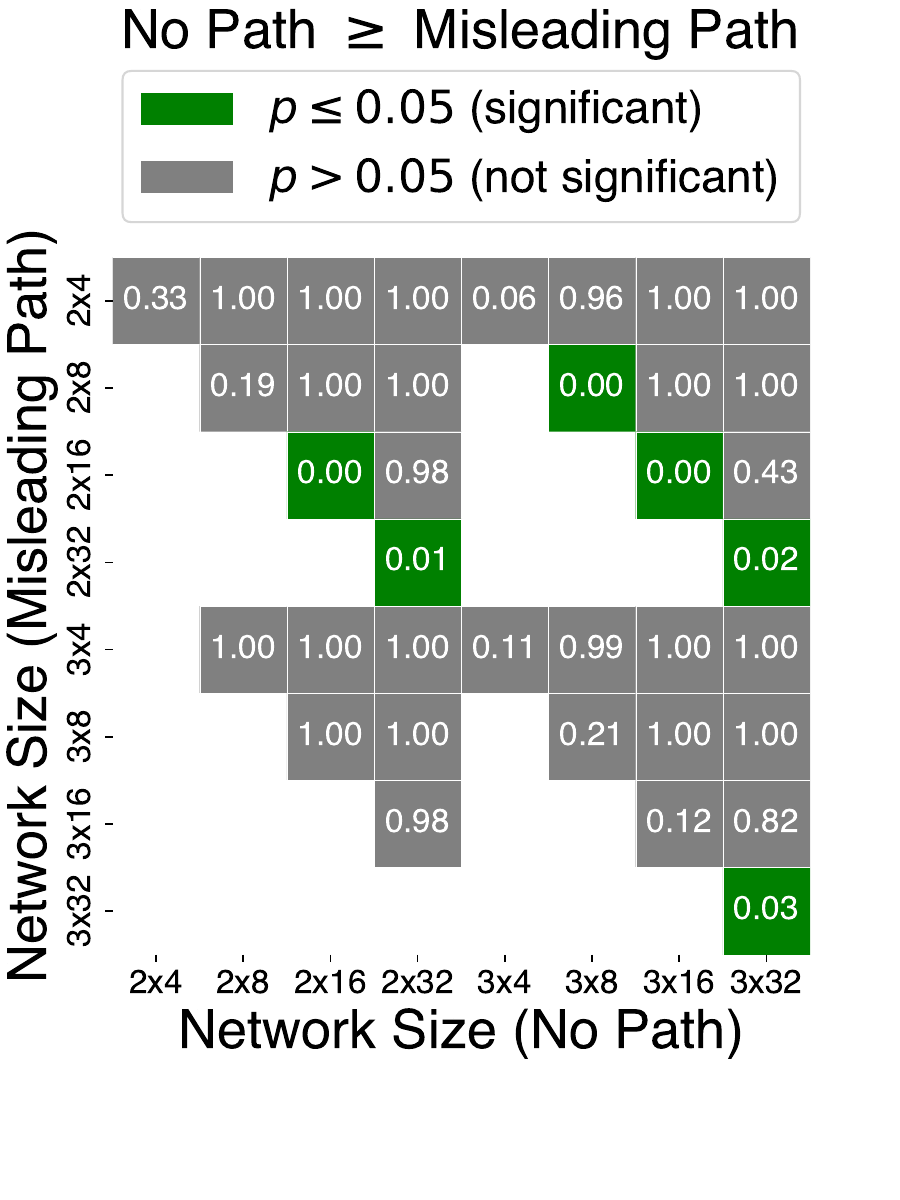}
    \includegraphics[width=0.49\linewidth]{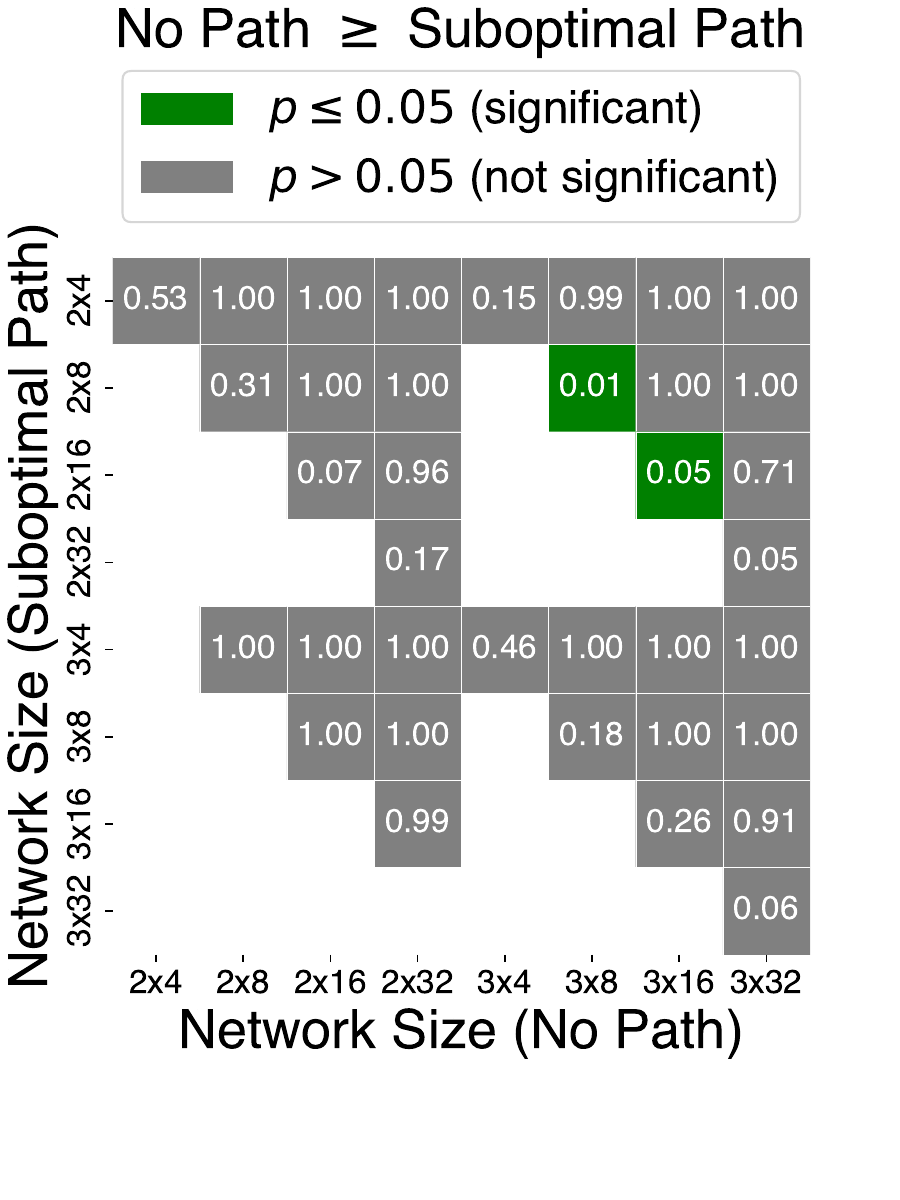}
    \includegraphics[width=0.49\linewidth]{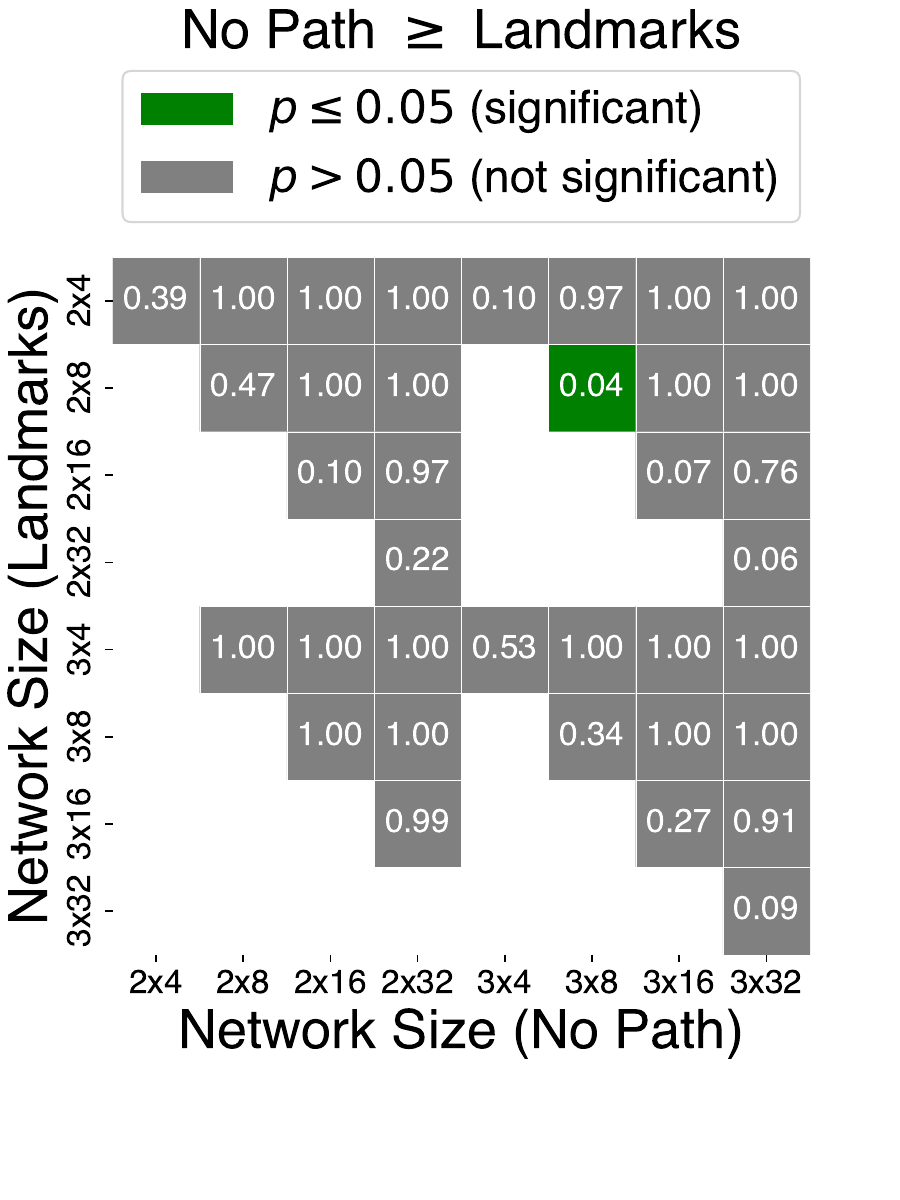}
    \caption{\textbf{Experiment 2. DQN Significant Tests.} Let $P_i$ and $P_j$ be the performances associated with capacities of the row and column. The plot should be read row-wise: when the $(i,j)$-cell is green, $P_i$ is significantly higher than $P_j$.}
    \label{fig:exp2_sigtest}
\end{figure}

\subsection{Experiment 3. Learning in the Presence of a Dynamic Path}

\begin{figure}[h]
    \centering
    \includegraphics[width=0.49\linewidth]{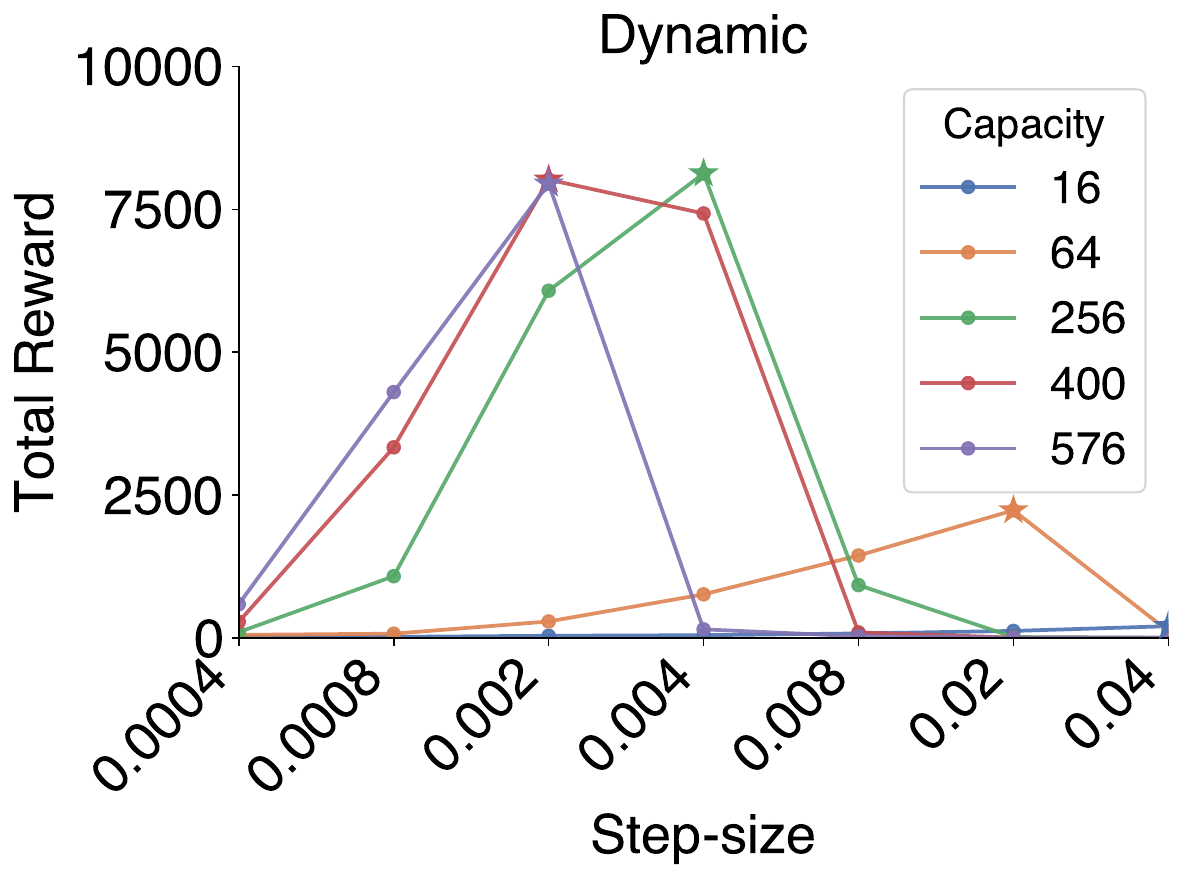}
    \includegraphics[width=0.49\linewidth]{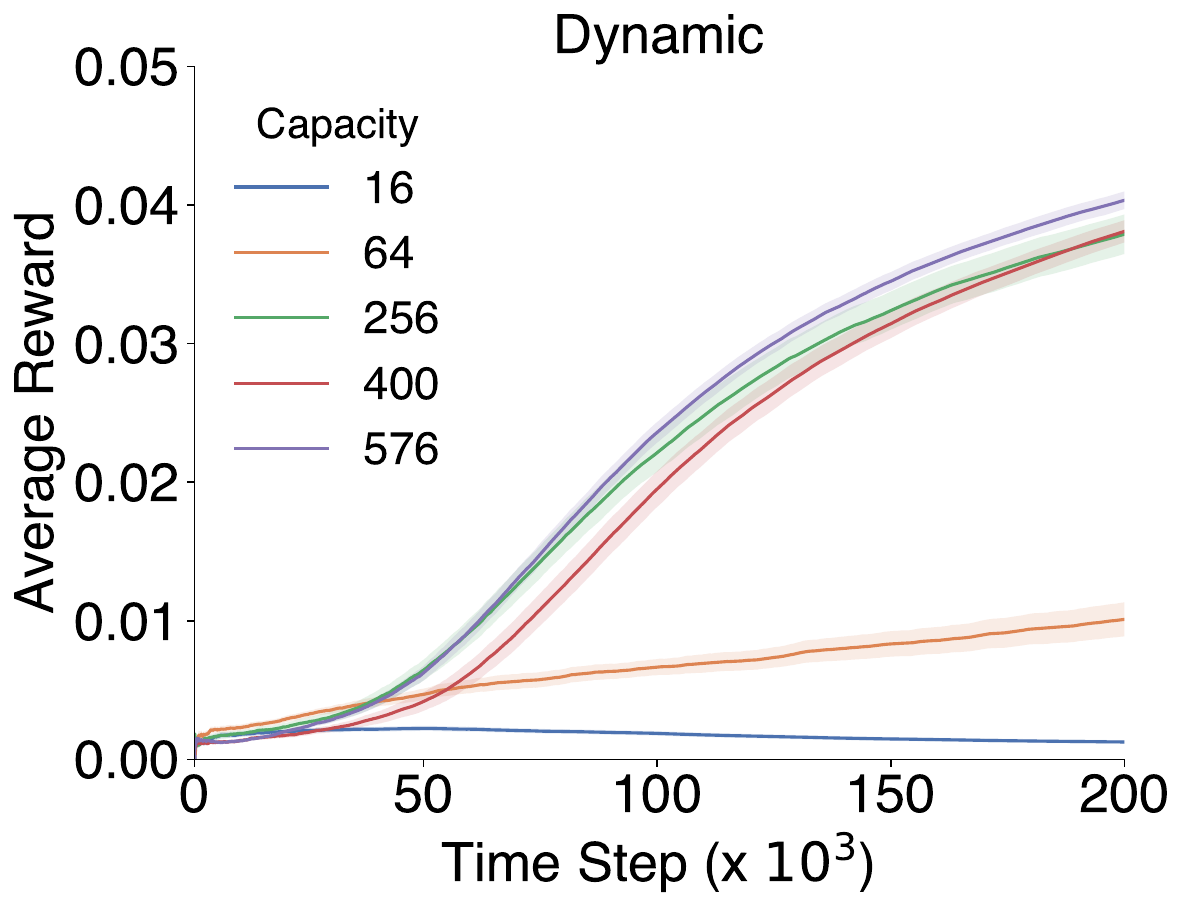}
    \caption{\textbf{Experiment 3. Dynamic Path Average Reward and Step-size sweeps.} Linear-Q.}
    \label{fig:exp3_stepsize_dyn_lin}
\end{figure}
\begin{figure}[h]
\centering
    \includegraphics[width=0.49\linewidth]{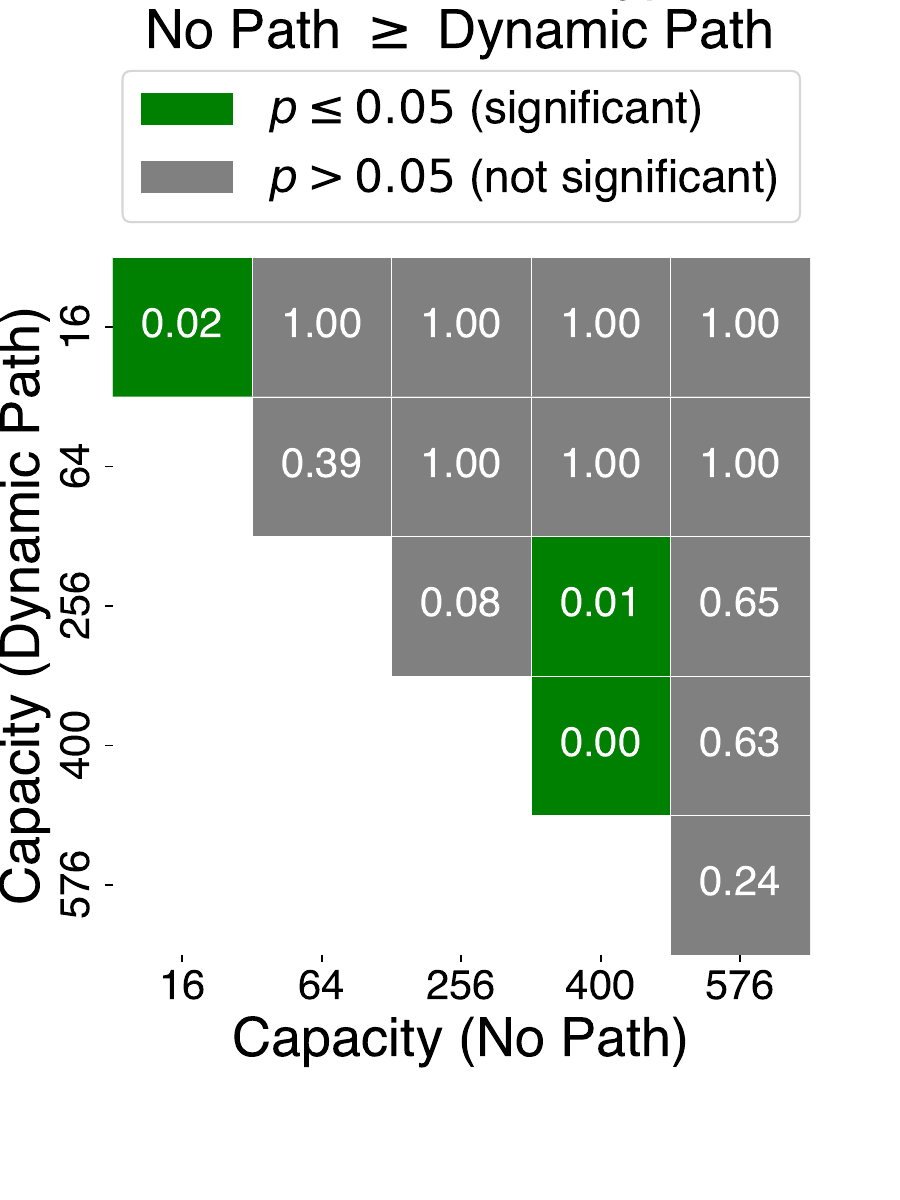}
    \caption{\textbf{Experiment 3. Linear Significance Tests.} Let $P_i$ and $P_j$ be the performances associated with capacities of the row and column. The plot should be read row-wise: when the $(i,j)$-cell is green, $P_i$ is significantly higher than $P_j$.}
    \label{fig:exp3_sigtest}
\end{figure}

\end{document}